\documentclass[twoside,11pt]{article}

\usepackage{jmlr2e}

\usepackage{mathtools}

\usepackage{lastpage}
\ShortHeadings{A Well-Tempered Landscape for Non-convex Robust Subspace Recovery}{Maunu, Zhang and Lerman} 

\usepackage{multirow}

\usepackage{amsmath}
\usepackage{amssymb}

\usepackage{hyperref}
\usepackage{cases}
\usepackage{setspace}
\usepackage{breakcites}
\usepackage{color}

\usepackage{algorithm,algorithmicx,algpseudocode}
\usepackage{xargs}

\newcommand{\reals}{\mathbb R}
\newcommand{\nats}{\mathbb N}
\newcommand{\E}{\mathbb E}

\def\bx{\boldsymbol{x}}
\def\by{\boldsymbol{y}}
\def\bY{\boldsymbol{Y}}
\def\bX{\boldsymbol{X}}

\def\bA{\boldsymbol{A}}
\def\bx{\boldsymbol{x}}
\def\bv{\boldsymbol{v}}

\def\bu{\boldsymbol{u}}
\def\bw{\boldsymbol{w}}

\def\bU{\boldsymbol{U}}
\def\bV{\boldsymbol{V}}
\def\bW{\boldsymbol{W}}

\def\bQ{\boldsymbol{Q}}
\def\bB{\boldsymbol{B}}
\def\bP{\boldsymbol{P}}

\def\bepsilon{\boldsymbol{\epsilon}}

\def\bQ{\boldsymbol{Q}}

\def\bDelta{\boldsymbol{\Delta}}

\def\bbeta{\boldsymbol{\beta}}
\def\bSigma{\boldsymbol{\Sigma}}
\def\bLambda{\boldsymbol{\Lambda}}
\def\bTheta{\boldsymbol{\Theta}}
\def\bR{\boldsymbol{R}}
\def\bI{\boldsymbol{I}}

\def\cX{\mathcal{X}}

\def\cF{\mathcal{F}}
\def\cC{\mathcal{C}}

\def\cS{\mathcal{S}}

\def\cP{\mathcal{P}}
\def\cA{\mathcal{A}}
\def\cG{\mathcal{G}}
\def\cN{\mathcal{N}}

\def\cI{\mathcal{I}}

\def\bzero{\boldsymbol{0}}

\def\di{\mathrm{d} }

\DeclareMathOperator{\dist}{\mathrm{dist}}
\DeclareMathOperator{\SNR}{\mathrm{SNR}}

\DeclareMathOperator{\Sp}{\mathrm{Sp}}

\DeclareMathOperator{\Var}{\mathrm{Var}}
\DeclareMathOperator{\diag}{\mathrm{diag}}

\begin{document}

\title{A Well-Tempered Landscape for Non-convex Robust Subspace Recovery}

 \author{\name Tyler Maunu \email maunut@mit.edu  \\
 \addr  Department of Mathematics \\ Massachusetts Institute of Technology \\ Cambridge, MA 02139 
 \AND
 \name Teng Zhang  \email teng.zhang@ucf.edu  \\
 \addr Department of Mathematics \\ University of Central Florida \\ Orlando, FL 32816
 \AND
 \name Gilad Lerman  \email lerman@umn.edu  \\
 \addr School of Mathematics \\ University of Minnesota \\ Minneapolis, MN 55455
 }


\maketitle

\begin{abstract}%
    We present a mathematical analysis of a non-convex energy landscape for robust subspace recovery. We prove that an underlying subspace is the only stationary point and local minimizer in a specified neighborhood under a deterministic condition on a dataset. If the deterministic condition is satisfied, we further show that a geodesic gradient descent method over the Grassmannian manifold can exactly recover the underlying subspace when the method is properly initialized. Proper initialization by principal component analysis is guaranteed with a simple deterministic condition. Under slightly stronger assumptions, the gradient descent method with a piecewise constant step-size scheme achieves linear convergence. The practicality of the deterministic condition is demonstrated on some statistical models of data, and the method achieves almost state-of-the-art recovery guarantees on the Haystack Model for different regimes of sample size and ambient dimension.  In particular, when the ambient dimension is fixed and the sample size is large enough, we show that our gradient method can exactly recover the underlying subspace for any fixed fraction of outliers (less than 1). 
\end{abstract}


\section{Introduction}

Robust subspace recovery (RSR) involves estimating a low-dimensional linear subspace in a corrupted dataset. It assumes that a portion of the given dataset lies close to or on a subspace, which we will refer to as the ``underlying subspace". The other portion of the dataset is assumed to be corrupted and may lie far from the underlying subspace.  In this regime, noise and corruption are separate entities: corruption refers to large and potentially arbitrary changes to a data point, while noise is a small perturbation of a data point.

A basic method for modeling data by a low-dimensional subspace is Principal Component Analysis (PCA)~\citep{jolliffe2002principal}. PCA is popular for both reducing noise and capturing low-dimensional structure within data. However, PCA is notoriously sensitive to corrupted data and does not perform well in many regimes of the RSR problem.

Many strategies have been proposed for the RSR problem, which are reviewed in~\citet{lerman2018overview}. However, despite nice progress, many of the proposed methods have inherent issues. Perhaps the largest flaw in many existing methods is computational time: many methods require calculation of a full covariance matrix or matrix inversion. If we have a dataset of $N$ points in $\reals^D$, these calculations typically have complexity $O(ND^2)$ or $O(D^3)$.
On the other hand, the PCA $d$-dimensional subspace can be calculated with complexity $O(NDd)$. Some recent proposals for RSR have complexity that scales like $O(NDd)$, but they either do not have satisfying theoretical guarantees or have extra user specified parameters.
Ideally, we would like algorithms that run in $O(NDd)$ because the set of $d$-dimensional subspaces has dimension $O(Dd)$.

The key point of our work is the development of a computationally efficient and provably accurate method for RSR. We desire a method that has complexity $O(NDd)$ and that does not sacrifice theoretical guarantees. The method we propose involves minimization of the robust least absolute deviations energy function. Minimizing this function involves solving a non-convex optimization problem that is NP-hard.

Even though the problem is NP-hard in general, we derive conditions that ensure the energy landscape is well-behaved in a substantial neighborhood of an underlying subspace. These conditions also ensure that a geodesic gradient descent method can locally recover an underlying subspace, and the convergence of this method is linear under some slightly stronger assumptions. This linear convergence implies that this $O(NDd)$ algorithm is very efficient. Furthermore, we give the most complete discussion of recovery under various regimes of corruption and relatively broad statistical models.
In particular, we show that our method is robust to very high percentages of corruption under special generative models.
To our knowledge, we give the strongest guarantees on a non-convex method for RSR to date and even obtain stronger results than some convex methods.

In the rest of this section, we give some necessary background for our method and an overview of this work. First, in Section~\ref{subsec:essback}, we give some essential background to understand our approach to RSR. Then, Section~\ref{subsec:contribution} outlines the main contributions of our work. Finally, Section~\ref{subsec:organization} summarizes the structure of the paper, and Section~\ref{subsec:notation} discusses necessary notation.

\subsection{Essential Background}
\label{subsec:essback}

Here, we briefly summarize the necessary background to understand the primary contribution of this work. First, for the rest of this paper, we assume a linear subspace setting, which means that we only consider underlying subspaces that are linear. We leave the case of affine subspaces to future work. For simplicity of discussion, we advocate centering by the geometric median for real data when the center is not known.

The essential problem of RSR is an optimization problem over the Grassmannian, which is the set of linear subspaces of a fixed dimension. The Grassmannian is a non-convex set, which makes optimization over it hard. Frequently, this leads to NP-hard or SSE-hard formulations~\citep{hardt2013algorithms,clarkson2015input}. In this paper, we denote by $G(D,d)$ the Grassmannian of linear $d$-dimensional subspaces in $\reals^D$ and refer to such subspaces as $d$-subspaces.

It is illuminating to first outline the PCA subspace problem, since it has a similar form to our methodology. The basic formulation for the PCA subspace problem can be cast as an optimization over $G(D,d)$. For a dataset $\cX = \{ \bx_1,\dots,\bx_N\} \subset \reals^D$ centered at the origin, the PCA $d$-subspace is the solution of the least squares problem
\begin{equation}
    \min_{L \in G(D,d)} \sum_{i=1}^N \|\bx_i - \bP_L \bx_i\|_2^2,
    \label{eq:pca}
\end{equation}
where $\bP_L$ denotes the orthogonal projection matrix onto the subspace $L$.
As has been noted in many past works, the least squares formulation is sensitive to corrupted data.

Some have attempted to make the PCA formulation robust by considering least absolute deviations:
\begin{equation}\label{eq:robopt}
    \min_{L \in G(D,d)} \sum_{i=1}^N \| \bx_i - \bP_L \bx_i\|_2.
\end{equation}
This optimization can be thought of as estimating a geometric median subspace. Some have tried to directly optimize this problem~\citep{ding2006r1,lerman2017fast}, while others have tried to solve convex relaxations of it~\citep{mccoy2011two,xu2012robust,zhang2014novel,lerman2015robust}.

The goal of this work is to directly analyze the energy landscape of~\eqref{eq:robopt} and guarantee that a non-convex gradient descent method  for this energy minimization can recover an underlying subspace. This gradient descent method leads to huge gains in speed over previous convex methods but does not sacrifice accuracy in subspace recovery.

In the RSR problem setup, it is common to refer to the uncorrupted portion of the dataset as inliers, which lie on or near the underlying subspace. The case where the inliers lie on the underlying subspace is referred to as the noiseless RSR setting, while the case where the inliers lie near the underlying subspace is referred to as the noisy RSR setting. The corrupted points in the dataset are referred to as outliers, which are assumed to lie somewhere in the ambient space.
Exact recovery in the noiseless setting refers to when a method outputs the underlying subspace exactly. Near recovery in the noisy setting refers to when a method outputs a good approximation of the underlying subspace, where the goodness of approximation depends on the noise level.

\subsection{Contribution of This Work}
\label{subsec:contribution}

As mentioned, the goal of this work is to recover an underlying subspace by directly optimizing the non-convex function in~\eqref{eq:robopt}. To motivate why such a procedure might work, Figure~\ref{fig:landscape1} demonstrates the landscape of the energy function in \eqref{eq:robopt} for two simulated datasets. The novelty of this paper consists of the following observation in these and certain other datasets: despite non-convexity, the energy landscape appears to exhibit basins of attraction around the underlying subspaces. In other words, the energy function decreases over $G(D,d)$ in the direction of $L_*$ within some neighborhood. Indeed, it appears that direct minimization of the energy in a local neighborhood would yield exact recovery of the underlying subspace. It is important to emphasize that this phenomenon is inherently local. Looking at the energy plots in Figure~\ref{fig:landscape1}, it appears that there may be other local minimizers far from the underlying subspace, and so proper initialization is quite important.

\begin{figure}[!t]
\centering
\includegraphics[width = .45\textwidth,trim=40 150 40 170,clip]{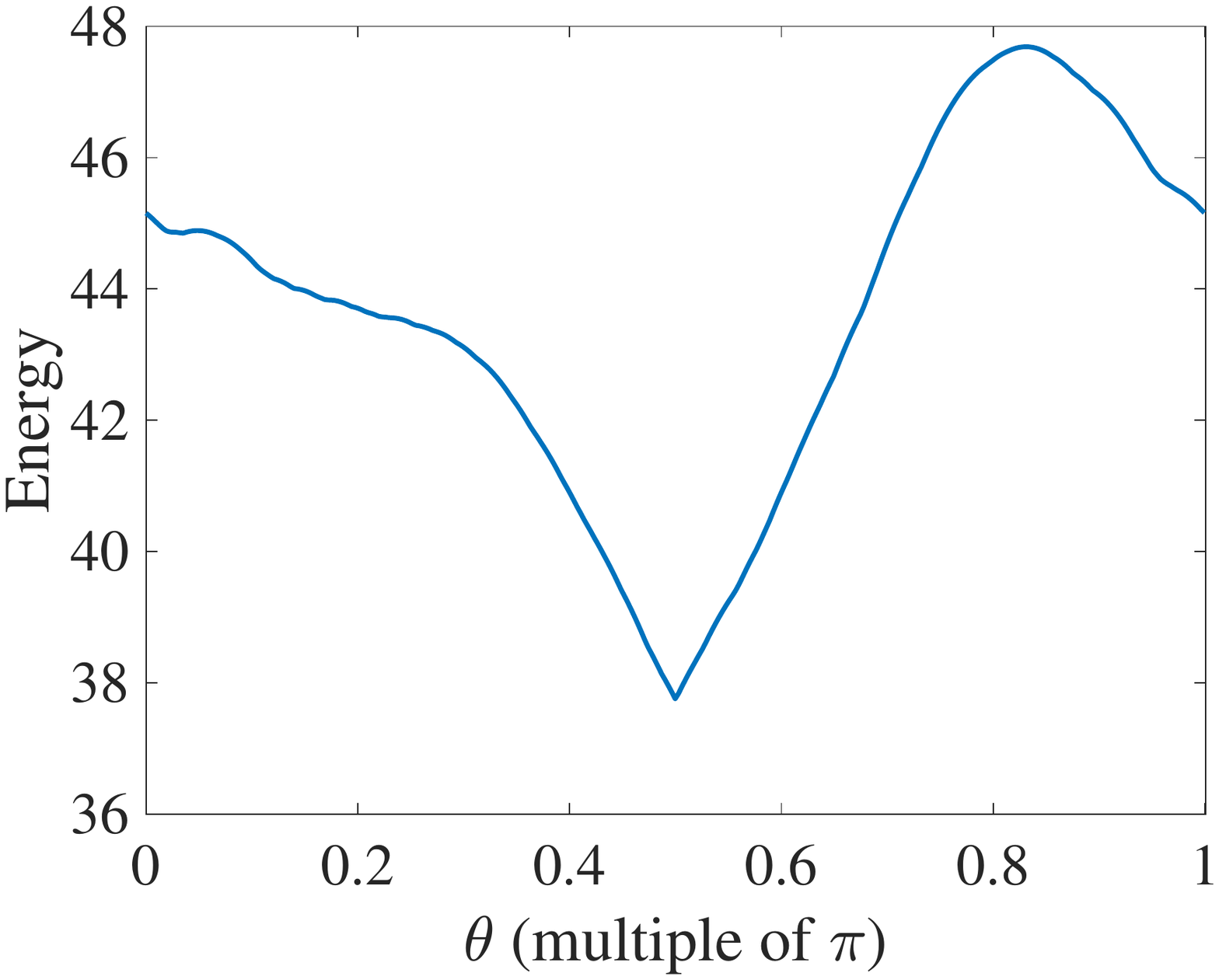}
\includegraphics[width = .45\textwidth,trim=40 150 40 170,clip]{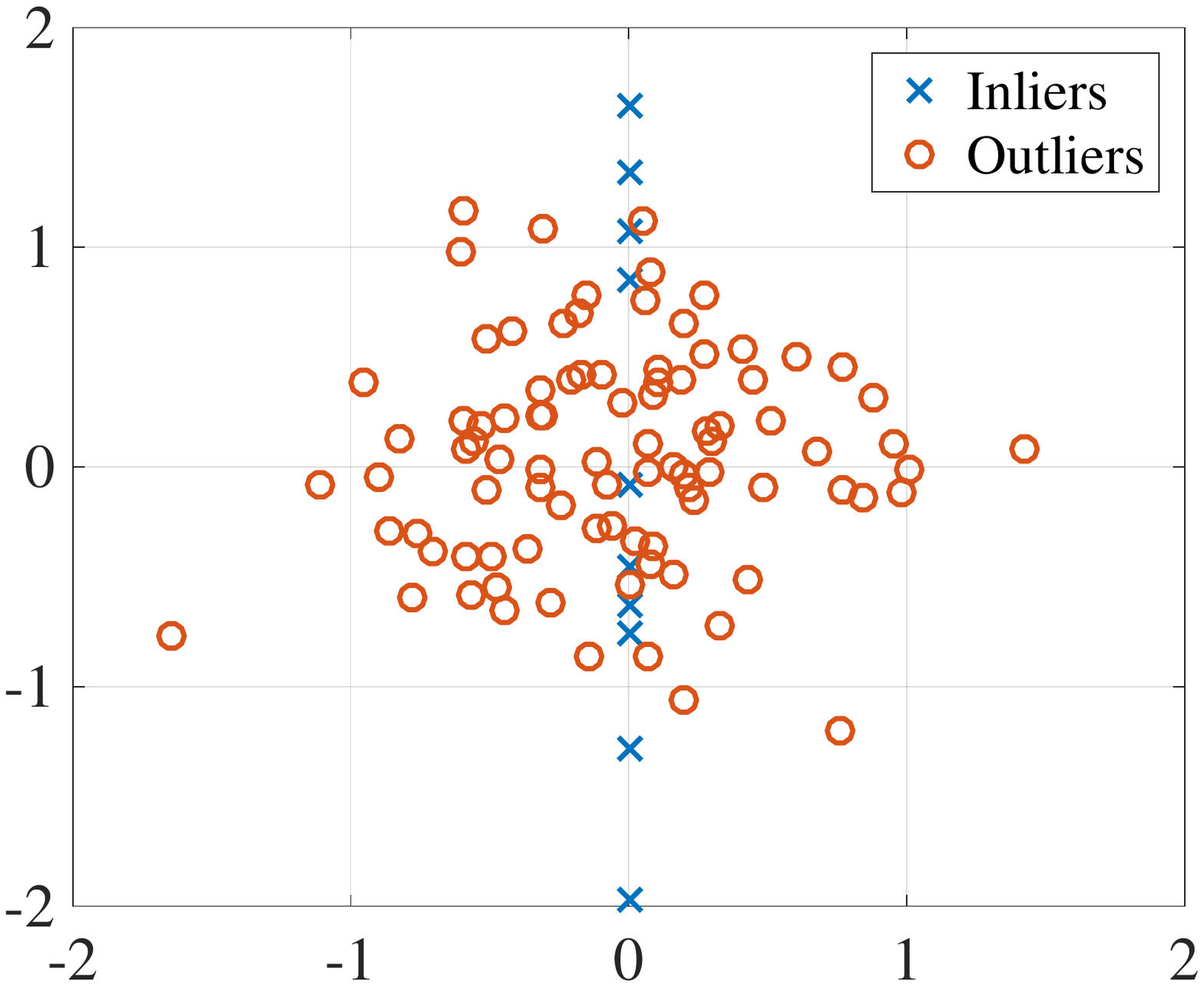}
\includegraphics[width = .45\textwidth,trim=40 150 40 170,clip]{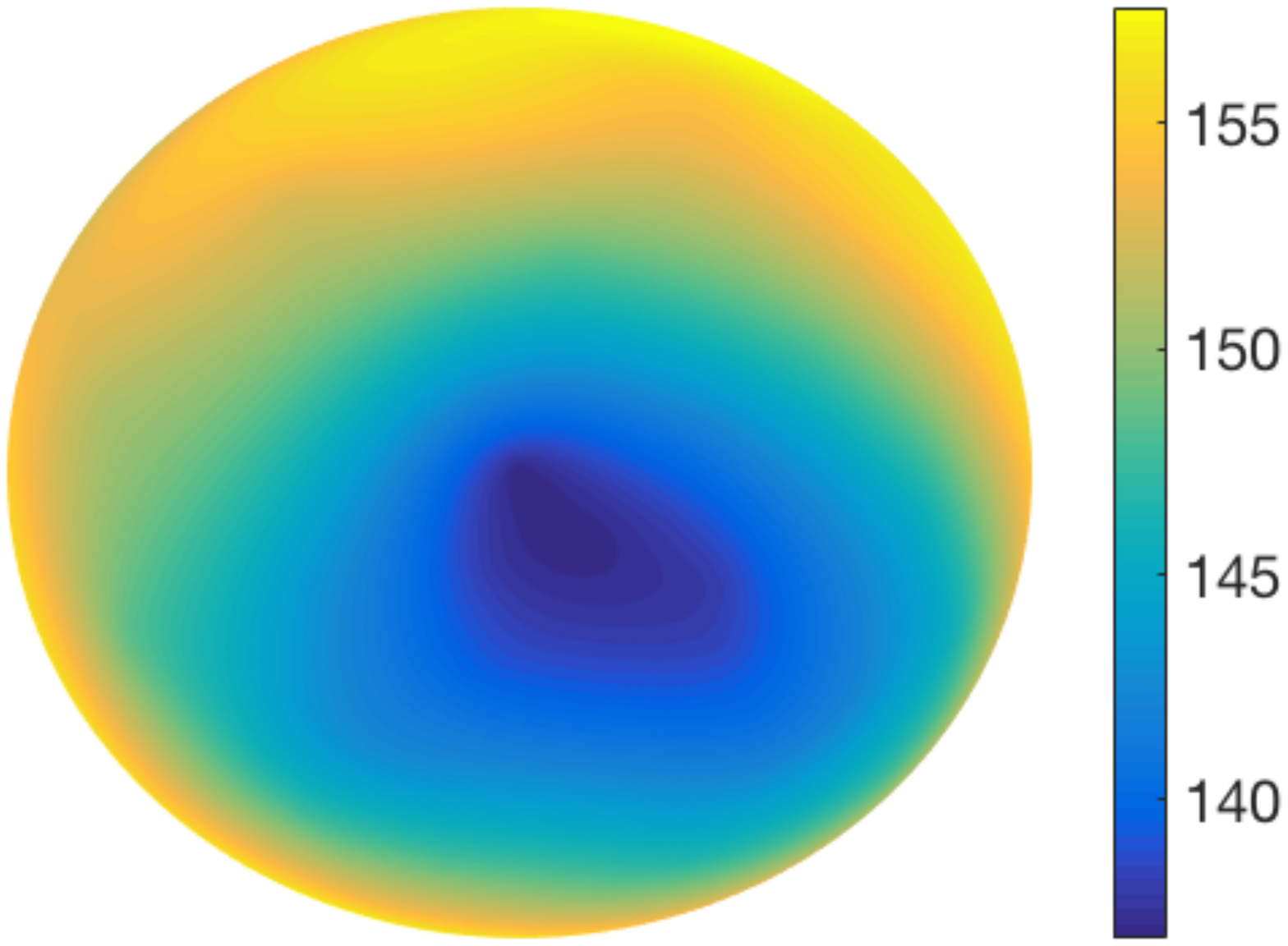}
\includegraphics[width = .45\textwidth,trim=40 150 40 170,clip]{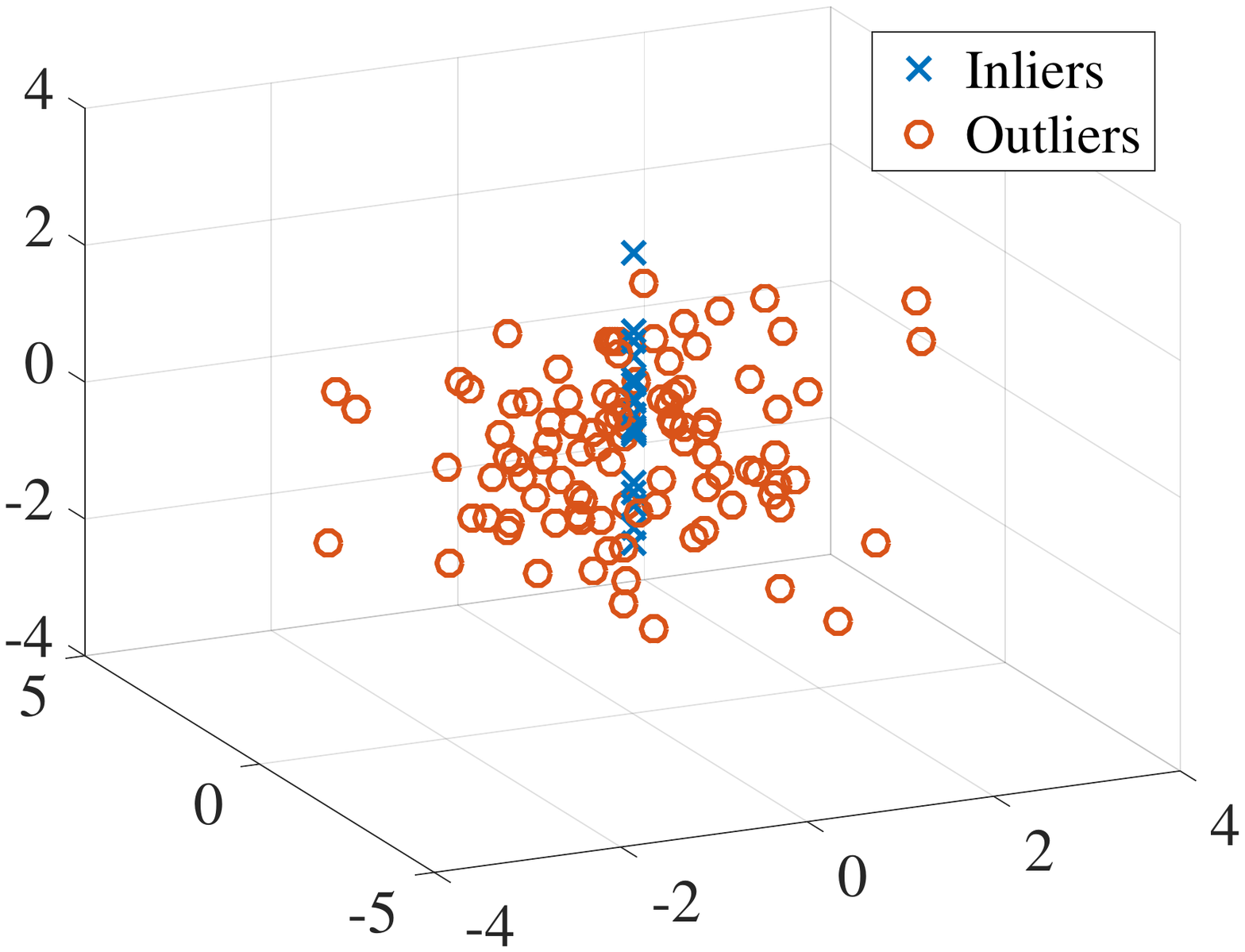}
\caption{Demonstration of the energy landscape of~\eqref{eq:robopt} over $G(2,1)$ and $G(3,1)$ with simulated Gaussian data. The simulated datasets are demonstrated on the right and the corresponding energy landscape is depicted on the left.
Top: In $\reals^2$, $90$ outliers are i.i.d.~$\cN(\bzero,\bI/2)$ and $10$ inliers are i.i.d.~$\cN(0,1)$ along the $y$-axis. $G(2,1)$ is identified with a semicircle and parameterized by angle. The energy is depicted as a function of the parametrizing angle. Its global minimum is at $\pi/2$, which corresponds to the underlying line at the $y$-axis of  $\reals^2$. Bottom: In $\reals^3$, $100$ outliers are i.i.d.~$\cN(\bzero,4\bI_3/3)$ and 20 inliers are i.i.d.~$\cN(0,1)$ along the $z$-axis. $G(3,1)$ is identified with the top hemisphere which is flattened to a circle. The energy function is depicted by a heat map on that circle. The global minimum is at the center of the circle, which corresponds to the underlying line at the $z$-axis of $\reals^3$.}\label{fig:landscape1}
\end{figure}

Our key contributions follow:
\begin{itemize}
\item  We show that, under deterministic conditions, the robust energy landscape of~\eqref{eq:robopt} exhibits basins of attraction around an underlying subspace as seen in Figure~\ref{fig:landscape1}. Theorems~\ref{thm:landscape} and~\ref{thm:landscapenoise} formulate this result for the energy landscape of the noiseless and noisy RSR settings respectively.
\item We propose a geodesic gradient descent algorithm in~Section~\ref{sec:grad}. Theorems~\ref{thm:conv} and~\ref{thm:linconv} show that this algorithm exactly recovers the underlying subspace under the conditions of Theorem~\ref{thm:landscape} with proper initialization. Theorem~\ref{thm:conv} guarantees sublinear convergence of this algorithm to the underlying subspace. With some additional, slightly stronger assumptions, Theorem~\ref{thm:linconv} guarantees linear convergence to the underlying subspace using a piecewise constant step-size scheme. These results can be generalized to near recovery in the noisy RSR setting as well (see Remark~\ref{remark:extendnoise}).
\item Lemma~\ref{lemma:pcainit} guarantees that we can initialize in the correct local neighborhood using PCA under a similar deterministic condition. This yields a complete guarantee for geodesic gradient descent with PCA initialization.
\item The deterministic guarantees are shown to hold for a variety of statistical models of data. In particular, we achieve competitive guarantees for recovery under the special Haystack Model. More specifically, we consider three different regimes of sample size. The first regime, $N = O(D)$, describes the scenario of a relatively small sample. The smallest ratio of inliers to outliers where exact recovery is still possible for this model among all algorithms is $d/(D-d)$, although this has only been established for algorithms of complexity $O(N D^2)$ at best~\citep{hardt2013algorithms,zhang2016robust}. We guarantee instead exact recovery for our $O(NDd)$ algorithm in this regime under a larger ratio of inliers to outliers, namely of order $O(d/\sqrt{D})$ (see Corollary \ref{cor:genhaystack}). In the regime of larger samples, when $N = O(D^2)$ , exact recovery with the even smaller ratio $O(d/\sqrt{D(D-d)})$ was previously obtained for a convex algorithm~\citep{zhang2014novel}. Theorem~\ref{thm:haystack} implies recovery with this same small fraction but in the regime  $N = O(d(D-d)^2\log(D))$. Except for this convex algorithm and our proposed algorithm, we are unaware of any other method that is guaranteed to obtain exact recovery under such a small fraction in the regime $N = O(d(D-d)^2\log(D))$. Beyond this, in the regime of very large samples, we show that our method can exactly recover an underlying subspace with any fixed fraction of outliers, where $N$ is at least some polynomial order of $D$ and also depends on the ratio of inliers to outliers. This is the only efficient RSR method with such a guarantee.
\end{itemize}

We will close this section by briefly commenting on how to read the theoretical results of this paper.
The first three bullets constitute the primary theoretical interest of this work, since they guarantee the usefulness of our non-convex RSR method under some fairly general conditions. In particular, the stability statistic developed in Section~\ref{sec:noiselessstat} lies at the core of much of our analysis, and many of the more complicated results are extensions of this statistic's analysis in Theorem~\ref{thm:landscape}. The discussion of statistical models as outlined in the fourth bullet above is merely included for interpretability of the general conditions we offer. With such models, we develop a heuristic understanding of these conditions, and, in particular, we gain insight into trade-offs between inliers and outliers. These models also allow for easier comparison of the theoretical guarantees of all RSR methods.

\subsection{Paper Organization}
\label{subsec:organization}

First, in~Section~\ref{sec:backandprevwork}, we review previous work on the RSR problem. In~Section~\ref{sec:theory}, we describe deterministic conditions that ensure that the energy landscape of~\eqref{eq:robopt} behaves nicely around an underlying subspace. Then, in Section~\ref{sec:grad}, we outline a geodesic gradient descent method on the set of subspaces. We show that this method can locally recover an underlying subspace for datasets satisfying the deterministic conditions and that the convergence rate is linear under some slightly stronger assumptions. We obtain exact recovery in the noiseless RSR setting and near recovery in the noisy RSR setting.
Then,~Section~\ref{sec:statmod} shows that the conditions hold for certain statistical models of data. In particular, we obtain almost state-of-the-art results on recovery under the Haystack Model and also consider a range of other models. Next, Section~\ref{sec:stabdemo} gives simulations that agree with the theoretical results of this paper. Finally, Section~\ref{sec:conclusion} concludes this work and discusses possible future directions.

\subsection{Notation}
\label{subsec:notation}

Before presenting the results of this paper, we explain our commonly used notation. The letter $L$ is used to refer to $d$-subspaces. We use bold upper case letters for matrices and bold lower case letters for column vectors. For a matrix $\bA$, $\Sp(\bA)$ is the subspace spanned by the columns of $\bA$, $\sigma_j(\bA)$ denotes the $j$th singular value of $\bA$, and, if $\bA$ is square, then $\lambda_j(\bA)$ denotes its $j$th eigenvalue. The spectral norm of $\bA$ is $\| \bA\|_2$, and the Euclidean 2-norm for vectors is denoted by $\| \cdot \|$. The notation $\widetilde{\bA}$ denotes projection of the columns of $\bA$ to the unit sphere.
For $d \leq D$, the set of semi-orthogonal $D \times d$ matrices is denoted by $O(D,d) = \{\bV \in \reals^{D \times d}: \bV^T \bV = \bI_d\}$. For $\bV \in O(D,d)$, we denote its columns by $\bv_1,\dots,\bv_d$. We recall that $G(D,d)$ denotes the Grassmannian, that is, the set of $d$-dimensional linear subspaces of $\reals^D$. The orthogonal projection matrix onto the subspace $L = \Sp(\bV)$ is denoted by $\bP_{L}$, and we interchangeably use $\bP_{\bV} = \bP_{L}$. The projection onto the orthogonal complement of $L$ is $\bQ_L = \bI - \bP_L$.
We denote the largest principal angle between two subspaces $L_1$ and $L_2$ by $\theta_1(L_1, L_2)$, that is, $\theta_1(L_1, L_2) = \arccos\left( \sigma_d (\bP_{L_1} \bP_{L_2})\right)$.
We say that an event occurs with high probability (w.h.p.) if the probability is bounded below by $1-O(N^a)$, for some absolute constant $a>0$. We say that an event occurs with overwhelming probability (w.o.p.) if the probability is bounded below by $1-O(e^{-aN^b})$, for absolute constants $a,b>0$.
The notation $f(x) \lesssim g(x)$ is used to denote that $f(x) < C g(x)$ for some absolute constant $C$ (and the notation $\gtrsim$ is used in the same way).

\section{Background and Review of Previous Work}
\label{sec:backandprevwork}

In this section, we will review past work on the RSR problem and necessary background concepts for this work. First, Section~\ref{sec:prevwork} discusses past attempts to solve the RSR problem. Then, Section~\ref{sec:prelim} gives the background concepts that are necessary to understand our later results.

\subsection{Review of Previous Work}
\label{sec:prevwork}

The most ubiquitous subspace modeling framework uses principal component analysis~\citep{jolliffe2002principal}.
Its optimization problem, which is formulated in \eqref{eq:pca}, is non-convex since $G(D,d)$ is non-convex.
However, despite its non-convexity, this problem has a direct solution, which is calculated from the singular value decomposition of the data matrix $\bX = [\bx_1,\dots,\bx_N]$.  This problem also has a nice energy landscape. Indeed, if the $d$th singular value of $\bX$ is larger than the $(d+1)$st singular value, then the global minimum is unique, and there are no other local minima; otherwise, if the $d$th singular value is equal to the $(d+1)$st, then all local minima are globally optimal. Saddle points are also guaranteed to be sufficiently far from the global minimum, and they have explicit expressions. We discuss the PCA energy landscape further in Appendix~\ref{app:pca}. These nice properties of the PCA subspace optimization are not shared by the algorithms for RSR that we discuss next.

Examples of works on RSR include~\citet{maronna2005principal,maronna2006robust,ding2006r1,zhang2009median,lerman2011robust,mccoy2011two,xu2012robust,coudron2012sample,hardt2013algorithms,zhang2014novel,goes2014robust,lerman2014lp,lerman2015robust,zhang2016robust,lerman2017fast,cherapanamjeri2017thresholding}.
A comprehensive overview of this topic is given in~\citet{lerman2018overview}.

We note that this problem is distinct from what is typically called robust PCA (RPCA) \citep{candes2011robust,zhou2010stable}, where the corruptions occur element-wise throughout the whole data matrix rather than some samples being wholly corrupted.
Algorithms for RPCA typically do not perform well in the RSR setting, and algorithms for RSR do not perform well in the RPCA setting.

RSR is inherently non-convex due to the non-convexity of $G(D,d)$. Robust versions of the PCA energy may have more complicated landscapes in general.
One way of making PCA robust is to simply project the data to the unit sphere, $S^{D-1}$, before running PCA~\citep{locantore1999robust,maronna2005principal,maronna2006robust}.
This deals with PCA's sensitivity to the scaling of the data and makes it easier to find directions that robustly capture variance. However, it is still not able to deal with correlated outlier directions and does not have good asymptotic guarantees even for simple models~\citep{lerman2017fast}.

As mentioned, another way to make PCA robust is to consider least absolute deviations~\citep{ding2006r1,mccoy2011two,xu2012robust,zhang2014novel,lerman2014lp,lerman2015robust,lerman2017fast}.
The first use of least absolute deviations in subspace modeling was the work on orthogonal regression by~\citet{osborne1985analysis}. This was not extended to general subspace modeling until much later~\citep{watson2001some,ding2006r1}. Previous works have considered convex relaxation of this energy~\citep{mccoy2011two,xu2012robust,zhang2014novel,lerman2015robust}. However, such convex relaxations are generally slow and may not approximate the underlying problem well. Indeed, most either have complexity $O(ND^2)$ or $O(D^3)$.

The works of~\cite{lerman2011robust,lerman2014lp} established under a certain model that an underlying subspace is recoverable by the minimizer of~\eqref{eq:robopt}. However, they did not provide a guaranteed algorithm for minimizing this energy. The estimates of these works do not hold for small sample sizes: they only hold for large $N$.
\citet{lerman2017fast} developed the FMS algorithm, which employs iteratively reweighted least squares to optimize~\eqref{eq:robopt}. However, the FMS algorithm does not have deterministic guarantees of fast convergence or deterministic results on recovery of the underlying subspace.
The FMS algorithm does have theoretical guarantees of approximate recovery for a very special model of data, with relatively large samples.
In contrast, we directly minimize~\eqref{eq:robopt} by gradient descent, and we provide deterministic guarantees of fast convergence and subspace recovery.

Another recent work on RSR was given by~\citet{cherapanamjeri2017thresholding}, where they propose Thresholding based Outlier Robust PCA (TORP). TORP has analysis for arbitrary outliers and noise, as long as the percentage of outliers is known in advance. While the tolerance to very low percentages of arbitrary outliers is not that impressive, the noise analysis is somewhat novel. Under Gaussian noise, the authors are able to show similar sample complexity as that of PCA. One downside of this algorithm is that one must know the percentage of outliers as an input. Further, since the guarantees are only for adversarial models of outliers, there is no discussion of improved estimates when the outliers are not adversarial but instead obey a specific statistical model.

In the existing literature, only a few methods achieve the complexity bound of $O(TNDd)$, where $T$ is iteration count. These include SPCA~\citep{maronna2005principal}, RANSAC and RandomizedFind~\citep{hardt2013algorithms,ariascastro2017ransac}, FMS~\citep{lerman2017fast}, and TORP~\citep{cherapanamjeri2017thresholding}. 
Among these, the algorithms either do not have sufficiently satisfying guarantees for recovery, or they do not have a good bound on $T$, or they require additional parameters. 
SPCA is the fastest out of these algorithms since it has $T=1$. Here, we are slightly abusing the complexities and assuming the cost of running PCA is $O(NDd)$, despite the fact that PCA is also an iterative algorithm. We choose this convention due to the fact that many methods use PCA as a sub-routine.  While SPCA is somewhat robust to arbitrary outliers, it cannot exactly recover subspaces in the presence of outliers. However, SPCA is nice since it is quite general and lacks the specialized assumptions of many methods.
RANSAC requires a user to input specialized parameters, such as the consensus number and a consensus threshold. RANSAC can also only bound $T$ in probability under certain conditions. This also goes for the analysis of RandomizedFind given by~\citet{ariascastro2017ransac}, along with their updated algorithm that has complexity $O(TDd)$. In many cases, though, this $T$ can be very large. For both the RANSAC and RandomizedFind methods, recovery guarantees exist in the noiseless RSR setting under specialized assumptions, but there are no satisfying extensions of either method to noise. 
On the other hand, TORP~\citep{cherapanamjeri2017thresholding} requires a user to input the percentage of outliers that is not known in general. TORP has a guarantee of linear convergence under certain conditions, but, as we mentioned earlier, it does not have satisfying guarantees for subspace recovery.
FMS~\citep{lerman2017fast} only has guarantees for rate of convergence and recovery for very special models of data.

One way to compare the theoretical guarantees of various methods is to assume a statistical model of data and then determine which algorithm performs best in this model.
For example, one common choice of model in past works was the Haystack Model, which can be seen in~\citet{lerman2015robust}.
Another model was to assume spherically symmetric outliers, and inliers spherically symmetric on an underlying subspace~\citep{lerman2014lp,lerman2017fast}.
Others have examined models with arbitrary outliers~\citep{xu2012robust,cherapanamjeri2017thresholding}. In this work, after giving our general theoretical guarantees, we will show how they can be applied to a variety of statistical models of data. 

This paper also fits in to the surge of recent work that has focused on non-convex optimization for many structured data problems~\citep{dauphin2014identifying,hardt2014understanding,jain2014iterative,ge2015escaping,lee2016gradient,arora2015simple,mei2018landscape,ge2016matrix,boumal2016nonconvexphase,sun2015nonconvex,SunQuWright_nonconvex_sphere_2015,ma2018implicit}. Some work has focused on non-convex optimization for robust PCA~\citep{netrapalli2014non,yi2016fast,zhang2017robust}, which is a related but different problem than RSR. Others have attempted to solve non-convex versions of the RSR problem~\citep{lerman2017fast,cherapanamjeri2017thresholding}, but, as we have discussed, these methods each have their own shortcomings.

This work is partially built on optimization on manifolds, and in particular there are important results on optimization over the Grassmannian manifold~\citep{edelman1999geometry,absil2004riemannian}.~\citet{edelman1999geometry} develop gradient descent on the Grassmannian and give formulations for Newton's method and conjugate gradient for the Grassmannian. We discuss optimization on the Grassmannian in more detail in the next section.

Many other recent works have also focused on using optimization on the Grassmannian to solve various problems~\citep{zhang2009median,goes2014robust,thomas2014learning,zhang2016global,ye2016schubert,lim2016statistical}.
The work of~\citet{zhang2016global} examines a rank one geodesic gradient scheme for solving online PCA. Their setting is distinctly different from ours since they attempt to solve the PCA problem rather than RSR. They also only prove recovery of the PCA solution for a specific model of Gaussian noise, and no deterministic condition for global recovery is given.
Further, while we assume centered data in this paper, \cite{thomas2014learning} and~\cite{lim2016statistical} consider estimation on the affine Grassmannian.

\subsection{Review of Optimization over $G(D,d)$}
\label{sec:prelim}

The minimization in~\eqref{eq:robopt} involves optimization over the Grassmannian manifold. To understand the energy landscape, one must have a basic understanding of the geometry of $G(D,d)$ and how to calculate derivatives over it.

We can write the energy function in~\eqref{eq:robopt} in two equivalent ways. First, as a function over $G(D,d)$, we write
\begin{equation}\label{eq:energygrass}
    F(L;\cX) = \sum_{\cX } \|\bQ_{L} \bx_i \| .
\end{equation}
On the other hand, we can represent points in $G(D,d)$ by equivalence classes of points in $O(D,d)$. For any $\bV \in O(D,d)$, the subspace $\Sp(\bV)$ can be represented by the equivalence class $[\bV] = \{ \bV \bR : \bR \in O(d,d) \}$. For $\bV \in O(D,d)$, the energy~\eqref{eq:energygrass} is equivalent to
\begin{equation}
    \label{eq:robenergyunreg}
    F(\bV;\cX) = \sum_{i=1}^N \|(\bI -\bV \bV^T)\bx_i \| .
\end{equation}
While both formulations are equivalent, we use~\eqref{eq:energygrass} to formulate geodesic derivatives over $G(D,d)$ and the coordinate representation in~\eqref{eq:robenergyunreg} to calculate gradients.
In the following, the geodesic derivative of~\eqref{eq:energygrass} will be used to characterize the local landscape, and the gradient of~\eqref{eq:robenergyunreg} will be used to analyze the performance of the gradient descent algorithm we discuss later in Section~\ref{sec:grad}.

One can measure the distance between subspaces in $G(D,d)$ using the principal angles. For a discussion of principal angles between subspaces, see Appendix \ref{sec:grassgeo}. Denoting the largest principal angle between $L_0$ and $L_1$ by $\theta_1(L_0,L_1)$, we can define a metric on $G(D,d)$ by $\dist(L_0,L_1) = \theta_1(L_0,L_1)$. We then define a ball on the metric space $G(D,d)$ by
\begin{equation*}
B(L,\gamma) = \{L' \in G(D,d): \ \theta_1(L',L) < \gamma\}.
\end{equation*}
We say that an element of $O(D,d)$ lies in the ball $B(L,\gamma)$ if the subspace spanned by its columns lies in $B(L,\gamma)$.

In the following, we frequently use a construction for geodesics on the Grassmannian. For a review of this construction and necessary terminology, see Appendix \ref{sec:grassgeo} or~\S3.2.1 of~\cite{lerman2014lp}. Suppose that the interaction dimension between $L_1$ and $L_2$ is $k$, that is, $k = d-\dim(L_0 \cap L_1)$. Let $\theta_1,\dots,\theta_k$ be the nonzero principal angles for $L_0$ and $L_1$ (in decreasing order), and let the respective principal vectors for $L_0$ and $L_1$ be $\bv_1,\dots,\bv_k$ and $\by_1,\dots,\by_k$. Finally, let $\bu_1,\dots,\bu_k$ be a complementary orthogonal basis for $L_1$ with respect to $L_0$. We can use these to parameterize a geodesic $L(t)$ with $L(0) = L_0$ and $L(1) = L_1$, where the formula is given in~\eqref{eq:grassgeo} of Appendix \ref{sec:grassgeo}. Then, following~\citet{lerman2014lp} and~\citet{lerman2017fast}, we can calculate the directional geodesic subderivative of~\eqref{eq:energygrass} at $L_0$ in the direction of $L_1$:
\begin{equation}\label{eq:geoderiv}
	\frac{\di}{\di t} F(L(t);\cX) \Big|_{t=0} = -\sum_{\|\bQ_{_0} \bx_i\| > 0 } \frac{\sum_{j=1}^k \theta_j (\bv_j^T \bx_i) (\bu_j^T \bx_i)}{\|\bQ_{L_0} \bx_i\|}.
\end{equation}
A subderivative of~\eqref{eq:robenergyunreg} with respect to $\bV$ is
\begin{equation}\label{eq:derF}
   \frac{\partial}{\partial \bV} F(\bV;\cX) = - \sum_{\|\bQ_{\bV} \bx_i\|>0} \frac{\bx_i \bx_i^T \bV}{\|\bQ_{\bV} \bx_i \|}.
\end{equation}

The definition of subderivative and subdifferential as we use them are given next.
For a more in depth discussion of these concepts, see, for example, \citet{clarke1990optimization} and~\citet{ledyaev2007nonsmooth}. In both of the derivatives~\eqref{eq:geoderiv} and~\eqref{eq:derF}, the sum is taken over all points that do not lie in $\Sp(\bV)$. This restriction is what makes them both subderivatives. For any general function $f(x)$, a subderivative of $f$ at $x_0$ is any number in the subdifferential $\partial f(x_0)$. In turn, the subdifferential of $f$ at $x_0$ is the set of all numbers between the one-sided derivatives of $f$ at $x_0$. For~\eqref{eq:geoderiv}, the subdifferential is defined to be the set of all numbers between
\begin{equation*}
    a = \lim_{t \to 0^-} \frac{F(L(t);\cX) - F(L(0);\cX)}{t} \text{ and }  b=\lim_{t \to 0^+} \frac{F(L(t);\cX) - F(L(0);\cX)}{t}.
\end{equation*}
In other words, the subdifferential is $[\min(a,b),\max(a,b)]$, which is the set of all instantaneous tangent slopes at $L(0)$.
For the other case of~\eqref{eq:derF}, for any entry of $\bV$, $\bV_{ij}$, let $\bDelta$ be the matrix of all zeros except $\bDelta_{ij} = 1$. Then, the subdifferential of $F(\bV;\cX)$ for $\bV_{ij}$ is all numbers between
\begin{equation*}
	a_{ij} = \lim_{t \to 0^-} \frac{F(\bV + t\bDelta;\cX) - F(\bV;\cX)}{t} \text{ and } b_{ij}=\lim_{t \to 0^+} \frac{F(\bV + t\bDelta;\cX) - F(\bV;\cX)}{t}.
\end{equation*}
This can be generalized to any direction $\bDelta$ with $\| \bDelta \|_F = 1$, where the subdifferential is the convex hull of the one sided derivatives.
We say that the subdifferential is less than a number if all of its elements are bounded above by that number, that is,
\begin{equation}
    \partial F(L(t);\cX)|_{t=0} < M \iff a < M \ \forall \ a \in  \partial F(L(t);\cX)|_{t=0}.
\end{equation}

To gain an intuition for these concepts, we display a visualization of the derivative and subdifferential for a simulated energy landscape in Figure~\ref{fig:diffsubdiff}. The derivative follows the standard definition from calculus on manifolds and is just the slope of the tangent line displayed on the left in Figure~\ref{fig:diffsubdiff}. On the other hand, at points where the function $F(L(t);\cX)$ is non-smooth at $t=0$, we use the subdifferential instead. The extreme slopes for the subdifferential are displayed on the right in Figure~\ref{fig:diffsubdiff}.
\begin{figure}[!t]
\centering
\includegraphics[width = .45\textwidth]{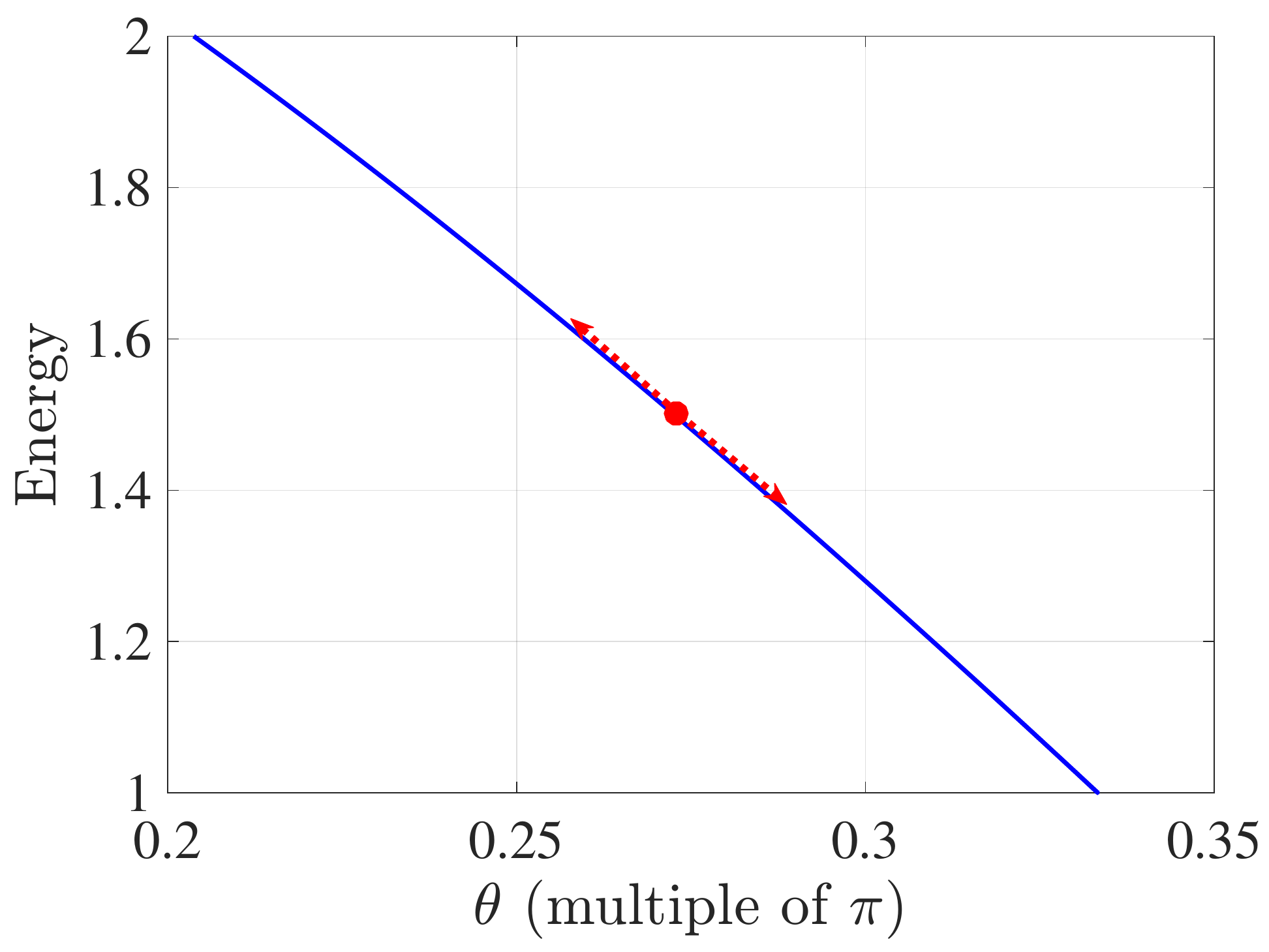}
\includegraphics[width = .45\textwidth]{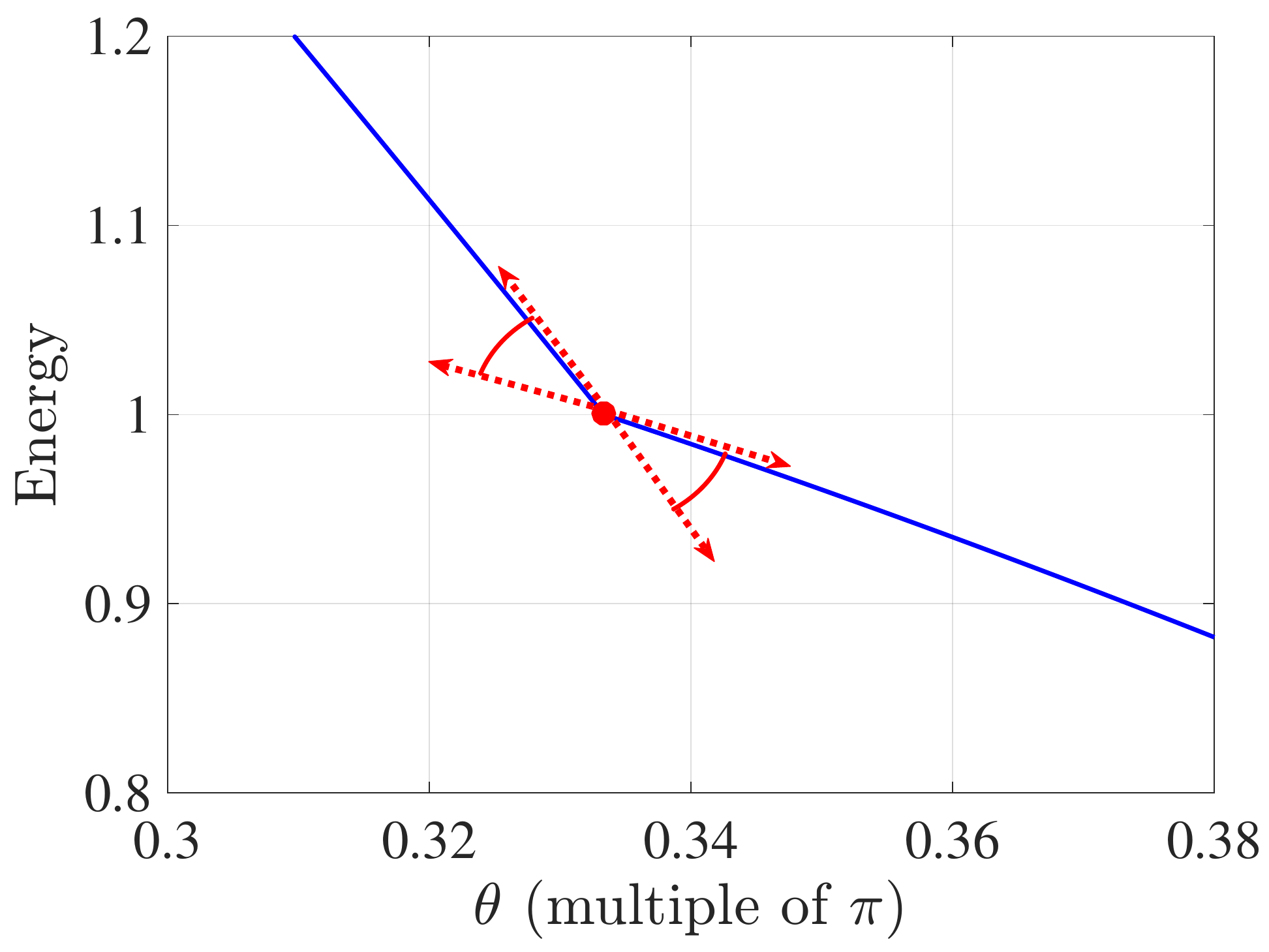}
\caption{Demonstration of the derivative and subdifferential of the energy in \eqref{eq:robopt}.
We assume that $d=1$, $D=2$ and identify $G(2,1)$ with the unit circle. The energy function in \eqref{eq:robopt} is thus parameterized  by angle and its graph is similar to the one in the top left sub-figure of  Figure~\ref{fig:landscape1}. In both images, the slope of the red two-sided arrow represents the magnitude of the directional geodesic derivative or subderivative over $G(2,1)$, which is the same as derivative or subderivative with respect to the representing angle. The first image shows a differentiable energy function on the given domain. In the second image, there is a value where the energy function is not differentiable. In this case, we use the subdifferential, which is the set of slopes of all lines between the acute angles formed by the two red two-sided arrows. Note that this subdifferential is bounded above by a negative number. We later prove that this property generally holds for the energy function of \eqref{eq:robopt} under certain conditions.}\label{fig:diffsubdiff}
\end{figure}

In future sections, to save space, we will write the sums in~\eqref{eq:geoderiv} and~\eqref{eq:derF} as $\sum_{\cX}$ and leave the condition $\|\bQ_{\bV} \bx_i\|>0$ as implied.
Following Section 2.5.3 of~\cite{edelman1999geometry}, to respect the geometry of the Grassmannian, the (sub)gradient of~\eqref{eq:robenergyunreg} is defined as
\begin{equation}\label{eq:grad}
  \nabla F(\bV;\cX) = \bQ_{\bV} \frac{\partial}{\partial \bV} F(\bV;\cX).
\end{equation}

\section{A Well-Tempered Landscape for Least Absolute Deviations}
\label{sec:theory}

We assume a dataset $\cX = \{ \bx_1,\dots,\bx_N\} \subset \reals^D$  that can be partitioned into corrupted (outlier) and uncorrupted (inlier) parts.
We refer to $\cX$ as an inlier-outlier dataset, where in the coming sections, we will more rigorously define this notion in the noiseless and noisy RSR settings. We denote the subsets of inliers and outliers in $\cX$ as $\cX_{\mathrm{in}}$ and $\cX_{\mathrm{out}}$, respectively. The corresponding data matrices for $\cX_{\mathrm{in}}$ and $\cX_{\mathrm{out}}$ are $\bX_{\mathrm{in}}$ and $\bX_{\mathrm{out}}$, where columns represent data points.

As stated previously, the basic problem of RSR is to recover the subspace $L_*$ from an inlier-outlier dataset.
In the noiseless setting one can try to exactly recover this subspace, and in the noisy setting one may try to approximately recover it. In the latter case, this means that we wish to estimate it up to a specified approximation error, which may depend on the level of noise in the inliers.
In order for this problem to be well-defined, basic assumptions must be made.
Indeed, if all inliers lie at the origin, then any subspace would be a solution to the RSR problem. This issue, among others, was extensively discussed in \S III-A of~\citet{lerman2018overview}. Our theoretical results for recovery will depend on a condition formulated later in this section that ensures the problem is well-defined.

First,~Section~\ref{sec:landstat} discusses some statistics that play a fundamental role in our analysis.
Next,~Section~\ref{subsec:landscapestab} uses these statistics to develop the deterministic conditions that ensure the energy landscape of~\eqref{eq:robopt} behaves nicely around an underlying subspace in both the noiseless and noisy RSR settings.

\subsection{Landscape Statistics}
\label{sec:landstat}

Equipped with the notions laid out in Section~\ref{sec:prelim}, we are ready to define some important statistics for the landscape of~\eqref{eq:robopt}.  These statistics are inspired by those originally discussed in \citet{lerman2015robust}, and they are later used for our stability results in Theorems~\ref{thm:landscape} and~\ref{thm:landscapenoise}. We first discuss the noiseless RSR setting in Section~\ref{sec:noiselessstat} and then the noisy RSR setting in Section~\ref{subsec:noisystat}.

\subsubsection{The Noiseless RSR Setting}
\label{sec:noiselessstat}

For the noiseless setting, we assume that the inliers, $\cX_{\mathrm{in}} \subset \cX$, lie on a low-dimensional linear subspace $L^* \in G(D,d)$, and the rest of the points, $\cX_{\mathrm{out}} = \cX \setminus \cX_{\mathrm{in}}$, are in $\reals^D \setminus \{ L^* \}$. We call $\cX$ defined in this way a noiseless inlier-outlier dataset.

The permeance of the inliers in a noiseless inlier-outlier dataset is defined as
\begin{equation}\label{eq:perm}
   \cP(\cX_{\mathrm{in}}) = \lambda_d \left( \sum_{\bx \in \cX_{\mathrm{in}}} \frac{\bx_i \bx_i^T }{\|\bx_i\|} \right).
\end{equation}
Here, $\lambda_d(\cdot)$ denotes the $d$th eigenvalue of a matrix. Notice that large values of $\cP$ ensure that the inliers are well-distributed. In other words, they permeate throughout $L_*$.

We also define an alignment statistic for the noiseless inlier-outlier dataset $\cX$. With some abuse of notation, we write $\nabla F(L;\cX)$ to refer to the gradient with respect to some basis of $L$, where the choice of basis does not matter.
The alignment statistic of a set of outliers with respect to a subspace is
\begin{equation}\label{eq:align}
\cA(\cX_{\mathrm{out}},L) =  \left\| \nabla F(L;\cX_{\mathrm{out}}) \right\|_2.
\end{equation}
It is not hard to show that \eqref{eq:align} is invariant with respect to choice of basis for $L$.
In effect, if this term is always small, then the outliers are not concentrated in any low-dimensional space. In other words, they are not aligned.
In our later analysis, we use a simple and illuminating bound for $\cA$:
\begin{align}\label{eq:alignmentbd}
   \cA(\cX_{\mathrm{out}},L) &\leq\sqrt{N_{\mathrm{out}}} \| \bX_{\mathrm{out}} \|_2.
\end{align}
The derivation for this bound is left to Appendix \ref{sec:alignbd}. We note that this bound may be tight, but in most cases it is not.

We will define a stability statistic for a neighborhood of $L_*$, ${B(L_*, \gamma)}$. This neighborhood depends on a parameter $\gamma$, which fixes the maximum principal angle of subspaces in this neighborhood with $L_*$. Using the permeance and alignment defined in~\eqref{eq:perm} and~\eqref{eq:align}, the stability statistic of a noiseless inlier-outlier dataset is
\begin{align}\label{eq:stab}
    \cS(\cX,L_*,\gamma) &= \cos(\gamma)  \cP(\cX_{\mathrm{in}}) - \sup_{L \in {B(L_*,\gamma)} }  \cA(\cX_{\mathrm{out}}, L).
\end{align}
The simple condition required in most of our theoretical analysis is $\cS(\cX,L_*,\gamma) > 0$. This essentially means that the amount that the inliers permeate the underlying subspace must be able to beat the alignment of the outliers with respect to any subspace.

Note that $\cS(\cX, L_*, 0) = \cP(\cX_{\mathrm{in}}) - \cA(\cX_{\mathrm{out}}, L_*)$ is a tighter stability condition than the one in (2.4) of~\citet{lerman2015robust}. Indeed, the stability expression of~\citet{lerman2015robust}, takes the form
\begin{equation}
    \cS_{\mathrm{REAP}}(\cX, L_*) = \frac{1}{4 \sqrt{d}} \cP_{\mathrm{REAP}}(\cX_{\mathrm{in}}) - \cA_{\mathrm{REAP}} (\cX_{\mathrm{out}}, L_*).
\end{equation}
Here, $\cP_{\mathrm{REAP}}$ and $\cA_{\mathrm{REAP}}$ are actually lower and upper bounds on the permeance and alignment defined in~\eqref{eq:perm} and~\eqref{eq:align}, respectively. This, together with the extra factor of $1/(4\sqrt{d})$, means $\cS_{\mathrm{REAP}}(\cX, L_*)$ is not as tight as $\cS(\cX, L_*, \gamma)$ for  $\gamma=0$ . The upside is that the REAPER alignment only needs to be examined at a single point $L_*$, whereas $\cS(\cX, L_*, \gamma)$ becomes hard to estimate as $\gamma$ increases. It is not clear in general which statistic is tighter when $\gamma > 0$.

\subsubsection{The Noisy RSR Setting}
\label{subsec:noisystat}

The noisy setting occurs when the inliers lie near the low-dimensional subspace rather than exactly on it. In this case, we need to be more careful with the statistics of our inlier points. For each inlier point, we write $\bx_i = \bP_{L_*} \bx_i + \bepsilon_i$, where $\bP_{L_*} \bx_i \in L_*$ and $\bepsilon_i \in L_*^\perp$ is added noise. Then, $\cX^{\mathrm{dns}} = (\bP_{L_*} \cX_{\mathrm{in}}) \cup \cX_{\mathrm{out}}$ is the corresponding noiseless inlier-outlier dataset (here, the $\cdot^{\mathrm{dns}}$ superscript stands for ``de-noised").
We assume that the noise in our data is uniformly bounded by $\epsilon$, that is, $\|\bx_i - \bP_{L_*} \bx_i \| < \epsilon$ for all $\bx_i \in \cX_{\mathrm{in}}$.

Some small technical issues come up with noisy RSR datasets that make the conditions harder to interpret. However, the following discussion is just a generalization of the previous section on the noiseless case after dealing with these technicalities.

To write the stability statistic in the noisy RSR setting, we must define the following set-valued functions of $\cX_{\mathrm{in}}$.
These are defined for a unit vector $\bw \in L_* \cap S^{D-1}$ and small-projection cutoff $\delta$. They are meant to distinguish between inliers who have a projection onto $\bw$ with length bigger than $\delta$, and inliers who have a projection  onto $\bw$ with length less than or equal to $\delta$. These functions are defined as
\begin{align}
    \cF_0(\cX_{\mathrm{in}}, \bw, \delta) & = \{ \bx \in \cX_{\mathrm{in}} : |\bw^T \bx| \leq \delta \},
    \label{eq:noiseinout}
    \\
    \cF_1(\cX_{\mathrm{in}}, \bw, \delta) & = \{ \bx \in \cX_{\mathrm{in}} : |\bw^T \bx| > \delta \}.
    \label{eq:noiseinin}
\end{align}
Inliers in the first set are coined ``small-projection inliers", and inliers in the latter set are coined ``large-projection inliers".

With these sets, our noisy inlier-outlier stability statistic is
\begin{align}\nonumber
    \cS_n(\cX,L_*, \epsilon, \delta, \gamma) &=    \frac{\cos(\gamma-2\arctan(\epsilon/\delta))}{2} \min_{\bw \in L_* \cap S^{D-1}} \left(\sum_{ \bx_i \in \cF_1(\cX_{\mathrm{in}}, \bw, \delta) } \frac{\bw^T \bP_{L_*} \bx_i \bx_i^T \bP_{L_*} \bw}{\| \bP_{L_*} \bx\| + \sqrt{\epsilon^2 + \delta^2}} \right) \\
                                             &- \sqrt{\delta^2 + \epsilon^2}\max_{\bw \in L_* \cap S^{D-1}} \#\left( \cF_0(\cX_{\mathrm{in}},\bw,\delta) \right)  -\sup_{{B(L_*,\gamma)}} \cA(\cX_{\mathrm{out}}, L).
    \label{eq:stabn}
\end{align}
This statistic is somewhat similar to what we had in the noiseless RSR setting, although now we have separated our inlier terms into two parts.
The first term behaves like the permeance from the noiseless RSR setting, with the addition that small-projection inliers are trimmed.
The last term is again the alignment of the outliers.
The middle term is quite technical, and it is meant to capture cases when inliers may have large angle with a fixed direction of $L_*$.
If we take $\delta \to 0$ and $\epsilon/\delta \to 0$, then the stability almost becomes our original stability, with the added factor of 1/2 on the permeance term. If $\cS_n(\cX,L_*, \epsilon, \delta, \gamma) > 0$, then we will demonstrate later that recovery is possible up to accuracy $\eta = 2 \arctan(\epsilon / \delta)$.

An illustration of the small-projection cutoff $\delta$, the noise bound $\epsilon$, and the accuracy $\eta$ for noisy inliers is given in Figure~\ref{fig:delta} to help ease understanding of our statistic.  For simplicity, we show the case of a one-dimensional subspace in $\reals^2$. Here, the vectors $\bw \in L_* \cap S^{D-1}$ are $(1,0)^T$ or $(-1,0)^T$. Since these two vectors are equivalent for the two functions in~\eqref{eq:noiseinout} and~\eqref{eq:noiseinin}, the large-projection inliers and small-projection inliers are only determined by the magnitude of $\delta$.
Thus, the cutoff defined by a certain choice of $\delta$ in~\eqref{eq:noiseinout} and~\eqref{eq:noiseinin} corresponds to separating small and large inliers by their $x$-value.

\begin{figure}[!t]
\centering
\includegraphics[width = .45\textwidth]{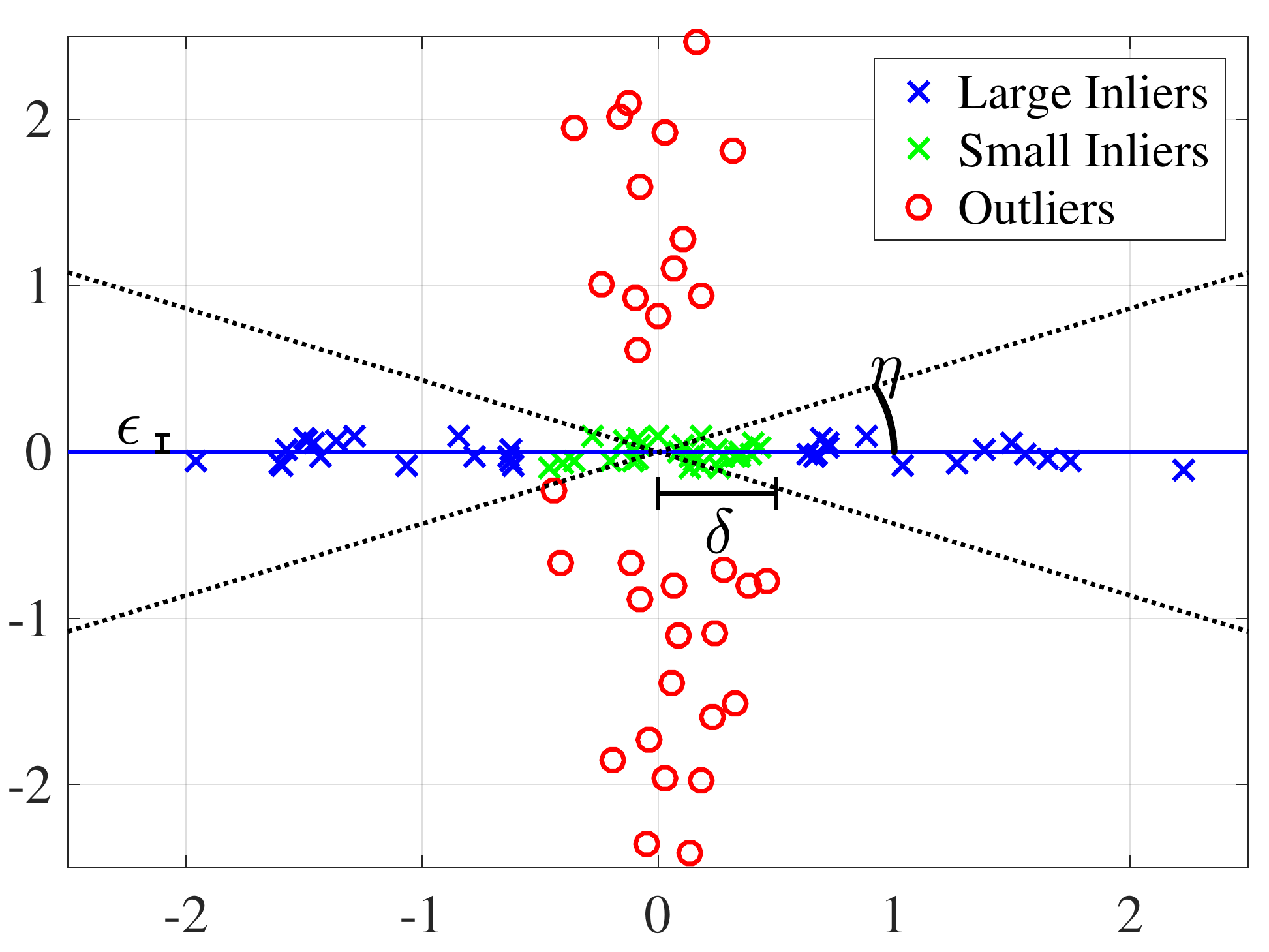}
\caption{Demonstration of a noisy inlier-outlier dataset, where $d=1$ and $D=2$. The following parameters used in our analysis are also displayed: the small-projection cutoff $\delta$, the noise bound $\epsilon$, and the accuracy $\eta$. Note that, in this example, inliers within a $\delta$-neighborhood of the origin are removed from the permeance calculation, and $\epsilon$ the maximum distance of the inliers to the $X$-axis. Our analysis guarantees recovery up to the accuracy $\eta$, which means that we would recover the $x$-axis within the acute angle formed by the dotted lines.}\label{fig:delta}
\end{figure}

We note that the statistic is by no means tight and future work should analyze how accurate these methods can be with noise.
As we will discuss in Section~\ref{sec:conclusion}, one could also study RSR in settings with high noise, such as heavy tailed noise or under the spiked model.
The main point of the noisy statistic is to show that our results yield $\epsilon/\delta$-approximate recovery when the noise is uniformly bounded by $\epsilon$. Here, $\delta$ is constrained in that the stability condition, $\cS_n(\cX,L_*, \epsilon, \delta, \gamma) > 0$, must hold.

\subsection{The Local Landscape under Stability}
\label{subsec:landscapestab}

In this section, we will give results that prove the local stability of the energy landscape of~\eqref{eq:robopt}. We begin with the theorem for the noiseless RSR setting in~Section~\ref{subsec:noiselessland}, and then prove an analogous result for the noisy RSR setting in~Section~\ref{sec:noisyland}.

\subsubsection{Stability in the Noiseless Case}
\label{subsec:noiselessland}

We show that positivity of the stability statistic given in~\eqref{eq:stab} with $0<\gamma<\pi/2$ implies stability of $L_*$ as a minimizer in the neighborhood ${B(L_*, \gamma)}$. Stability of $L_*$ means that it is the only critical point and minimizer in ${B(L_*, \gamma)}$, and, at all other points in this neighborhood, there exists a direction in $G(D,d)$ such that the energy landscape looks like one of the two cases displayed in Figure~\ref{fig:diffsubdiff}; in other words, there is a direction of decrease.

\begin{theorem}[Stability of $L_*$]\label{thm:landscape}
	Suppose that a noiseless inlier-outlier dataset 
with an underlying subspace $L_*$ satisfies $\cS(\cX, L_*, \gamma) > 0$, for some $0<\gamma < \pi/2$. Then,
 all points in ${B(L_*, \gamma)} \setminus \{L_*\}$ have a subdifferential along a geodesic strictly less than $-\cS(\cX, L_*, \gamma)$, that is, it is a direction of decreasing energy. This implies that $L_*$ is the only critical point and local minimizer in ${B(L_*,\gamma)}$.
\end{theorem}

\begin{proof}[Proof of Theorem~\ref{thm:landscape}]
    The main point behind the proof is the following statement. We show that, for any $L \in {B(L_*,\gamma)} \setminus \{L_*\}$, there is a geodesic $L(t)$ with $L(0) = L$, and an open interval around 0, $\cI = (\theta_1(L,L_*)-\gamma,\delta(L))$, for some $\delta(L) > 0$, such that
    \begin{align}
    	 &\text{For } t \in \cI, \ \theta_1(L(t),L_*) \text{ is a strictly decreasing function;} \label{eq:geocloser}\\
        &\text{For } t \in \cI, \ F(L(t); \cX) \text{ is a strictly decreasing function.}\label{eq:Fdecr}
    \end{align} 
	In simple words, the function $F(L(t);\cX)$ is decreasing as $L(t)$ moves closer to $L_*$. This implies that $L_*$ is the only critical point and minimizer in $B(L_*, \gamma)$ by a perturbation argument, which we will explicitly state at the end of the proof.

Fix a subspace $L \in {B(L_*,\gamma)} \setminus \{L_*\}$, and let the principal angles between $L$ and $L_*$ be $\theta_1,\dots,\theta_d$.
Also, choose a set of corresponding principal vectors $\bv_1,\dots,\bv_d$ and $\by_1,\dots,\by_d$ for $L$ and $L_*$, respectively, and let $l\geq 1$ be the maximum index such that $\theta_1 = \dots = \theta_l$. We let $\bu_1, \dots,\bu_l$ be complementary orthogonal vectors for $\bv_1,\dots,\bv_l$ and $\by_1,\dots,\by_l$. For $t \in [0,1]$, we form the geodesic
\begin{equation}\label{eq:specgeo}
	L(t) = \Sp(\bv_1\cos(t) + \bu_1 \sin(t),\dots,\bv_l\cos(t) + \bu_l \sin(t),\bv_{l+1},\dots,\bv_d).
\end{equation}
Notice that this geodesic is parameterized by arclength in terms of the metric defined in Section~\ref{sec:prelim}.
This geodesic moves only the $l$ furthest directions of $L(0)$ towards $L_*$ and fixes those directions that are closer than $\theta_1$. This geodesic certainly satisfies~\eqref{eq:geocloser}. We have also removed dependence on $\theta_1$, since this unnecessarily impacts the magnitude of the geodesic subderivative~\eqref{eq:geoderiv}.
Following~\eqref{eq:geoderiv} with no dependence on $\theta_1, \dots, \theta_l$ and then using \eqref{eq:derF} and~\eqref{eq:grad}, we have
\begin{align}\label{eq:grassgeoderbd}
	\frac{\di}{\di t} F(L(t);\cX) &\Big|_{t=0} = -\sum_{j=1}^l  \sum_{\cX} \frac{\bv_j^T \bx_i \bx_i^T \bu_j}{\|\bQ_{L} \bx_i\|}  \\ \nonumber
	&= - \sum_{j=1}^l \left(\sum_{\cX_{\mathrm{in}} } \frac{\bv_j^T \bx_i \bx_i^T \bu_j}{\|\bQ_{L} \bx_i\|} +\sum_{\cX_{\mathrm{out}}} \frac{\bv_j^T \bx_i \bx_i^T \bu_j}{\|\bQ_{L} \bx_i\|} \right)\\ \nonumber
	&\leq - \sum_{j=1}^l \left( \frac{\cos(\theta_1) \sin(\theta_1)}{\sin(\theta_1)} \sum_{\cX_{\mathrm{in}} } \frac{\by_j^T \bx_i \bx_i^T \by_j}{\|\bx_i\|} - \sup_{\bV \in {B(L_*,\gamma)}} \left\| \nabla F(\bV;\cX_{\mathrm{out}}) \right\|_2 \right) \\ \nonumber
	&\leq - l \cS(\cX, L_*, \gamma) \leq - \cS(\cX, L_*, \gamma) < 0.
\end{align}
The third line is obtained by noting that the inliers and the vectors $\by_j$, $1 \leq j \leq l$, are contained $L_*$, the vectors $\bv_j$, $1 \leq j \leq l$, are contained in $L$, and $\theta_1=...=\theta_l$. Therefore, for all $1 \leq j \leq l$ and all inliers $\bx_i$,  $\bv_j^T \bx_i \bx_i^T \bu_j \geq \cos(\theta_1) \sin(\theta_1) (\by_j^T \bx_i)^2$. In the third line we also used~\eqref{eq:derF} and~\eqref{eq:grad} for the outlier term and maximized it over $B(L_*, \gamma)$.
Thus,~\eqref{eq:grassgeoderbd} implies that every subspace in $\overline{B(L_*,\gamma)} \setminus \{L_*\}$ has a direction with negative local subderivative.

The argument above was for a special subderivative at $L(0)$. We now extend this argument to show that this implies that the subdifferential at every point is bounded above by $-\cS(\cX, L_*, \gamma)$. On the one hand, if $L(0)$ contains no points of $\cX$, then the subderivative corresponds to the actual derivative, and we therefore have that $F(L(t);\cX)$ is decreasing, as desired.

On the other hand, assume that $L(0) \cap \cX$ is non-empty. Since $\cX$ is finite, there exists a neighborhood of 0, $(-\omega, \omega)$, for $0 < \omega < \pi/2$, such that $L(t)$ has empty intersection with $\cX \setminus L(0)$, that is,  
\begin{equation}\label{eq:ltnointersect}
    \{\bx \in \cX : \bx \in L(t),\text{ for some } t \in (-\omega,\omega)\} = L(0) \cap \cX.
\end{equation} 
Notice that the changing directions in~\eqref{eq:specgeo} form a geodesic along $G(D,l)$ while the other directions are fixed (i.e., $L(t) \supset \Sp(\bv_{l+1}, \dots, \bv_d)$ for all $t$). Thus, derivatives of $F(L(t); \cX)$ over $G(D,d)$ are equivalent to derivatives of $F(\tilde L(t); \bQ_{\bv_{l+1},\dots,\bv_d}\cX)$ over $G(D,l)$, where $\tilde L(t) = \Sp(\bv_1\cos(t) + \bu_1 \sin(t),\dots,\bv_l\cos(t) + \bu_l \sin(t)))$. Here, $\bQ_{\bv_{l+1}, \dots, \bv_d}$ is the projection onto the orthogonal complement of $\Sp(\bv_{l+1}, \dots, \bv_d)$, which is applied to all points of $\cX$. By~\eqref{eq:ltnointersect}, this geodesic over $G(D,l)$ has the property 
$$\tilde L(t) \cap \cX = \emptyset, \ \forall \ t \in (-\omega,0) \cup (0,\omega).$$
Consequently, $F(\tilde L(t); \bQ_{\bv_{l+1},\dots,\bv_d}\cX)$ is continuously differentiable on $(-\omega,0)$ and $(0,\omega)$ as a function over $G(D,l)$,   and it is apparent that the bounds in~\eqref{eq:grassgeoderbd} also hold for all of these derivatives. Putting these facts together, by the continuity of the derivatives of $F(\tilde L(t); \bQ_{\bv_{l+1},\dots,\bv_d}\cX)$ for $t \in (-\omega,0) \cup (0,\omega)$, the subdifferential of $F(L(t);\cX)$ at $t=0$ is bounded above by $- \cS(\cX, L_*, \gamma)$, which in turn implies that~\eqref{eq:Fdecr} holds. 

Finally, there are clearly no other critical points than $L_*$ in $B(L_*, \gamma)$, because every point has a direction of decrease. To show that~\eqref{eq:geocloser} and~\eqref{eq:Fdecr} imply that $L_*$ is a local minimizer, consider a one-dimensional perturbation of $L_*$, $L'$. In other words, $\theta_1(L_*,L') >0$ and $\theta_2(L_*,L')=0$. Then,~\eqref{eq:geocloser} and~\eqref{eq:Fdecr} imply that if $L(t)$ is the geodesic between $L'$ and $L_*$, then $F(L(t);\cX)$ is decreasing for $t \in (\theta_1(L',L_*)-\gamma,\theta_1(L',L_*))$. The more general perturbation case is just an extension of this argument. Indeed, a $d$-dimensional perturbation may be written as a sequence of one-dimensional perturbations.
\end{proof}

\subsubsection{Stability with Small Noise}
\label{sec:noisyland}

The following theorem generalizes Theorem~\ref{thm:landscape} for the case of a noisy inlier-outlier dataset. The proof of this theorem is left to Appendix~\ref{sec:landscapenoiseproof}, and it essentially follows that of the noiseless RSR setting with the altered stability statistic. Here, we only guarantee that there is a large region with no critical points up to a precision of $\eta = 2 \arctan(\epsilon / \delta)$, which is determined by the noise level and inlier permeance.

\begin{theorem}[Stability of $B(L_*,\eta)$ with Noise]\label{thm:landscapenoise}
    Assume a noisy inlier-outlier dataset, with an underlying subspace $L_*$ and noise parameter $\epsilon > 0$, that satisfies for some $\delta > \epsilon > 0$ the stability condition $\cS_n(\cX,L_*, \epsilon, \delta,\gamma) > 0$. Let $\eta = 2 \arctan(\epsilon / \delta)$ and assume  further that $\eta < \gamma$.
    Then, all points in ${B(L_*,\gamma) \setminus B(L_*,\eta)}$ have a subdifferential along a geodesic strictly less than $-\cS_n(\cX,L_*, \epsilon, \delta,\gamma)$, that is, it is a direction of decreasing energy. This implies that the only local minimizers and saddle points in ${B(L_*,\gamma)}$ are in $B(L_*,\eta)$.
\end{theorem}

\section{A Geodesic Gradient Method for RSR and its Guarantees}
\label{sec:grad}

In this section, we discuss a geodesic gradient descent method for minimizing~\eqref{eq:energygrass}. First, Section~\ref{sec:graddesc} gives the details of our algorithm. After laying out the algorithm, we discuss convergence to a local minimizer under deterministic conditions in Section~\ref{subsec:convergence}. Then, Section~\ref{subsec:pcainit} shows that the PCA $d$-subspace is a good initializer under a separate deterministic condition.

\subsection{Minimization by Geodesic Gradient Descent}
\label{sec:graddesc}

We use gradient descent to minimize~\eqref{eq:robenergyunreg}. Denoting the singular value decomposition of the negative gradient by $ -\nabla F(\bV;\cX) = \bU \bSigma \bW^{T} $, Theorem 2.3 of~\cite{edelman1999geometry} states that the geodesic starting at $\bV(0) = \bV$ with $\frac{\di}{\di t} {\bV}(t)|_{t=0} = -\nabla F(\bV;\cX)$ is
\begin{equation}\label{eq:gradgeo}
    \bV(t) = \bV \bW \cos(\bSigma t) \bW^{T} + \bU \sin(\bSigma t) \bW^{T}.
\end{equation}
Here, $\sin$ and $\cos$ are the typical matrix $\sin$ and $\cos$, defined by their corresponding power series.

We develop a geodesic gradient descent method using the construction in~\eqref{eq:gradgeo}. At a point $\bV^k$, we may choose a value of $t$ and move along the geodesic to the next iterate. For a sequence of step-sizes $(t^k)_{k \in \nats}$, the sequence of subspaces is defined recursively by $\bV^{k+1} = \bV^k(t^k)$. The full algorithm with a specific choice of step-size is given in Algorithm~\ref{alg:ggd} and is referred to as GGD. The complexity of this algorithm is $O(TNDd)$, where $T$ is the number of iterations. Note that we use a piecewise constant scheme for the step-sizes, which starts with step-size $t^1 = s$, and every $K$ iterations the step-size is shrunk by a factor of $1/2$.

\begin{algorithm} [ht!]
    \caption{RSR by Geodesic Gradient Descent (GGD)}\label{alg:ggd}
	\begin{algorithmic}[1]
        \State \textbf{Input:} dataset $\cX$, subspace dimension $d$, initial step-size $s$, tolerance $\tau$, constant step interval length $K$
        \State \textbf{Output:} $\bV^* \in O(D,d)$, whose columns span the robust subspace
      \State $\bV^1 = PCA(\bX,d)$
      \State $k=1$, $s = 1$
      \While{$\theta_1(\bV^{k},\bV^{k-1}) > \tau$ or $k=1$}
      	\State Compute $\nabla F(\bV^k; \cX)$ by \eqref{eq:derF} and \eqref{eq:grad}
	    \State Compute the SVD $\bU^k \bSigma^k \bW^k = -\nabla F(\bV^k;\cX)$
      	\State $s^k = s / 2^{\lfloor k / K \rfloor}$
        	\State $\bV^{k+1} = \bV^k \bW^k \cos(\bSigma^k s^k) \bW^{kT} + \bU^k \sin(\bSigma^k s^k) \bW^{kT}$ \Comment{\eqref{eq:gradgeo}}
        	\State $k = k+1$
      \EndWhile
	\end{algorithmic}
\end{algorithm}

The next section provides convergence guarantees for GGD with both the piecewise constant step-size of Algorithm \ref{alg:ggd} and with step-size $s/\sqrt{k}$. We later compare these two examples of step-size on a simulated dataset in Figure~\ref{fig:conv_sim}. The factor of $1/2$ in line 8 is called the shrink-factor and could be chosen to be any fraction. Later, in our experiments, we display the convergence properties of GGD with two choices of this factor.

\subsection{Local Convergence of GGD Under Stability}
\label{subsec:convergence}

In this section, we give convergence guarantees for GGD. Theorem~\ref{thm:conv} shows that the convergence is sublinear under the deterministic conditions of Theorem~\ref{thm:landscape} for step-size $t^k = s/\sqrt{k}$. Then, Theorem~\ref{thm:linconv} shows that the convergence of Algorithm~\ref{alg:ggd} is linear under a slightly stronger assumption. These results are all for the noiseless RSR setting. However, all of these results can be extended to the noisy RSR setting in a simple fashion using the notions described in Section~\ref{subsec:noisystat}.

As a reminder, Theorem~\ref{thm:landscape} implies that if $\cS(\cX, L_*, \gamma)>0$ in the noiseless RSR setting then $L_*$ is the only limit point in ${B(L_*,\gamma)}$ for GGD. In other words, there is no need to worry about saddle points or non-optimal critical points in this neighborhood of $L_*$. In the following theorem, we give a general sublinear convergence bound for GGD in the local neighborhood considered in Theorem~\ref{thm:landscape}. This implies that the algorithm can exactly recover the underlying subspace in the noiseless RSR problem. The step-size used in this theorem is $t^k = s/\sqrt{k}$ at iteration $k$ instead of the piecewise constant scheme seen in Algorithm~\ref{alg:ggd}. The proof of this theorem is left to Appendix~\ref{app:suppconvres}.

\begin{theorem}[Noiseless Sublinear Convergence]\label{thm:conv}
    Suppose that $\cX$ is an inlier-outlier dataset with an underlying subspace $L_*$. Suppose also that there exists $0< \gamma<\pi/2$ such that $\cS(\cX,L_*,\gamma)>0$ and that the initial GGD iterate is $\bV^1 \in {B(L_*,\gamma)}$. Then, for sufficiently small $s$ as input (which may depend on $\cS(\cX,L_*,\gamma)$, $d$, $D$, $N_{\mathrm{in}}$, and $N_{\mathrm{out}}$), modified GGD with $t^k = s/\sqrt{k}$ converges to $L_*$ with rate $\theta_1(L_k,L_*) < C/\sqrt{k}$, for some constant $C$.
\end{theorem}

While the rate $O(1/\sqrt{k})$ matches typical results in non-smooth optimization, faster convergence is desirable.
With the piecewise constant step-size given in Algorithm~\ref{alg:ggd} and a further deterministic condition, GGD linearly converges to the underlying subspace $L_*$.
\begin{theorem}[Noiseless Linear Convergence]\label{thm:linconv}
    Suppose that $\cX$ is an inlier-outlier dataset with an underlying subspace $L_*$. Suppose also that there exists $0< \gamma<\pi/2$ such that $\cS(\cX,L_*,\gamma)>0$ and that the initial GGD iterate is $\bV^1 \in {B(L_*,\gamma)}$. Assume further that
     \begin{align} \label{eq:gradconds}
        \inf_{L \in B(L_*,\gamma) \setminus \{L_*\} } \frac{1}{4} \left| \frac{\di}{\di t}F(L(t);\cX)|_{t=0} \right| &>  \sup_{L \in B(L_*, \gamma) \setminus \{ L_*\}} \sum_{ \cX \cap L } 2 \| \bx_i \|,
    \end{align}
    where $L(t)$ is a geodesic parameterized by arclength from $L$ through $L_*$.
    Then, for sufficiently large $K$ and sufficiently small $s$ as input (which may depend on $\cS(\cX,L_*,\gamma)$, $d$, $D$, $N_{\mathrm{in}}$, and $N_{\mathrm{out}}$), the sequence generated by Algorithm~\ref{alg:ggd} converges linearly to $L_*$.
\end{theorem}

\begin{remark}\label{remark:extendnoise}
    As mentioned earlier, Theorems~\ref{thm:conv} and~\ref{thm:linconv} can also be extended to the noisy RSR setting with more complicated statements. Indeed, these extensions follow the same ideas as Theorem~\ref{thm:landscapenoise} (which extends Theorem~\ref{thm:landscape} to the noisy RSR setting).
\end{remark}

We remark that the restriction in~\eqref{eq:gradconds} can be weakened, although it results in a more complicated theorem statement: for clarity, we show the simpler version in the theorem statement.  We refer to~\eqref{eq:gradconds} as the strong gradient condition.
In general, the sum in the right hand side of~\eqref{eq:gradconds} only contains a few points when the inliers and outliers are not too linearly dependent. For example, consider the case of inliers lying in general position within the subspace and outliers lying in general position. This is the case if all $D$-subsets of $\cX_{\mathrm{out}}$ are linearly independent and all $d$-subsets of $\cX_{\mathrm{in}}$ are linearly independent. Under this assumption, the right hand side contains less than $2(d-1)$ points. On the other hand, the left hand side may be sufficiently large for a wide range of statistical models of data.  For more interpretation of this condition for a specific model of data, we point the reader to Lemma~\ref{lemma:linconvmod}, where we demonstrate a simple way in which the condition might hold.

We are not able to give an explicit bound on the rate of convergence factor, which depends on certain statistical characteristics of the data. While we do not currently have estimates of this factor and it may depend on $N$, $D$, and $d$, we expect this dependence to not be too bad, especially based on our numerical experiments. 

\begin{proof}[Proof of Theorem~\ref{thm:linconv}]
    The proof of this theorem is a consequence of the following lemmas. The proof of these lemmas is deferred to Appendix~\ref{sec:lemmaproof}.

The first lemma locally bounds above the increase in cost around $L_*$.
    \begin{lemma}\label{lemma:decrbd}
        If $\cS(\cX, L_*, \gamma) > 0$, then
\begin{align}
    F(L;\cX)-F(L_*;\cX)< 2 \theta_1(L,L_*)  \sum_{\cX_{\mathrm{in}}} \|\bx_i\| , \ \forall L \in B(L_*,\gamma).
\end{align}
\end{lemma}
Notice that $\theta_1(L,L_*)$ is a measure of distance between $L$ and $L_*$. The next lemma bounds the magnitude of the gradient in $B(L_*, \gamma) \setminus{L_*}$.
\begin{lemma}\label{lemma:gradnormbd}
    If $\cS(\cX, L_*, \gamma) > 0$, then, for some $C_1 > 0$ (that depends on the data),
    \begin{align}
    \left| \frac{\di}{\di t}F(L(t);\cX)|_{t=0} \right| &> C_1 \sum_{\cX_{\mathrm{in}}} \| \bx_i \| , \ \forall L \in B(L_*, \gamma) \setminus{L_*},
\end{align}
where $L(t)$ is a geodesic parameterized by arclength from $L$ through $L_*$.
\end{lemma}
Finally, the third lemma bounds the decrease in cost between consecutive iterates.
\begin{lemma}\label{lemma:applip}
    If $L_k \in B(L_*,\gamma)$ and~\eqref{eq:gradconds} holds, then there exists $c_0>0$ such that for each step-size choice $t^k=c\theta_1(L_k,L_*)$ with $c<c_0$,
    \begin{align}
        F(L_{k})-F(L_{k+1})) \geq \frac{c\theta_1(L_k,L_*)}{2} \left| \frac{\di}{\di t}F(L(t);\cX)|_{t=0} \right|,
    \end{align}
    where $L(t)$ is a geodesic parameterized by arclength from $L_k$ through $L_*$.
\end{lemma}

Choosing the step-size $t^k = c \theta_1(L_k, L_*)$, with $c$ coming from Lemma~\ref{lemma:applip}, and combining the results of Lemmas~\ref{lemma:decrbd},~\ref{lemma:gradnormbd} and~\ref{lemma:applip}, we find that
\begin{align} \label{eq:combinelemmas}
    F(L_{k+1};\cX) - &F(L_{*};\cX) < F(L_{k};\cX) - F(L_{*};\cX)  - \frac{c\theta_1(L_k,L_*)}{2} \left| \frac{\di}{\di t}F(L(t);\cX)|_{t=0} \right| \\ \nonumber
    & \leq F(L_{k};\cX) - F(L_{*};\cX)  - \frac{C_1}{2}c \theta_1(L_k, L_*)  \left( \sum_{\cX_{\mathrm{in}}} \| \bx_i \| \right) \\ \nonumber
    & \leq F(L_{k};\cX) - F(L_{*};\cX)  - \frac{C_1 c}{4} \left( F(L_k;\cX) - F(L_*; \cX) \right) \\ \nonumber
    & \leq (1-C_2) (F(L_k; \cX) - F(L_*; \cX)) ,
\end{align}
where $0 < C_2 < 1$. Here, $(1-C_2)$ is the rate of convergence factor, which depends on $d$. Thus, if one could choose the step-size $t^k = c \theta_1(L_k, L_*)$, then the sequence of costs $F(L_k; \cX)$ would converge linearly to $F(L_*;\cX)$.

For all $L \in B(L_*, \gamma) \setminus \{ L_*\}$, letting $a= \theta_1(L, L_*)$, we find that
\begin{align}\label{eq:mingrad}
	F(L;\cX) - F(L_*; \cX) &= \int_{t=0}^{a} \frac{\di}{\di t} F(L(t); \cX) \di t \geq a \inf_{L' \in B(L_*, \gamma)} \left|\frac{\di}{\di t} F(L'(t); \cX)\right| \\ \nonumber
	&\geq a \cS(\cX, L_*, \gamma),
\end{align}
where $L(t)$ is the geodesic parameterized by arclength from $L_*$ through $L$ and $L'(t)$ is the geodesic parameterized by arclength from $L_*$ through $L'$. The last inequality follows from the argument in Theorem~\ref{thm:landscape}, and in particular from~\eqref{eq:grassgeoderbd}.
Thus, for $C_3 = 1/\cS(\cX, L_*, \gamma)$, where $0< C_3 < \infty$,~\eqref{eq:mingrad} implies that
\begin{equation}
    \theta_1(L_{k+1}, L_*) \leq C_3 (F(L_{k+1}; \cX) - F(L_*; \cX)).
\end{equation}
This means that linear convergence of the energy sequence, $(F(L_k; \cX))_{k \in \nats}$, gives linear convergence of the iterates, $(L_{k})_{k \in \nats}$.

So far, we have shown that there exists a $c$ such that choosing step-size $t^k = c \theta_1(L_k, L_*)$ leads to linear convergence of $L_k$ to $L_*$. However, this choice of step-size is purely theoretical, since in practice we would not know $\theta_1(L_k, L_*)$ at each iteration. We must now rectify this choice of step-size with that used in Algorithm~\ref{alg:ggd}.

Suppose a constant step-size $s$ and a constant $c$ satisfying the above argument. Then, the sequence $(L_k)_{k \in \nats}$ at least converges linearly to an element of the set $B(L_*, s/c)$ because~\eqref{eq:combinelemmas} holds as long as $t^k = s < c \theta_1(L_k, L_*)$. If instead the constant step-size is $s/2$, we will get linear convergence to an element of the set $B(L_*, s/2c)$, albeit at a slower rate. Notice that, at each shrinking step (i.e., switching to step-size $s/2$ from $s$), if the sequence has already reached $B(L_*, s/c)$, we can bound the rate of convergence factor $(1-C_2)$ by
\begin{equation}
\left(1 - \frac{s/2}{c\theta_1(L_k, L_*)}\frac{C_1c}{4} \right) \leq \left(1 - \frac{C_1c}{8} \right).
\end{equation}
Further, for all subsequent steps where $s/2 < c \theta_1(L_k, L_*)$, we have that the rate of convergence factor is strictly less than $1 - C_1c /(8 \sqrt{d} (d+1))$. Thus, as long as the time between shrinking is large enough, we are guaranteed to have linear convergence. If the first step-size is allowed to run long enough, we find that the number of steps $m$ between shrinking needs to be at most
\begin{equation*}
    \left(1 - \frac{C_1c}{8} \right)^m < \frac{1}{2}.
\end{equation*}
\end{proof}

\subsection{Complete Guarantee with PCA Initialization}
\label{subsec:pcainit}

Notice that the previous section does not give a complete guarantee because the results are local. In other words, they assume that we first initialize in $B(L_*, \gamma)$ and then run GGD. To make these results practical, we will show that it is possible to initialize in this neighborhood under a simple deterministic condition. As before with Theorems~\ref{thm:conv} and~\ref{thm:linconv}, we consider only the noiseless RSR setting here. Extensions to the noisy case do not require much more effort.

The result of this section shows that PCA initializes in $B(L_*, \gamma)$ under a similar deterministic condition to $\cS(\cX, L_*, \gamma)>0$.
This is quantified in the following lemma.
\begin{lemma} \label{lemma:pcainit}
Suppose that, for a noiseless inlier-outlier dataset,
\begin{equation}\label{eq:pcainit}
    \sqrt{2} \sin(\gamma) \lambda_d(\bX_{\mathrm{in}} \bX_{\mathrm{in}}^T) - \|\bX_{\mathrm{out}}\|_2^2 > 0.
\end{equation}
Then, $L_{PCA} \in {B(L_*,\gamma)}$.
\end{lemma}
\begin{proof}
    The proof of this lemma is a direct consequence of the Davis-Kahan $\sin \theta$ Theorem~\citep{davis1970rotationIII}, which has a nice formulation in~\citet{vu2013fantope}. Let $L_* \in G(D,d)$ span the principal $d$-subspace of $\bX_{\mathrm{in}} \bX_{\mathrm{in}}^T$ and $L_{PCA} \in G(D,d)$ span the principal $d$-subspace of $\bX_{\mathrm{out}} \bX_{\mathrm{out}}^T + \bX_{\mathrm{in}} \bX_{\mathrm{in}}^T$. Then, applying Corollary 3.1 of~\citet{vu2013fantope} to these matrices yields
    \begin{equation}
        \left| \sin\left( \theta_1 \left(L_*, L_{PCA}  \right) \right) \right| \leq \sqrt{2} \frac{\| \bX_{\mathrm{out}}\|^2}{\lambda_d(\bX_{\mathrm{in}} \bX_{\mathrm{in}}^T)}.
    \end{equation}
    Thus, if~\eqref{eq:pcainit} holds, then we are guaranteed that 
    \begin{equation}
        \left| \sin\left( \theta_1 \left(L_*, L_{PCA}  \right) \right) \right| < | \sin (\gamma)|.
    \end{equation}
\end{proof}

The condition required in~\eqref{eq:pcainit} bears some nice similarity to the earlier condition of Theorems~\ref{thm:landscape} and~\ref{thm:conv}, $\cS(\cX, L_*, \gamma)>0$. Here the first term in~\eqref{eq:pcainit} is like the inlier permeance and the second term is like the outlier alignment.

In summary, if both of the conditions $\cS(\cX, L_*, \gamma)>0$ and~\eqref{eq:pcainit} hold, then GGD with step-size $s/\sqrt{k}$ exactly recovers $L_*$ with convergence rate $O(1/\sqrt{k})$. If we additionally have that~\eqref{eq:gradconds} holds, then GGD with the step-size in Algorithm~\ref{alg:ggd} linearly converges to $L_*$.

\section{Guarantees for Specific Statistical Models}
\label{sec:statmod}

In this section, we discuss some statistical models of data and determine when they satisfy the assumptions of our theorems. These models are meant to illustrate that our conditions are satisfied in a wide range of RSR examples, and they begin to explore the recovery limits of our algorithm. These results should provide some useful context and lead to easier interpretation of the general conditions given in Section~\ref{sec:theory} and Section~\ref{sec:grad}.

The main idea behind the study of statistical models of data is to compare the theoretical guarantees of various algorithms with a given choice of metric.
A natural metric for this purpose is the signal to noise ratio (SNR). The SNR is the ratio of inliers to outliers, $N_{\mathrm{in}} / N_{\mathrm{out}}$, and we are interested in the minimal SNR that allows exact (or sufficiently near) recovery of an RSR algorithm under a given set of assumptions on the data. In other words, under further assumptions on the data generating model, we want to derive a more interpretable condition than $\cS(\cX, L_*, \gamma) > 0$.

Due to the properties of these models, different SNR bounds may arise for different regimes of sample size. A first common small sample regime assumes that $N=O(D)$, where the SNR bounds are usually higher to account for the increased variation in the data. Another regime of slightly larger samples is obtained when $N=O(D^p)$ and $p>1$ is sufficiently small. Under some special statistical models, the SNR may decrease as $p$ increases. That is, a larger fraction of outliers can be tolerated with larger orders of sample size.
A third regime uses very large, and possibly arbitrarily large, $N$. We refer to the SNR bound of this third case as the very large $N$ regime. Under some very special models of data, the SNR bound can go to zero in this regime as the sample size increases. However, in this case, the sample size must depend on the SNR itself.

In the following, not only will we show almost state-of-the-art results for GGD on the Haystack Model of~\citet{lerman2015robust}, but we will also demonstrate how our convergence theorem holds for other more general models of data. This is an important step towards understanding how RSR algorithms perform outside of the simple Haystack Model. The general statistical models considered below assume that the inliers and outliers are sampled from probability distributions that obey certain assumptions.
These models are the following:
\begin{itemize}
    \item Assumptions on the outliers:
        \begin{itemize}
            \item Bounded support distributions
            \item Sub-Gaussian distributions
        \end{itemize}
    \item Assumptions on the inliers:
        \begin{itemize}
            \item Continuous, bounded support distributions
            \item Sub-Gaussian distributions
        \end{itemize}
\end{itemize}
In each of these cases, the inliers and outliers will be assumed to be i.i.d.~samples from the given distributions. The precise definition of each of these distributions will be given later.

In Section~\ref{sec:bdoutlier}, we bound the alignment of outliers under the above outlier models. Then, in Section~\ref{sec:subgaussin}, we bound the permeance of inliers under the above inlier models.
The goal of these first two subsections is to understand how each part of $\cS(\cX, L_*, \gamma)$ behaves on its own.
After this, in Section~\ref{sec:modstability} we prove that $\cS(\cX, L_*, \gamma) > 0$ under certain conditions on these models of inliers and outliers.
Next, we show in Section~\ref{subsec:pcainittheory} that PCA can initialize in $B(L_*, \gamma)$ in a wide range of cases.
Then, Section~\ref{sec:theorydisc} gives an in depth discussion of the Haystack Model, where we show that GGD with step-size $s/\sqrt{k}$ has almost state-of-the-art guarantees. The discussion considers the previously mentioned three regimes of sample size.
Finally, Section~\ref{subsec:linconvdisc} gives an idea of how statistical models can also ensure that the strong gradient condition in Theorem~\ref{thm:linconv} holds, which gives more evidence that the method may converge linearly in practice.

\subsection{Outlier Distributions with Restricted Alignment}
\label{sec:bdoutlier}

We explain the two assumptions on outliers listed above, which lead to bounds on the alignment. We first discuss bounded distributions in Section~\ref{subsec:bddoutlier} and then discuss sub-Gaussian distributions in Section~\ref{subsec:subgaussoutlier}.

\subsubsection{Bounded Support Distributions}
\label{subsec:bddoutlier}

We consider the case of outliers drawn from a distribution with bounded support. This assumption is needed because our bound on the alignment scales like the spectral norm of $\bX_{\mathrm{out}}$, which can be very large for even a single large outlier. An outlier distribution of this type has the form
\begin{equation}\label{eq:bddistdef}
    \cX_{\mathrm{out}} \sim \mu , \ \mu(\reals^D \setminus B(\bzero,M)) = 0,
\end{equation}
where $\mu$ represents the probability measure and $M$ is a uniform bound on the magnitude of the outliers.
In this case, we have the worst-case bound 
\begin{equation}\label{eq:bdoutbd}
	\| \bX_{\mathrm{out}} \|_2 < M \sqrt{N_{\mathrm{out}}}.
\end{equation}
In the special case where $\cX_{\mathrm{out}} \sim \mathrm{Unif}(B(\bzero,M))$, the following bound was provided in Lemma 8.4 of~\citet{lerman2015robust}:
\begin{equation}\label{eq:unifballbd}
\| \bX_{\mathrm{out}} \|_2 \leq M \left( \sqrt{ \frac{N_{\mathrm{out}}}{D-0.5}} + \sqrt{2} + \frac{t}{\sqrt{D-0.5}}\right), \ \text{w.p. at least } 1-1.5 e^{-t^2}.
\end{equation}
We remark that~\eqref{eq:bdoutbd} holds under any sampling from a bounded distribution and~\eqref{eq:unifballbd} holds under i.i.d.~sampling of a special distribution.

From these bounds and \eqref{eq:alignmentbd}, we get a sense of how the alignment scales for different types of outliers. When outliers are more adversarial but still bounded, the alignment scales like $O(N_{\mathrm{out}})$. On the other hand, when outliers have the special distribution $\mathrm{Unif}(B(\bzero,M))$, we can bound the alignment by $N_{\mathrm{out}}/\sqrt{D}$. Later, in Theorem~\ref{thm:haystack}, we show how to improve this to $O(N_{\mathrm{out}}/\sqrt{D(D-d)})$, due to the fact that the bound in~\eqref{eq:alignmentbd} is not tight.

\subsubsection{Sub-Gaussian Distributions}
\label{subsec:subgaussoutlier}

Rather than assume that the outliers are bounded, we can instead assume they come from a sub-Gaussian distribution. In this case, we have the following lemma.

\begin{lemma}\label{lemma:subgaussoutlier}
	Suppose that the outliers follow a sub-Gaussian distribution with covariance $\bSigma_{\mathrm{out}} / D_{\mathrm{out}}$, which has rank $D_{\mathrm{out}}$. Then,
     \begin{equation}\label{eq:outconcbd}
     \| \bX_{\mathrm{out}}\|_2 \leq \left\|\bSigma_{\mathrm{out}}^{1/2} \right\|_2  \left( 2 \frac{\sqrt{N_{\mathrm{out}}}}{\sqrt{D_{\mathrm{out}}}} + C \right),
	\end{equation}
	with probability at least $1-e^{-c N_{\mathrm{out}}}$. Here, $c$ and $C$ depend on the sub-Gaussian norm of $\bSigma_{\mathrm{out}}^{-1/2} \bx$, where $\bx$ follows the outlier distribution.
\end{lemma}
\begin{proof}
	Note that the transformed data, $\bSigma_{\mathrm{out}}^{-1/2} \bX_{\mathrm{out}}$, is isometric. Therefore, we bound the spectral norm of the outliers by
	\begin{equation*}
	\|\bX_{\mathrm{out}}\|_2 = \|\bSigma_{\mathrm{out}}^{1/2} \bSigma_{\mathrm{out}}^{-1/2} \bX_{\mathrm{out}}\|_2 \leq \|\bSigma_{\mathrm{out}}^{1/2}\|_2 \|\bSigma_{\mathrm{out}}^{-1/2} \bX_{\mathrm{out}}\|_2.
	\end{equation*}
	We can apply Theorem 5.39 of~\citet{vershynin2012introduction} to the last term in this inequality. This yields the bound in~\eqref{eq:outconcbd}.
\end{proof}

\subsection{Permeating Inlier Distributions}
\label{sec:subgaussin}

We will look at two assumptions that yield permeance of inliers. First, Section~\ref{subsec:genposinlier} examines continuous distribution inliers, and then Section~\ref{subsec:subgaussinlier} looks at sub-Gaussian inliers.

\subsubsection{Bounded Continuous Inlier Distributions}
\label{subsec:genposinlier}

An i.i.d.~sample from a distribution with a continuous density lies in general position with probability 1. In the case of a continuous distribution on a subspace $L_*$, this means that no $d$-subset of them lies on a $d-1$-dimensional subspace. Some slightly stronger assumptions are also needed to easily prove that the inlier permeance has a nontrivial lower bound, which includes the distribution having bounded support, although other assumptions could be used. We will refer to these distributions as bounded continuous distributions. The proof of this proposition is given in Appendix \ref{subsubsec:genpositionin}.
\begin{proposition}\label{prop:genposition}
    Suppose that the inliers are sampled from a distribution that has a continuous density with respect to the uniform measure on $L_*$. Suppose further that this distribution has mean zero and has support contained in $B(\bzero,M) \cap L_*$, for some constant $M$. Then,
    \begin{equation}\label{eq:genposin}
        \cP(\cX_{\mathrm{in}}) \gtrsim \frac{N_{\mathrm{in}}}{M} \min_{\bv \in L_* \cap S^{D-1}} \Var(\bv^T \bx) , \ \text{w.h.p.}
    \end{equation}
\end{proposition}
Here, the probability goes to $1$ and the permeance goes to $\infty$ as $ N_{\mathrm{in}} \to \infty$.

\subsubsection{Sub-Gaussian Inlier Distributions}
\label{subsec:subgaussinlier}

We show how the assumption of a sub-Gaussian distribution provides a lower bound for the permeance of inliers. The proof of this theorem is given in Appendix~\ref{subsubsec:subgaussin}.
\begin{proposition}\label{prop:subgaussinlier}
	Suppose that inliers are sampled i.i.d.~from a sub-Gaussian distribution with covariance $\bSigma_{\mathrm{in}} / d$, which has rank $d$. Then, for $0<a<1$ satisfying $(1-a)^2 N_{\mathrm{in}} >  C_1^2 d$,
	\begin{align}\label{eq:subgaussinlier}
	\cP(\cX_{\mathrm{in}}) &\geq   \frac{\lambda_d(\bSigma_{\mathrm{in}})}{ \lambda_1(\bSigma_{\mathrm{in}})^{1/2}} (1-a)^2 \frac{N_{\mathrm{in}}}{d} + O\left( \sqrt{\frac{N_{\mathrm{in}}}{d}} \right),
	\text{ w.p.~at least } 1 - 4e^{- c_1 a^2 N_{\mathrm{in}}}.
	\end{align}
	 Here, $c_1$ and $C_1$ are constants that depend on the sub-Gaussian norms of $\bSigma_{\mathrm{in}}^{-1/2} \bx$ and $\widetilde{\bSigma_{\mathrm{in}}^{-1/2} \bx}$, where $\bx$ is a random vector that follows the inlier distribution.
\end{proposition}
Notice that the choice of $a$ here affects both the bound and the probability. One could, in principal, choose $a = N^{-1/2 + \epsilon}$ for some constant $\epsilon > 0$ and still achieve an overwhelming probability bound.

\subsection{Combining Statistical Models to Enforce $\cS(\cX, L_*, \gamma) > 0$}
\label{sec:modstability}

In this section, we explicitly compare the permeance and alignment bounds for these statistical models of data to see when we can expect to have $\cS(\cX, L_*, \gamma) > 0$, which is the essential assumption in Theorems~\ref{thm:landscape} and~\ref{thm:conv}. Together with the result of the next section on PCA initialization, this implies that GGD exactly recovers $L_*$ provided that the SNR is appropriately bounded from below in these models. First,~Section~\ref{subsubsec:boundoutgenin} will look at the case of bounded outliers and bounded continuous inliers. Then, Section~\ref{subsubsec:genhaystack} will discuss the case of sub-Gaussian inliers and outliers in what we call the Generalized Haystack Model.

\subsubsection{Bounded Outliers and Bounded Continuous Inliers}
\label{subsubsec:boundoutgenin}

Under the assumption of bounded outliers and bounded continuous inliers, we can guarantee $\cS(\cX, L_*, \gamma) > 0$ for large enough SNR and sample sizes.
This results in the following proposition.
\begin{proposition}[Stability with Bounded Outliers and Continuous Bounded Inliers]
	Suppose that the outliers follow a bounded distribution and the inliers follow a mean zero, bounded distribution with continuous density on $L_*$. Then, for a fixed parameter $0 < \gamma < \pi / 2$,  $\cS(\cX, L_*, \gamma) > 0$ w.h.p. for sufficiently large SNR and $N$.
\end{proposition}
\begin{proof}
	First, the result of Proposition~\ref{prop:genposition} bounds the permeance of inliers.
	On the other hand, the outliers follow a bounded distribution. This implies that
	\begin{equation}\label{eq:genbdout}
	\max_{L \in G(D,d)} \cA(\cX, L) \leq M N_{\mathrm{out}}.
	\end{equation}
	Thus, comparing~\eqref{eq:genposin} and~\eqref{eq:genbdout}, for both $N_{\mathrm{in}} / N_{\mathrm{out}}$ and $N$ sufficiently large, $\cS(\cX, L_*, \gamma) > 0$ w.h.p.
\end{proof}

\subsubsection{The Generalized Haystack Model: sub-Gaussian Inliers and Outliers}
\label{subsubsec:genhaystack}

Next, we propose the Generalized Haystack Model as a special case of sub-Gaussian inliers and outliers. Fix a positive diagonal matrix $\bLambda_{\mathrm{in}} \in \reals^{d \times d}$ and $\bV^* \in O(D,d)$, which spans $L_* \in G(D,d)$. Letting $\bSigma_{\mathrm{in}} = \bV^* \bLambda_{\mathrm{in}} \bV^{*T}$, we assume that $N_{\mathrm{in}}$ inliers are i.i.d.~sampled from a sub-Gaussian distribution with covariance $\bSigma_{\mathrm{in}}/d$.
Fix a symmetric positive semi-definite matrix $\bSigma_{\mathrm{out}} \in \reals^{D \times D}$ and assume $N_{\mathrm{out}}$ outliers are i.i.d.~sampled from a sub-Gaussian distribution with covariance $\bSigma_{\mathrm{out}} /D_{\mathrm{out}}$, where $D_{\mathrm{out}}$ is the rank of $\bSigma_{\mathrm{out}}$. This specifies a Generalized Haystack Model with parameters $N_{\mathrm{in}}$, $\bSigma_{\mathrm{in}}$, $N_{\mathrm{out}}$, $\bSigma_{\mathrm{out}}$, $D_{\mathrm{out}}$, and $d$.  An example dataset drawn from a Generalized Haystack Model is given in Figure~\ref{fig:genhaystack}.
\begin{figure}[!t]
	\centering
	\includegraphics[width = .7\textwidth]{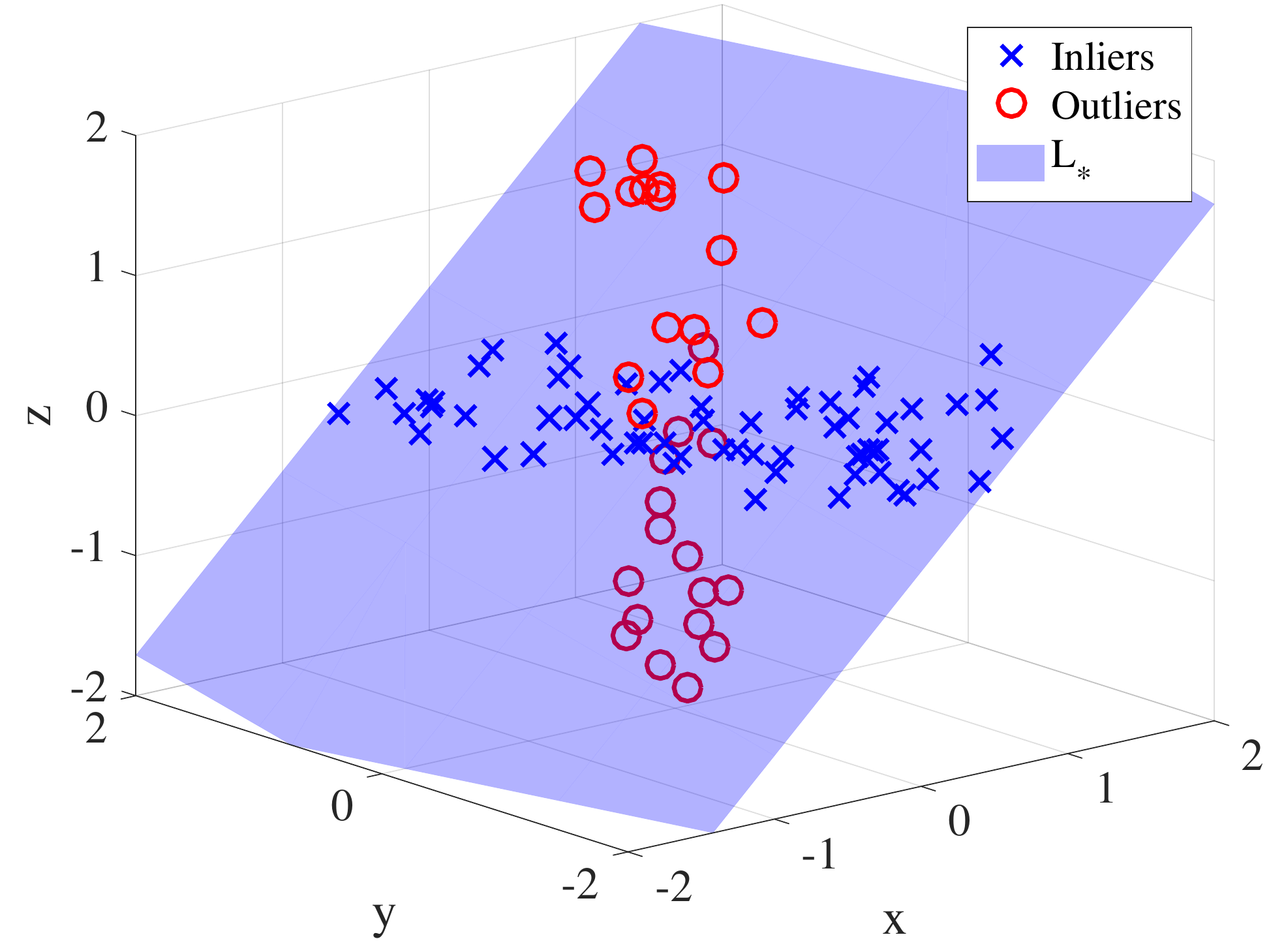}
	\caption{Example dataset drawn from a Generalized Haystack Model, where $d=2$, $D=3$, $N_{\mathrm{in}}=100$, and $N_{\mathrm{out}}=40$. Here, $L_*$ is a random 2-dimensional subspace. Inliers are sampled i.i.d.~from a normal distribution supported on $L_*$ that has variance $4$ and $0.09$ in its principal directions. The outliers are sampled i.i.d.~from a normal distribution with covariance $\bSigma_{\mathrm{out}}=\diag(.04, .04, 2.25)$. }\label{fig:genhaystack}
\end{figure}
This model generalizes the Haystack Model, which was proposed by~\citet{lerman2015robust} as a simple model with spherically symmetric Gaussian distributions of inliers and outliers.
In the latter model, inliers are distributed i.i.d.~$\cN(\bzero, \sigma_{\mathrm{in}}^2 \bP_{L_*}/d)$, and outliers are distributed i.i.d.~$\cN(\bzero, \sigma_{\mathrm{out}}^2 \bI / D)$. This defines the Haystack Model with parameters $N_{\mathrm{in}}$, $\sigma_{\mathrm{in}}$,
$N_{\mathrm{out}}$, $\sigma_{\mathrm{out}}$, $D$ and $d$.

We can combine the previous results on sub-Gaussian inliers and outliers to yield a theoretical guarantee.
We show that, under certain conditions on the Generalized Haystack Model, $\cS(
\cX, L_*, \gamma) > 0$ with overwhelming probability.

\begin{theorem}[Stability of the Generalized Haystack Model]\label{thm:genhaystack}
	Suppose that the dataset $\cX$ follows the Generalized Haystack Model with parameters $N_{\mathrm{in}}$, $\bSigma_{\mathrm{in}}$, $N_{\mathrm{out}}$, $\bSigma_{\mathrm{out}}$, $D_{\mathrm{out}}$, and $d$. Suppose also that $0 < \gamma < \pi/2$ and
	\begin{align}\label{eq:probgenhaystack}
	\SNR > \frac{1}{\cos(\gamma)} \frac{ \lambda_1(\bSigma_{\mathrm{in}})^{1/2}}{\lambda_d(\bSigma_{\mathrm{in}})} \lambda_1(\bSigma_{\mathrm{out}}^{1/2}) \frac{2}{(1-a)^2} \frac{d}{\sqrt{D_{\mathrm{out}}}} + o(1).
	\end{align}
	Then $\cS(\cX, L_*, \gamma)>0$ with probability at least $  1 - 4e^{- c_1 a^2 N_{\mathrm{in}}} - 2e^{-c_2 N_{\mathrm{out}}}$, provided that $(1-a)^2 N_{\mathrm{in}} > C_1^2 d$, where $c_1$, $c_2$ and $C_1$ are constants depending on the sub-Gaussian norms of the inliers and outliers.
\end{theorem}
\begin{proof}
	It is left to compare the bounds derived earlier in Lemma~\ref{lemma:subgaussoutlier} and Proposition~\ref{prop:subgaussinlier}. Notice that~\eqref{eq:probgenhaystack} can be obtained by requiring the right hand side of~\eqref{eq:outconcbd} to be less than $\cos(\gamma)$ times the right hand side of~\eqref{eq:subgaussinlier}. This results in precisely the statement in the theorem.
\end{proof}

Scaling the inlier and outlier covariance matrices by $d$ and $D_{\mathrm{out}}$, respectively, ensures that in the spherically symmetric case (where $\bSigma_{\mathrm{in}}$ and $\bSigma_{\mathrm{out}}$ are orthogonal projections onto subspaces of $\reals^D$ of dimensions $d$ and $D_{\mathrm{out}}$ respectively) the inliers and outliers have the same typical length. Thus, with this normalization, differences in the traces of $\bSigma_{\mathrm{in}}$ and $\bSigma_{\mathrm{out}}$ translate into differences in typical scale between inliers and outliers. We emphasize that it is important to prove results for general sub-Gaussian distributions rather than just spherically symmetric Gaussians. This is due to the fact that simpler strategies, like running PCA and then filtering points far from the PCA subspace, can be applied to the symmetric case with great success. The Generalized Haystack Model allows for certain adversarial outliers: for example, outliers can be contained in a low-dimensional subspace as well. Nevertheless, since the Haystack Model has been addressed by several previous works and since it is easy to improve our estimates for it, we address it in Section~\ref{sec:theorydisc}.

\subsection{PCA Initialization}
\label{subsec:pcainittheory}

The discussed models can also guarantee good initialization by PCA. This is an essential ingredient to actually have a practical algorithm. We demonstrate this on the specific case of the Generalized Haystack Model. However, this sort of argument can be extended to the other types of models discussed above as well (such as bounded distributions of outliers and bounded continuous inliers).

We must have a lower bound on the SNR that depends on the parameters of the sub-Gaussian distributions in order for the following proposition to hold. A short proof for this proposition is given in Appendix \ref{sec:haystackcorproof}. It essentially states that Lemma~\ref{lemma:pcainit} holds with high probability under certain conditions on the Generalized Haystack Model.
\begin{proposition}[PCA Initialization with Sub-Gaussian Models]\label{thm:pcainit}
	Suppose that the dataset $\cX$ follows the Generalized Haystack Model with parameters $N_{\mathrm{in}}$, $\bSigma_{\mathrm{in}}$, $N_{\mathrm{out}}$, $\bSigma_{\mathrm{out}}$, $D_{\mathrm{out}}$, and $d$. Suppose also that, for some $0 < \gamma < \pi/2$,
	\begin{equation}\label{eq:pcainitbd}
	\SNR \geq \frac{\sqrt{2}}{\sin(\gamma)} \frac{ d }{ D_{\mathrm{out}} } \frac{\lambda_1(\bSigma_{\mathrm{out}})}{\lambda_d ( \bSigma_{\mathrm{in}} ) } + o(1) .
	\end{equation}
	Then, for large enough $N = N_{\mathrm{out}} + N_{\mathrm{in}}$, the PCA $d$-subspace is contained in $B(L_*, \gamma)$ w.h.p.
\end{proposition}

\subsection{Performance of GGD Under the Haystack Model}
\label{sec:theorydisc}

We assume here the simpler Haystack Model and show that GGD performs almost as well as state-of-the-art methods on datasets drawn from this model. We compute results for three different regimes of sample size. These are the small sample regime $N=O(D)$, the larger sample regime $N= O(D^p)$, for $p>1$ sufficiently small, and the very large $N$ regime, where $N$ must depend on the SNR as well. In the larger sample regime, GGD requires at least $N = O(d(D-d)^2 \log (D))$, which is not more than $N = O(D^{3+\epsilon})$. In the very large $N$ regime, in addition to dependence of $N$ on a power of $D$, it also depends on a negative power of the SNR and thus is very large for small SNR. In our case, the very large $N$ regime considers sample sizes of the order $N_{\mathrm{out}}=O(\max(d^3 D^3 \log^3(N_{\mathrm{out}}),(dN_{\mathrm{out}}/N_{\mathrm{in}})^6))$.
The big O notation is slightly abused here, as we are really indicating results for finite $N$ and $D$: the order is meant to illustrate the relation between these finite values.                                      
We compare all of these results together in Table~\ref{tab:SNR}.

\subsubsection{Bounds for Sample Size $N=O(D)$}

We first translate the bounds obtained previously in Theorem~\ref{thm:genhaystack} to this special model. We choose $a=1/2$ and thus obtain the following corollary:
\begin{corollary}\label{cor:genhaystack}
	Suppose that the dataset $\cX$ follows the Haystack Model with parameters $N_{\mathrm{in}}$, $\sigma_{\mathrm{in}}$, $N_{\mathrm{out}}$, $\sigma_{\mathrm{out}}$, $D$, and $d$. Then, if $N_{\mathrm{in}} > 4 C_1^2 d$ and
	\begin{equation}\label{eq:genhaystacksnr}
        \SNR \geq 8 \frac{1}{\cos(\gamma)} \frac{\sigma_{\mathrm{out}}}{\sigma_{\mathrm{in}}} \frac{d}{\sqrt{D}} + o(1),
	\end{equation}
	$\cS(\cX, L_*, \gamma)>0$ with probability at least $  1 - 4e^{- c_1 N_{\mathrm{in}}/4} - 2e^{-c_2 N_{\mathrm{out}}}$, where $c_1$, $c_2$ and $C_1$ are absolute constants.
\end{corollary}
Notice that in this case, we obtain strong probabilistic estimates for even small sample sizes of $N=O(D)$. For the full theoretical guarantee, we also need to consider~\eqref{eq:pcainitbd}, and we must choose a value for $\gamma$. To balance between the $\sin(\gamma)$ and $\cos(\gamma)$ in~\eqref{eq:pcainitbd} and \eqref{eq:genhaystacksnr}, respectively, we fix $\gamma = \pi/4$. From these equations, for this fixed $\gamma$, we conclude that our theoretical SNR for the Haystack Model in the small sample regime is
	\begin{equation}
	\label{eq:full_small_regime}
	\SNR \geq \max \left(8 \sqrt{2} \frac{\sigma_{\mathrm{out}}}{\sigma_{\mathrm{in}}} \frac{d}{\sqrt{D}},  2 \frac{\sigma_{\mathrm{out}}^2}{\sigma_{\mathrm{in}}^2 } \frac{ d }{ D } \right).
	\end{equation}

	On the other hand, previous works~\citep{hardt2013algorithms,lerman2015robust,zhang2016robust} obtained optimal bounds for this model when $N=O(D)$ and the SNR is on the order of
	\begin{equation}
	\SNR \gtrsim  \frac{\sigma_{\mathrm{out}}}{\sigma_{\mathrm{in}}} \frac{d}{(D-d)}.
	\end{equation}
	We remark that the bound of \citet{lerman2015robust} for the REAPER algorithm requires the assumption $d < (D-1)/2$ and its constant is relatively large.
    This is in contrast to \citet{hardt2013algorithms} and~\citet{zhang2016robust}, who do not have restrictions on $d$ and do not have dependence on $\sigma_{\mathrm{out}}/\sigma_{\mathrm{in}}$.
	In this regime, we are unable to establish sharp results like the ones of REAPER, Tyler's M-estimator, or RandomizedFind. These estimates are better by a factor of $\sqrt{D}/(D-d)$ than our current estimate. Nevertheless, in this regime
	the complexity of our algorithm is $O(N D d)$, whereas the complexity of the mentioned algorithms is $O(N D^2)$ or $O(D^3)$.

    \subsubsection{Bounds for Sample Size $N=O(d(D-d)^2\log(D))$}

	\citet{zhang2014novel} obtained the following sharper bound for the GMS algorithm under the Haystack Model and the larger sample regime of $N = O(D^2)$:
	\begin{equation}
	\SNR \gtrsim  \frac{\sigma_{\mathrm{out}}}{\sigma_{\mathrm{in}}} \frac{d}{\sqrt{D(D-d)}}.
	\end{equation}
	We remark that this is the sharpest bound for any similar sample regime under this model when ${\sigma_{\mathrm{out}}} \approx {\sigma_{\mathrm{in}}}$.
    While the bounds mentioned above \citep{hardt2013algorithms,lerman2015robust,zhang2016robust} hold for any regime of sample size, they are worse by a factor of $\sqrt{D/(D-d)}$.
	We show that a similar bound in a similar regime holds for the GGD algorithm.  Indeed, the primary deficiency in Corollary~\ref{cor:genhaystack} is that we use the loose bound on the alignment,
	$\sqrt{N_{\mathrm{out}}} \| \bX_{\mathrm{out}}\|_2$. However, one could instead operate using the precursor to this bound, 
\begin{equation}\label{eq:alignbdiso}
\cA(\cX_{\mathrm{out}}, L) \leq \| \widetilde{\bQ_{L} \bX_{\mathrm{out}}}\|_2 \| \bX_{\mathrm{out}} \|_2.
\end{equation}
Using this bound instead, we have the following theorem, which shows that GGD achieves the optimal SNR bound under the Haystack Model in the region $N = O(d \, (D-d)^2 \, \log(D))$, which is at worst $O(D^3 \log D)$.
\begin{theorem}[Stability of the Haystack Model]\label{thm:haystack}
	Suppose that the dataset $\cX$ follows the Haystack Model with parameters $N_{\mathrm{in}}$, $\sigma_{\mathrm{in}}$, $N_{\mathrm{out}}$, $\sigma_{\mathrm{out}}$, $D$, and $d$. If
	\begin{align}\label{eq:haystackSNR}
	\SNR &\geq  \frac{ \sigma_{\mathrm{in}}}{\sigma_{\mathrm{out}}} \frac{1}{\cos(\gamma)} \frac{5}{(1-a)^2} \frac{d}{\sqrt{D(D-d)}} + o(1),
	\end{align}
	then $\cS(\cX, L_*, \gamma)>0$ with probability at least
	\begin{equation}\label{eq:probhaystackbd}
	1 -2e^{-N_{\mathrm{out}}/16} - e^{-N_{\mathrm{out}} / 4} - C_1 \exp\left(-\frac{N_{\mathrm{out}}}{4(D-d)} + \frac{d(D-d) }{2}\log\left( \frac{D}{D-d} \right)  \right) - 4e^{- c_1 a^2 N_{\mathrm{in}}}.
	\end{equation}
\end{theorem}

Here, we can take $a$ close to 0 (e.g., we can take $a = o(N_{\mathrm{in}}^{-1/2})$), and so we ignore the factor of $1/(1-a)^2$ in~\eqref{eq:haystackSNR}. The full theoretical guarantee, which also considers~\eqref{eq:pcainitbd}, is established similarly to \eqref{eq:full_small_regime}. We conclude that our theoretical SNR bound for the Haystack Model in this regime is 
\begin{equation}\label{eq:SNRreg2}
\SNR \geq \max \left(
5\sqrt{2} \frac{\sigma_{\mathrm{out}}}{\sigma_{\mathrm{in}}} \frac{d}{\sqrt{D(D-d)}},  2 \frac{\sigma_{\mathrm{out}}^2}{\sigma_{\mathrm{in}}^2 } \frac{ d }{ D } \right).
\end{equation}

The final results for the two different regimes, that is,~\eqref{eq:full_small_regime} and~\eqref{eq:SNRreg2}, are for $\gamma = \pi/4$ and initialization by PCA. As a side note, if the SNR grows, we see that larger values of $\gamma$ may be tolerated for GGD. In particular, for large sample sizes and sufficiently large SNRs, $\gamma$ can be sufficiently close to $\pi/2$. In this case, random initialization of GGD are expected to work as well as the PCA initialization. We quantify this claim more rigorously in the special case where $d<D/2$ and $d,D\rightarrow\infty$.
Based on the analysis of extreme singular values of random Gaussian matrices \citep{rudelson2008littlewood}, it can be shown that with high probability, a random initialization lies in ${B(L_*,\gamma)}\setminus \{L_*\}$, where $\cos(\gamma)=O(1/\sqrt{Dd})$. Therefore, GGD  with random initialization succeeds with high probability under the given assumptions on $d$ and $D$ when $N_{\mathrm{in}}/N_{\mathrm{out}} \geq O\left( d \sigma_{\mathrm{out}} / \sigma_{\mathrm{in}}  \right)$.

\subsubsection{Bounds for Very Large $N$}

    In the large sample regime, one can prove something much stronger. Indeed, for any fraction of outliers, it is obvious that PCA asymptotically recovers the underlying subspace $L_*$~\citep{lerman2017fast}. Further, it was shown that FMS can asymptotically recover $L_*$ for any fraction of outliers with better dependence on the sample size than PCA~\citep{lerman2017fast} (while the result given for FMS is for the spherized Haystack Model, the result can be extended to the non-spherized version as well). It would be very surprising if GGD could not do something similar.

In fact, this type of result can be extended to GGD as well, which means that, as $N \to \infty$, we can take SNR$\to 0$. However, while the PCA and FMS subspace estimators converge to the underlying subspace as $N \to \infty$, the probability of either exactly recovering the underlying subspace for any fixed $N$ is zero. In contrast, GGD can exactly recover the underlying subspace with overwhelming probability for finite, yet large, $N$. This result is stated in the following theorem, which is proved in Appendix~\ref{app:snrzeroproof}.
\begin{theorem}\label{thm:haystacksnrzero}
    Suppose that $\cX$ follows the Haystack Model with parameters $N_{\mathrm{in}}$, $\sigma_{\mathrm{in}}$, $N_{\mathrm{out}}$, $\sigma_{\mathrm{out}}$, $D$, and $d$. 
    For any $\SNR$ lower bound $\alpha>0$ (i.e., $\mathrm{SNR} > \alpha$) and for any $N_{\mathrm{out}}$ at least $O(\max(d^3 D^3 \log^3(N_{\mathrm{out}}),(dN_{\mathrm{out}}/N_{\mathrm{in}})^6))$, GGD recovers $L_*$ w.o.p. In particular, $N_{\mathrm{out}}$ is at least $O(d/\alpha^6)$,
\end{theorem}
For this theorem to hold, we see that $N$ needs to be quite large, especially for low SNRs. However, note that we still obtain strong probabilistic bounds for large enough finite $N$. On the other hand, taking $\alpha \to 0$ in this theorem requires $N_{\mathrm{out}} \to \infty$, which implies asymptotic recovery for GGD for any SNR in this model.

\subsubsection{Comparison of all Haystack Model Results}

\begin{table}[!ht]
	\footnotesize
	\centering
	\begin{tabular}{|c|l|}\hline
        \multirow{3}{*}{\textbf{PCA}} & {$N_{\mathrm{in}}/N_{\mathrm{out}} ``\gtrapprox"0$, when $N \to \infty$, $D$ fixed}  \\ \cline{2-2}
		&  \multirow{1}{*}{\emph{No exact recovery and poor estimates for finite $N$.}} \\ 
		\hline
		\hline

        \color{blue} \multirow{5}{*}{\textbf{GGD}}  & \color{blue} {$N_{\mathrm{in}}/N_{\mathrm{out}} \geq  \max \left( 8 \sqrt{2}  \frac{\sigma_{\mathrm{out}}}{\sigma_{\mathrm{in}}} \frac{d}{\sqrt{D}} , 2 \frac{\sigma_{\mathrm{out}}^2}{\sigma_{\mathrm{in}}^2 } \frac{ d }{D} \right) $, when $N = O(D)$} \\ 
                                                    & \color{blue} {$N_{\mathrm{in}}/N_{\mathrm{out}} \geq \max \left( 5\sqrt{2} \frac{\sigma_{\mathrm{out}}}{\sigma_{\mathrm{in}}} \frac{d}{\sqrt{D(D-d)}} ,  2 \frac{\sigma_{\mathrm{out}}^2}{\sigma_{\mathrm{in}}^2 } \frac{ d }{D} \right)$, when $N = O(d(D-d)^2 \log (D))$} \\ 
                                                    & \color{blue} {Any $N_{\mathrm{in}}/N_{\mathrm{out}} > 0 $, when $N_{\mathrm{out}}\gtrsim \max(d^3 D^3 \log^3(N),(dN_{\mathrm{out}}/N_{\mathrm{in}})^6)$} \\ 
                                                    & \color{blue} {\quad $\implies$ $N_{\mathrm{in}}/N_{\mathrm{out}} \gtrapprox 0 $, when $N \to \infty$, $D$ fixed} \\ \cline{2-2}
		& \color{blue} \multirow{1}{*}{\emph{Deterministic condition, results for a variety of data models.}} \\ 
		\hline
		\hline
		
		\multirow{4}{*}{\textbf{FMS}} & \multicolumn{1}{l|}{$ N_{\mathrm{in}}/N_{\mathrm{out}} ``\gtrapprox" 0 $, when $N \to \infty$, $D$ fixed} \\ \cline{2-2}
		\multirow{3}{*}{} & \emph{Approximate recovery for large samples from spherized Haystack or from two} \\
		&\emph{one-dimensional subspaces on the sphere. Much better estimates for finite $N$} \\ 
		&\emph{than PCA.} \\ 
		\hline
		\hline
		
		\multirow{2}{*}{\textbf{REAPER}} & \multicolumn{1}{l|}{$ N_{\mathrm{in}}/N_{\mathrm{out}} \geq 16 \frac{\sigma_{\mathrm{out}}}{\sigma_{\mathrm{in}}} \frac{d}{D}$, when $N = O(D)$, $1 \leq d \leq (D-1)/2$} \\ \cline{2-2}
		\multirow{2}{*}{} & \emph{Deterministic condition, results for Haystack where $d<(D-1)/2$.} \\ 
		\hline
		\hline
		
		\multirow{2}{*}{\textbf{GMS}} & \multicolumn{1}{l|}{$ N_{\mathrm{in}}/N_{\mathrm{out}} \geq 4 \frac{\sigma_{\mathrm{out}}}{\sigma_{\mathrm{in}}} \frac{d}{\sqrt{(D-d)D}}$, when $N = O(D^2)$} \\ \cline{2-2}
		\multirow{2}{*}{} & \emph{Deterministic condition, results for Haystack that extends to elliptical outliers.} \\
		\hline
		\hline
		
		\multirow{3}{*}{\textbf{OP}} & \multicolumn{1}{l|}{$ N_{\mathrm{in}}/N_{\mathrm{out}} \geq \frac{121 d}{9} O\left( \max(1,\log(N)/d) \right) $, when $N = O(D)$} \\ \cline{2-2}
		\multirow{2}{*}{} & \emph{Deterministic condition (formulated for arbitrary outliers) with last term in} \\ &\emph{above formula replaced by an inlier incoherence parameter $\mu$.} \\ 
		\hline
		\hline
		
		\multirow{2}{*}{\textbf{HR-PCA}} & \multicolumn{1}{l|}{$ N_{\mathrm{in}}/N_{\mathrm{out}} \to \infty$, when $N \to \infty$, $D$ fixed}  \\ \cline{2-2}
		\multirow{2}{*}{} & \emph{Weak lower bound on the expressed variance, requires fraction of outliers as input.} \\
		\hline
		\hline		
		
		\multirow{2}{*}{\textbf{TME/(D)RF}} & \multicolumn{1}{l|}{$ N_{\mathrm{in}}/N_{\mathrm{out}} > \frac{d}{D-d}$, when $N = O(D)$} \\ \cline{2-2}
		& \multicolumn{1}{l|}{\emph{Result for ``general-position" data, but does not extend to noise. }}  \\ 
		\hline
		\hline
		
		\multirow{3}{*}{\textbf{TORP}}  &  \multicolumn{1}{l|}{$N_{\mathrm{in}}/N_{\mathrm{out}} \geq 128 d \max(1, \log(N)/d)^2 $, when $N = O(D)$}   \\ \cline{2-2}
		\multirow{3}{*}{} & \emph{Deterministic condition (formulated for arbitrary outliers) with last term replaced} \\ &\emph{by an inlier incoherence parameter $\mu$, requires fraction of outliers as input.} \\ 
		\hline
		\hline
		
		\multirow{4}{*}{\textbf{CP}} & {$N_{\mathrm{in}}/N_{\mathrm{out}} \geq d/(D-d^2) $ ($N = O(D)$, $d < \sqrt{D}$)} \\ & {$N_{\mathrm{in}}/N_{\mathrm{out}} \gtrapprox 0 $, when $N \to \infty$, $d < \sqrt{D}$, $D$ fixed}  \\ \cline{2-2}
		\multirow{3}{*}{} & \emph{Exact recovery for the spherized Haystack model with a random inlier subspace} \\ &\emph {and $d<\sqrt{D}$, and also for a special model of outliers around a line.} \\
		 \hline
		 \hline
		
		\multirow{2}{*}{\textbf{SSC}} & {$N_{\mathrm{in}}/N_{\mathrm{out}} \geq d/D \cdot ((\frac{N_{\text{in}}-1}{d})^ {\frac{cD}{d}-1}-1)^{-1} $, when $N <e^{c\sqrt{D}} / D$, w.h.p. in $D$}  \\ \cline{2-2}
		\multirow{2}{*}{} & \emph{Exact recovery for the spherized Haystack model with a random inlier subspace.} \\ 
		\hline
		\hline
		
		\multirow{5}{*}{\textbf{HOSC}} & {$N_{\mathrm{in}}/N_{\mathrm{out}} \geq \log(N) N^{-2(D-d)/(2D-d)}$}\\&
		{\quad $\implies N_{\mathrm{in}}/N_{\mathrm{out}} \gtrapprox 0$, where $N \to \infty$, $D$ fixed}  \\ \cline{2-2}
        \multirow{3}{*}{} & \emph{Result for outliers uniformly sampled from $[0,1]^D$ and inliers uniformly sampled} \\ & \emph{from the intersection of a $d$-subspace with $(0,1)^d$. Also extends to a union of} \\ 
        & \emph{manifolds and to settings with small noise.} \\ 
        \hline
		
	\end{tabular}
	\vspace{3mm}
	\caption{Updated table from~\citet{lerman2018overview}, which compares the lower bounds on SNR for many RSR methods under the Haystack Model. Included here are theoretical guarantees along with the corresponding sample size requirements for the result to hold with probability close to 1.
		\label{tab:SNR}}
\end{table}

We compare the lowest SNR guarantees for a variety of RSR algorithms under the Haystack Model.
Table~\ref{tab:SNR} replicates Table 1 of~\citet{lerman2018overview} with updated estimates. It both compares lower bounds on SNR under the Haystack Model and also briefly describes the actual data model that each algorithm has 
guarantees for. The algorithms include geodesic gradient descent (GGD), FMS~\citep{lerman2017fast}, REAPER~\citep{lerman2015robust}, GMS~\citep{zhang2014novel}, OP~\citep{xu2012robust} and LLD~\citep{mccoy2011two},              
HR-PCA~\citep{xu2013outlier}, Tyler's M-estimator (TME)~\citep{zhang2016robust}, TORP~\citep{cherapanamjeri2017thresholding}, CP~\citep{rahmani2016coherence}, SSC~\citep{soltanolkotabi2012geometric}, HOSC~\citep{ariascastro2011spectral},                            
RANSAC~\citep{ariascastro2017ransac}, and RF~\citep{hardt2013algorithms}.  
For each SNR bound, we also give the associated sample size, $N$, for each result to begin holding with high probability. Again, we note that the big O notation is slightly abused here since the results are really for non-asymptotic $N$ and $D$.                                      

Here, the symbol $\gtrapprox0$ is used for the SNR lower bound for recovery when a method can exactly recover an underlying subspace for any fixed SNR and large enough $N$. Similarly, the symbol $``\gtrapprox"0$ is used when a method can \emph{approximate} $L_*$ to any accuracy for any fixed SNR and large enough $N$. However, methods with $``\gtrapprox"0$ instead of $\gtrapprox 0$ cannot exactly recover the underlying subspace.

Among all algorithms, only PCA, GGD, TORP, RANSAC, and FMS run in $O(NDd)$ time. Since the strong gradient condition can be shown to hold for the Haystack model (see Corollary~\ref{cor:linconvhaystack} in Section~\ref{subsec:linconvdisc}), GGD also achieves linear convergence. This means that it is theoretically guaranteed to be the fastest out of these $O(NDd)$ algorithms (TORP has a guarantee for linear convergence, but under very restrictive assumptions on the SNR). Furthermore, among all algorithms, GGD is the only one with guarantees close to state-of-the-art for the small sample size regime. Among all algorithms it has the state-of-the-art result for the regime $N=O(d \, (D-d)^2 \, \log(D))$, although GMS obtains such a result for the smaller regime $N= O(D^2)$. GGD is also the fastest algorithm with a result for the large sample regime as well. We note that algorithms with worse complexity, such as CP, SSC, and HOSC, have guarantees in this setting as well. We have found it too complicated to compare the exact theoretical results of these various methods, and so we have instead opted to just show that they are guaranteed to have exact recovery for any percentage of outliers for large enough $N$. 

\subsection{A Note on the Strong Gradient Condition in Theorem~\ref{thm:linconv}}
\label{subsec:linconvdisc}

While it may be hard to interpret the strong gradient condition in Theorem~\ref{thm:linconv}, it is possible to show that it holds for a variety of the data models in this section. The brief discussion here is just meant to illustrate that this condition is, in fact, practical. It is not hard to extend the arguments below to other cases discussed earlier as well. This lemma relies on the following definition.
\begin{definition}
 	Sets of inliers and outliers, $\cX_{\mathrm{in}}$ and $\cX_{\mathrm{out}}$, within an inlier-outlier dataset are said to lie in general position with respect to each other if
 	\begin{equation}
 	\max_{G(D,d) \setminus \{L_*\}} \#(\cX \cap L) \leq d.
 	\end{equation}
\end{definition}
With this definition, we prove the following lemma.
\begin{lemma}\label{lemma:linconvmod}
	Suppose that the inliers and outliers follow distributions that satisfy~\eqref{eq:bddistdef} and lie in general position with respect to each other.
	Assume further that the inliers are drawn from a distribution such that $\cP(\cX_{\mathrm{in}}) \to \infty$ as $N \to \infty$, and that
	\begin{equation}
	\max_{L \in G(D,d)} \cA(\cX_{\mathrm{out}},L)  < c \cos(\gamma) \cP(\cX_{\mathrm{in}}),
	\end{equation}
	for some $0 < c < 1$ for all $N$. Then,~\eqref{eq:gradconds} holds for $N$ sufficiently large.
\end{lemma}

The general position assumption implies that the amount of data points in any $d$-subspace that is not $L_*$ is at most $d$. 
This general position-type assumption is also satisfied with probability 1 by any sample drawn from the Haystack Model, or by inliers drawn uniformly from $L_* \cap B(\bzero, M)$ and outliers drawn uniformly from $B(\bzero, M)$.

\begin{proof}
Under the assumptions of the lemma,~\eqref{eq:gradconds} reduces to
\begin{align} \label{}
\inf_{L \in B(L_*,\gamma) \setminus \{L_*\} } \left| \frac{\di}{\di t}F(L(t);\cX)|_{t=0} \right|  &>  2 d M,
\end{align}
where $L(t)$ is a geodesic parameterized by arclength from $L_k$ through $L_*$.
Since we know that $\left| \frac{\di}{\di t}F(L(t);\cX)|_{t=0} \right| \geq \cS(\cX, L_*, \gamma)$ (from the proof of Theorem~\ref{thm:landscape}), a sufficient condition for the strong gradient condition to hold is
\begin{equation}\label{eq:stronggradred}
	\cS(\cX, L_*, \gamma)  > 8 d M.
\end{equation}
The assumptions also imply that
\begin{equation}\label{}
\cS(\cX, L_*, \gamma)  > \cos(\gamma)(1-c) \cP(\cX_{\mathrm{in}}).
\end{equation}
Therefore, the reduction of the strong gradient condition in~\eqref{eq:stronggradred} is satisfied for $N$ sufficiently large, since $\cP(\cX_{\mathrm{in}}) \to \infty$ as $N \to \infty$.
\end{proof}

One can easily modify the proof of Theorem~\ref{thm:haystacksnrzero} and the proof of Lemma~\ref{lemma:linconvmod} to obtain linear convergence for the Haystack model. We state this result without proof in the following corollary.
\begin{corollary}\label{cor:linconvhaystack}
	Suppose that the inliers and outliers are distributed according to the Haystack Model with parameters $N_{\mathrm{in}}$, $\sigma_{\mathrm{in}}$, $N_{\mathrm{out}}$, $\sigma_{\mathrm{out}}$, $D$, and $d$. Then,~\eqref{eq:gradconds} holds w.h.p.~for $N$ sufficiently large.
\end{corollary} 

\section{Simulations}
\label{sec:stabdemo}

In this section, we run two simulations to verify the theory we proved earlier. All experiments are run using the Haystack Model. Some more comprehensive comparisons of the various RSR algorithms on synthetic and stylized datasets are contained in the review paper of~\citet{lerman2018overview}.

First, we attempt to demonstrate the stability condition $\cS$ from~\eqref{eq:stab}. While we cannot explicitly evaluate the maximum within this expression, we can instead simulate the values achieved by
\begin{equation}\label{eq:simstab}
    \cos(\gamma) \cP(\cX_{\mathrm{in}}) - \cA(\cX_{\mathrm{out}},L),
\end{equation}
for $L$ in a small neighborhood of $L_*$.

\begin{figure}[!t]
\centering
\includegraphics[width = .5\textwidth]{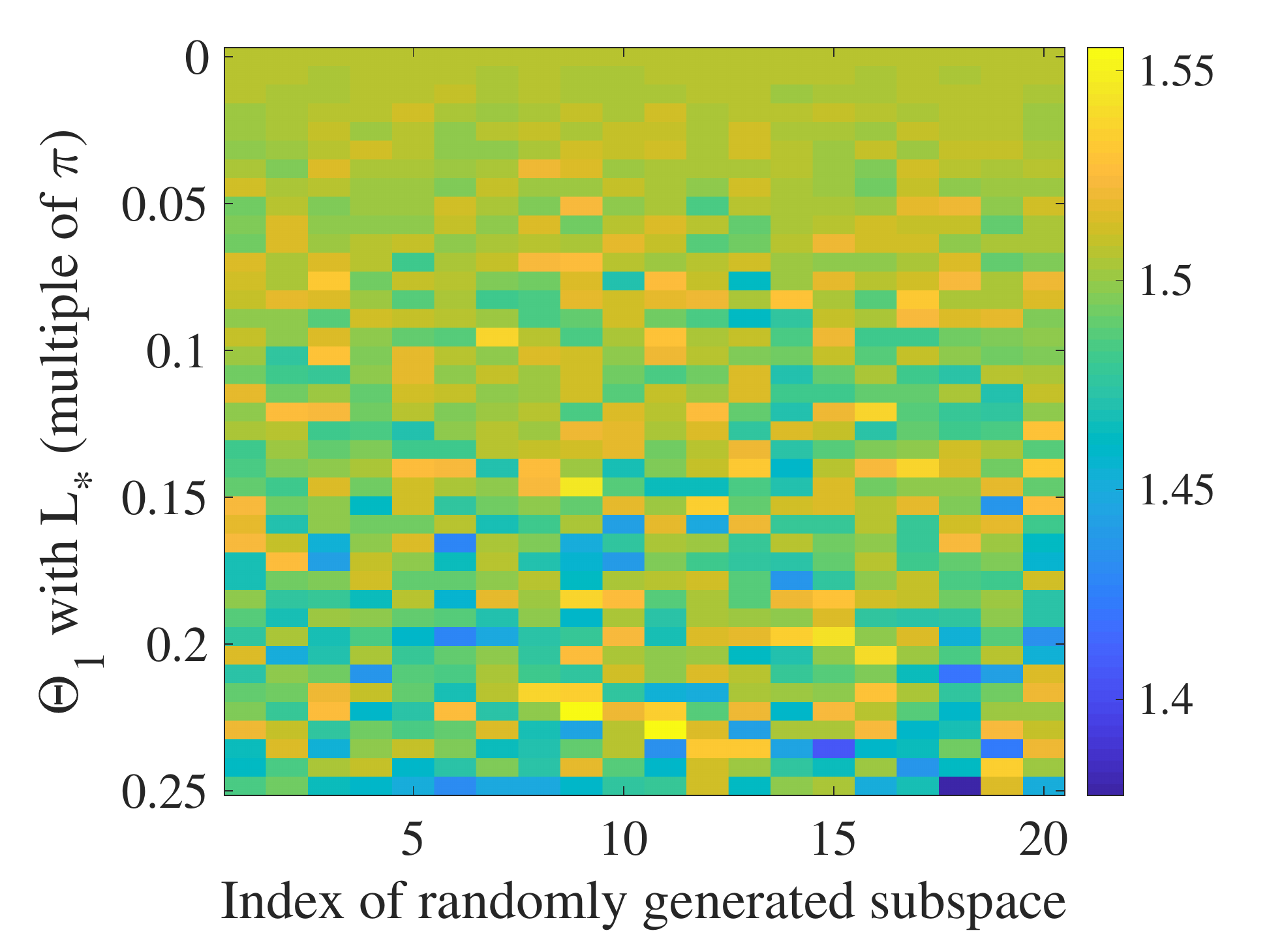}
\caption{Simulation of the stability statistic. Data was generated for the Haystack Model with parameters $N_{\mathrm{out}} = 200$, $\sigma_{\mathrm{out}} = 1$, $D=200$, and $d=10$. The $y$-axis of each figure represents the maximum principal angle of a randomly generated subspace with $L_*$. For each value of the maximum principal angle, the $x$-axis represents the index in the set of 20 randomly generated subspaces with the specified maximum principal angle.
}
\label{fig:cond_sim_1}
\end{figure}

The values achieved by~\eqref{eq:simstab} are simulated in Figure~\ref{fig:cond_sim_1} for the noiseless RSR settings with the fixed value of $\gamma = \pi/4$. The dataset for this figure is generated according to the Haystack Model outlined in~Section~\ref{sec:theorydisc}, with parameters $N_{\mathrm{in}} = 200$, $\sigma_{\mathrm{in}} = 1$,  $N_{\mathrm{out}} = 200$, $\sigma_{\mathrm{out}} = 1$, $D=200$, and $d=10$.
For this plot, the $y$-axis represents a distance from the underlying subspace
$L_*$ in terms of the maximum principal angle up to $\gamma$. The $x$-axis represents randomly generated subspaces at that distance from $L_*$. The color represents the value of~\eqref{eq:simstab}. As we can see, the value of~\eqref{eq:simstab} is indeed positive within this large neighborhood of $L_*$.

We also simulate the convergence properties of the GGD method in Figure~\ref{fig:conv_sim}.
The data was generated according to the Haystack Model with parameters $N_{\mathrm{in}} = 200$, $\sigma_{\mathrm{in}} = 1$,  $N_{\mathrm{out}} = 200$, $\sigma_{\mathrm{out}} = 1$, $D=100$, and $d=5$.
We compare different choices of step-size in accordance with the statements of Theorems~\ref{thm:conv} and~\ref{thm:linconv}. The green line denotes the
convergence of the step-size $t^k = s/\sqrt{k}$ in GGD with $s = 1/D$, which is the modified GGD from Theorem~\ref{thm:conv}. We also display three different types of piecewise constant step-size schemes with initial step-size $1/D$. The blue and red lines represent Algorithm~\ref{alg:ggd} with different choices of $K$ and shrink-factors. In one case, we use $K=20$ and a shrink-factor of $1/2$. In the other case, we use $K=50$ and a shrink-factor of $1/10$.
The pink line is a modification of this piecewise constant step-size scheme that decreases the cost monotonically. In this scheme, the algorithm only shrinks the step-size when the current constant step-size does not decrease the energy function. When this is the case, the step-size is shrunk by the largest factor $1/2^n$, for $n \in \nats$, such that the energy decreases. While this set-up performs the best, we do not yet have a theoretical justification for this scheme. We see that all of the choices of piecewise constant step-size converge linearly to $L_*$ (although the convergence is not monotonic). The small fluctuations around an error of $10^{-7}$ occur due to the machine precision limit in MATLAB.

\begin{figure}[!t]
\centering
\includegraphics[width = .5\textwidth]{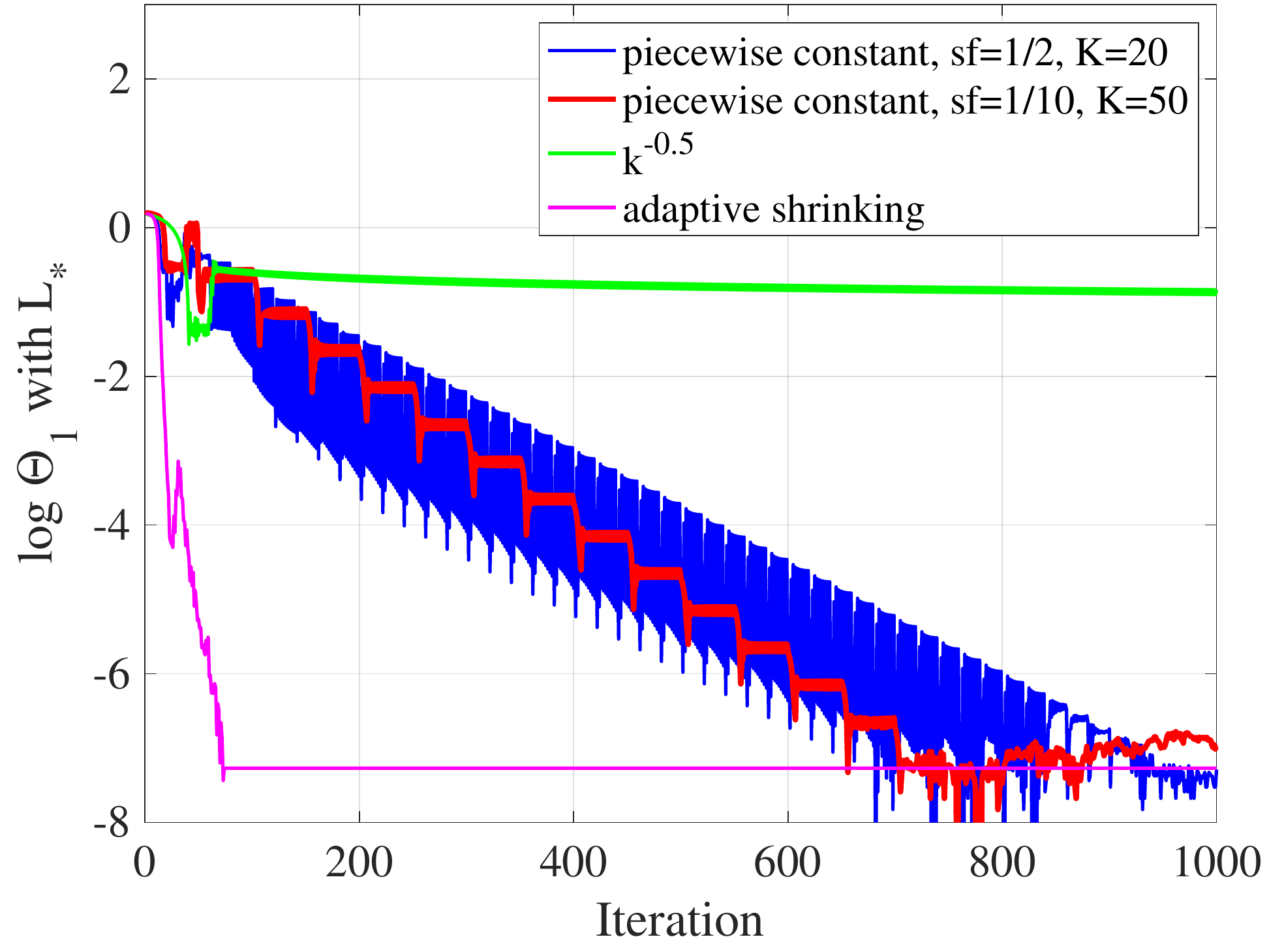}
\caption{Convergence characteristics of the GGD algorithm with different step-size choices. Data was generated according to the Haystack Model with parameters $N_{\mathrm{in}} = 200$, $\sigma_{\mathrm{in}} = 1$,  $N_{\mathrm{out}} = 200$, $\sigma_{\mathrm{out}} = 1$, $D=100$, and $d=5$..
The $y$-axis is the logarithm of the top principal angle with the underlying subspace $L_*$. For all  algorithms, the initial step-size is $s=1/D$. For the piecewise constant step-size scheme of Algorithm~\ref{alg:ggd}, we use two difference combinations of shrink-factor and time between shrinking, $K$. We also include an adaptive shrinking method, that only shrinks the constant step-size when the current step-size does not decrease the energy. \label{fig:conv_sim}}
\end{figure}

\section{Conclusion}
\label{sec:conclusion}

We have presented a deterministic condition that ensures the landscape of~\eqref{eq:robopt} behaves nicely around an underlying subspace: if $\cS(\gamma,0,L_*)>0$, the underlying subspace is the only minimizer and stationary point in $B(L_*, \gamma)$.  The deterministic condition also ensures the convergence of a non-convex gradient method for RSR. The convergence of this method is linear under some slightly stronger assumptions on the data. 
We have shown that the condition $\cS(\cX, L_*, \gamma) > 0$ and the strong gradient condition in~\eqref{eq:gradconds} hold for certain statistical models of data and have used these examples to understand the limits of recovery in various regimes. These models indicate the flexibility of the conditions to deal with a wide range of RSR datasets. GGD is even shown to have almost state-of-the-art SNR guarantees under the Haystack Model.

A main point of this work is to obtain exact recovery guarantees that are similar in spirit to the REAPER algorithm of~\cite{lerman2015robust} for a non-convex algorithm. We even manage to provide better estimates in large sample regimes. Indeed, the stability analysis of this work is inspired by the stability analysis of the REAPER algorithm, which is done with respect to a convex relaxation of the energy function considered here.
Since we do not relax the non-convex problem, our stability is tighter than the REAPER stability when considering sufficiently small neighborhoods. However, for large enough neighborhoods, the REAPER stability might be tighter than ours. For example, a difference shows up in small sample regimes of the Haystack Model: the dependence on a neighborhood makes our result weaker in terms of the SNR regimes we can tolerate. Nevertheless, as far as we know, there is no non-convex RSR competitor for the types of estimates we have developed in this paper and no other competitor with computational complexity of order $O(NDd)$. Furthermore, in larger sample regimes we obtain a stronger result than REAPER's.

While there are many directions that future work can take, we only specify a couple here. One avenue for future work is to extend this result to other data models. For example, one may consider more adversarial models of corruption. 
This is pursued in a forthcoming work~\citep{ML18}.

Another direction for future work involves study of subspace recovery in the presence of heavier noise. 
The only works that consider RSR in really noisy settings are~\citet{coudron2012sample}, \citet{minsker2015geometric}, and~\citet{cherapanamjeri2017thresholding}, although the area remains largely unexplored.
The noise considered in our work involves small perturbations from the underlying subspace. However, in general, more modern work in data science has focused on settings with heavier noise. For example, some have considered noise drawn from distributions with heavy tails~\citep{minsker2015geometric}, while others have considered PCA in the spiked model~\citep{johnstone2001distribution,baik2005phase}.  In the latter case, when the dimension is very large, the noisy inliers would most likely be very far from the underlying subspace. It is an interesting direction for future work to study RSR in both of these settings. 

Future work may also consider proof of convergence for other methods, such as Newton's method and conjugate gradient~\citep{edelman1999geometry}, or IRLS~\citep{lerman2017fast}, using the guarantees on the energy landscape of this paper. One could also consider using different frameworks for optimization over $G(D,d)$, such as a retraction based method like that specified in~\citet{absil2009optimization}. A quick heuristic argument indicates that the retraction formulation should agree with GGD up to first order, but we leave rigorous examination of this to future work.

Finally, one may also directly follow the ideas of~\cite{lim2016statistical,lim2018grassmannian} and consider extensions to the affine Grassmannian, since in practice we cannot assume that the data is properly centered.

A supplementary web page with code will be made available at \url{https://twmaunu.github.io/WTL/}.

\acks{This work was supported by NSF awards DMS-14-18386 and DMS-18-21266, a UMII MnDRIVE graduate assistantship, and a UMN Doctoral Dissertation Fellowship. The authors would like to thank Chao Gao, Nati Srebro, and anonymous reviewers for helpful comments.}

\appendix

\section{Landscape of the PCA energy}
\label{app:pca}

In this appendix, we discuss the landscape of the energy function described in \eqref{eq:pca}. Let $\bX\in\mathbb{R}^{D\times N}$ be a matrix with columns given by $\bx_1,\bx_2,\dots,\bx_N$, and assume that its singular values and left singular vectors are $\sigma_1\geq \sigma_2\geq\dots\geq\sigma_D$ and $\bv_1,\dots,\bv_D\in\mathbb{R}^D$, respectively. Under a generic scenario when $\sigma_d > \sigma_{d+1}$ and $\sigma_{D-d} > \sigma_{D-d+1}$, the landscape of the energy in \eqref{eq:pca} has the following properties:
\begin{itemize}
\item There exists a unique local minimum given by $L=\mathrm{span}(\bv_1,\dots,\bv_d)$, which is also the global minimizer; and a unique local maximum given by $L=\mathrm{span}(\bv_{D-d+1},\dots,\bv_D)$, which is also the global maximizer.
\item The set of saddle points are given by the set of $L$ such that $L=\mathrm{span}(\bv_{i_1},\dots,\bv_{i_d})$ for any $d$ distinct integers $(i_1,\dots,i_d)$ between $1$ and $D$, and $(i_1,\dots,i_d)$ can not be $(1,\dots,d)$ or $(D-d+1,\dots,D)$ (which correspond to the minimizer and the maximizer).
\end{itemize}
These properties can be derived by casting the PCA problem as a constrained optimization problem and then examining the corresponding Lagrangian. 
Since the number of saddle points is finite and all saddle points are orthogonal to the minimizer, there exists a local neighborhood around the minimizer such that the minimizer is the unique critical point inside this neighborhood, which is a property similar to Theorem~\ref{thm:landscape}. Using a strategy similar to this work, if an initialization is chosen appropriately such that the initialization lies in a neighborhood around the global minimizer, a gradient descent algorithm would converge to the global minimizer of \eqref{eq:pca}.

One other case of interest involves $\sigma_{d} = \sigma_{d+1} = \cdots = \sigma_{k}$, for $d < k \leq D$. In this case, the singular vectors corresponding to $\sigma_{d}, \dots, \sigma_k$ would span a subspace, and any orthonormal basis for $\Sp(\bv_d, \dots, \bv_{k})$ would result in a valid set of singular vectors corresponding to the singular values $\sigma_d, \dots, \sigma_{k}$. 
As an example of what happens in this case, if $\sigma_1 >\sigma_2 > \dots > \sigma_d = \dots = \sigma_k$ then the $d$-subspace spanned by $\bv_1, \dots, \bv_{d-1}$, $\bv_i \in \Sp(\bv_d, \dots, \bv_{k}) \cap S^{D-1}$ are minimizers. One can construct other sorts of continuums by also considering equality of singular values between $\sigma_1$ and $\sigma_d$. All cases constructed in these ways would yield a continuum of minimizers for the PCA $d$-subspace energy, which all achieve the same energy. Analogously, one can consider the cases where $\sigma_k = .... = \sigma_{D-d+1}$ for $1 \leq k < D-d+1$, and show a continuum of maximizers. Or, one can also consider the cases where $\sigma_k = .... = \sigma_{l}$ where $d < k < l < D-d+1$, and show a continuum of saddle points. 

In addition, while it is unrelated to the focus of this work, we remark that the landscape of PCA also has the ``strict saddle point property" discussed in \cite{ge2015escaping}: the Hessian of every saddle point has a negative eigenvalue. \citet{ge2015escaping} propose algorithms with theoretical guarantees to minimize such energy functions. This property can be seen by again examining the Lagrangian formulation of the PCA problem.

\section{Grassmannian Geodesics}
\label{sec:grassgeo}

In this appendix, we describe some basic geometric notions on the Grassmannian manifold, $G(D,d)$. Given two subspaces $L_1,L_2 \in G(D,d)$, the principal angles between the two subspaces are defined sequentially. The smallest angle, $\theta_d$, is given by
\begin{equation}\label{eq:thetad}
	\theta_d = \min_{\substack{\bv \in L_1,\by \in L_2\\ \|\bv\| = \|\by\| = 1 }} \arccos(|\bv^T \by|).
\end{equation}
The vectors $\bv_d$ and $\by_d$ which achieve the minimum are the principal vectors corresponding to $\theta_d$. The $d$ principal angles are defined sequentially by
\begin{equation}\label{eq:thetak}
	\theta_k = \min_{\substack{\bv \in L_1, \|\bv\| = 1, \bv \perp \bv_{k+1},\dots,\bv_{d} \\ \by \in L_2, \|\by\| = 1, \by \perp \by_{d+1},\dots,\by_{d} }} \arccos(|\bv^T \by|),
\end{equation}
and the corresponding principal vectors are found in the same way. The ordering defined in~\eqref{eq:thetad} and~\eqref{eq:thetak} is the reverse of what is usually used for principal angles: here, $\theta_1$ is the largest principal angle, while most other works denote the smallest principal angle with $\theta_1$.
Notice that if two principal angles are equal, the choice of principal vectors is not unique. Principal angles and vectors can be efficiently calculated: if $\bW_1 \in O(D,d)$ spans $L_1$ and $\bW_2 \in O(D,d)$ spans $L_2$, then we write the singular value decomposition
\begin{equation}
\bW_1^T \bW_2 = \bV_{12} \bSigma_{12} \bY_{12}^{T}.
\end{equation}
The principal angles are given in reverse order by $\arccos(\diag(\bSigma_{12}))$, and the corresponding principal vectors are given by the columns of $\bV_{12}$ and $\bY_{12}$ in reverse order.
Now let two subspaces $L_1$ and $L_2$ have principal angles $\theta_1,\dots,\theta_d$ and principal vectors $\bv_1,\dots,\bv_d$ and $\by_1,\dots,\by_d$, respectively. Let $k$ be the largest index such that $\theta_k > 0$, which is also known as the interaction dimension~\citep{lerman2014lp}, where $d-k$ is the dimension of $L_1 \cap L_2$. Then, we can define a complementary orthogonal basis $\bu_1,\dots,\bu_k$ for $L_2$ with respect to $L_1$ as
\begin{align}
	\bu_j \in \Sp(\bv_j,\by_j), \ \bu_j \perp \bv_j, \ \bu_j^T \by_j > 0.
\end{align}
As explained in~\citet{lerman2014lp}, for any two subspaces $L_1$ and $L_2$ such that $\theta_1 = \arccos(|\bv_1^T \by_1|) < \pi/2$, a unique geodesic on $G(D,d)$ with $L(0) = L_1$ and $L(1) = L_2$ can be parameterized by
\begin{equation}\label{eq:grassgeo}
	L(t) = \Sp(\bv_1 \cos(\theta_1 t) + \bu_1 \sin(\theta_1 t),\dots,\bv_k \cos(\theta_k t) + \bu_k \sin(\theta_k t),\bv_{k+1},\dots,\bv_d).
\end{equation}

In the paper, we will frequently use a reparametrization of geodesics to prevent the angles from needlessly affecting the magnitude of the derivatives. In these cases, we reparametrize the geodesic in~\eqref{eq:grassgeo} by arclength (in terms of metric defined Section~\ref{sec:prelim}) by writing
\begin{align}\label{eq:grassgeo}
L(t) = \Sp\Big(\bv_1 \cos(t) + \bu_1 \sin(t),\bv_2 \cos\left(\frac{\theta_2}{\theta_1} t\right) + \bu_2 \sin\left(\frac{\theta_2}{\theta_1} t\right),\dots, \\ \nonumber
\bv_k \cos\left(\frac{\theta_k}{\theta_1} t\right) + \bu_k \sin\left(\frac{\theta_k}{\theta_1} t\right),\bv_{k+1},\dots,\bv_d\Big).
\end{align}

\section{Bound on the Alignment}
\label{sec:alignbd}

Here we derive the simple bound on the alignment statistic seen in~\eqref{eq:alignmentbd}. Here, the $\widetilde{\cdot}$ notation denotes projection of the data points (columns of a data matrix) to the unit sphere, $S^{D-1}$. This bound is used to prove Theorem~\ref{thm:haystack}. The first step follows from~\eqref{eq:derF},~\eqref{eq:grad}, and~\eqref{eq:align}.
\begin{align}\label{eq:alignmentbound}
    \cA(\cX_{\mathrm{out}},\bV) &= \left \|   \sum_{\substack{\bx_i \in \cX_{\mathrm{out}}}} \frac{\bQ_{\bV} \bx_i}{\|\bQ_{\bV} \bx_i\|} \bx_i^T \bV \right\|_2 = \left\| \widetilde{\bQ_{\bV} \bX_{\mathrm{out}}} \bX_{\mathrm{out}}^T \bV \right\|_2 \\ \nonumber
                                &\leq \left\|  \widetilde{\bQ_{\bV} \bX_{\mathrm{out}}} \right\|_2 \left\| \bX_{\mathrm{out}}^T \bV \right\|_2 \leq \left\|  \widetilde{\bQ_{\bV} \bX_{\mathrm{out}}} \right\|_F \left\| \bX_{\mathrm{out}} \right\|_2 \\ \nonumber
                                &\leq \sqrt{N_{\mathrm{out}}} \| \bX_{\mathrm{out}} \|_2.
  \end{align}

\section{Supplementary Proofs}
  In this appendix, we give the supplementary proofs for the various theorems, propositions, and lemmas given in the paper.

\subsection{Proof of Theorem~\ref{thm:landscapenoise}}
\label{sec:landscapenoiseproof}

The proof of this theorem proceeds much in the same way as the proof of Theorem~\ref{thm:landscape}.

We first define the following geodesic on $G(D,d)$:
Fix a subspace $L \in {B(L_*,\gamma)} \setminus {B(L_*,\eta)}$, and let the principal angles between $L$ and $L_*$ be $\theta_1,\dots,\theta_d$. Here, we define $\eta = 2 \arctan(\epsilon/\delta)$.
Also, choose a set of corresponding principal vectors $\bv_1,\dots,\bv_d$ and $\bw_1,\dots,\bw_d$ for $L$ and $L_*$, respectively, and let $l$ be the maximum index such that $\theta_1 = \dots = \theta_l$. We let $\bu_1, \dots,\bu_l$ be complementary orthogonal vectors for $\bv_1,\dots,\bv_l$ and $\bw_1,\dots,\bw_l$, which exist since $\theta_1(L,L_*) > 0$. For $t \in [0,1]$, we form the geodesic
\begin{equation*}
	L(t) = \Sp(\bv_1\cos(t) + \bu_1 \sin(t),\dots,\bv_l\cos(t) + \bu_l \sin(t),\bv_{l+1},\dots,\bv_d).
\end{equation*}
This geodesic moves only the furthest directions of $L(0)$ towards $L_*$, and we have removed dependence on $\theta_1$, since this unnecessarily impacts the magnitude of the geodesic derivative~\eqref{eq:geoderiv}.

We will first prove the inequality $\partial F(L(t);\cX) < - \cS_n(\cX,L_*, \epsilon, \delta,\gamma)$. Using the derivative formula in~\eqref{eq:geoderiv}, a subderivative of $F(L(t);\cX)$ at $t=0$ is given by
\begin{align}\label{eq:noisegeoderiv}
    &\frac{\di}{\di t} F(L(t);\cX) \Big|_{t=0} = - \sum_{\substack{\bx_i \in \cX\\ \|\bQ_{L} \bx_i \| > 0}} \sum_{j=1}^l \frac{\bv_j^T \bx_i \bx_i^T \bu_j}{\|\bQ_L \bx_i\|} \\ \nonumber
    &= - \sum_{j=1}^l \left(\sum_{\substack{\bx_i \in \cF_1(\cX_{\mathrm{in}},\bw_j,\delta) \\ \|\bQ_{L} \bx_i \| > 0}}  \frac{\bv_j^T \bx_i \bx_i^T \bu_j}{\|\bQ_L \bx_i\|} + \sum_{\substack{\bx_i \in \cF_0(\cX_{\mathrm{in}},\bw_j,\delta) \\ \|\bQ_{L} \bx_i \| > 0}}  \frac{\bv_j^T \bx_i \bx_i^T \bu_j}{\|\bQ_L \bx_i\|} + \sum_{\substack{\bx_i \in \cX_{\mathrm{out}}\\ \|\bQ_{L} \bx_i \| > 0}} \frac{\bv_j^T \bx_i \bx_i^T \bu_j}{\|\bQ_L \bx_i\|} \right) .
\end{align}

We examine the terms in~\eqref{eq:noisegeoderiv} one by one. Using \eqref{eq:derF} and~\eqref{eq:grad}, we can bound the outlier term
\begin{equation}\label{eq:noiseoutbd}
-\sum_{\substack{\bx_i \in \cX_{\mathrm{out}}\\ \|\bQ_{L} \bx_i \| > 0}} \frac{\bv_j^T \bx_i \bx_i^T \bu_j}{\|\bQ_L \bx_i\|}  \leq \left\| \nabla F(L;\cX_{\mathrm{out}})  \right\|_2   .
\end{equation}
We know that $|\bx_i \cdot \bu_j | \leq \| \bQ_L \bx_i\|$ for all $i$ since $\bu_j \in \Sp(\bQ_L)$. We also know that, since $\theta_j > \eta$, $|\bv_j^T \bx_i | \leq | \bw_j^T \bx_i | \leq \sqrt{\delta^2 + \epsilon^2}$ for all $\bx_i \in \cF_0(\cX_{\mathrm{in}},\bw_j,\delta) $. These two observations imply that
\begin{equation}\label{eq:noiseinoutbd}
    - \sum_{\substack{\bx_i \in \cF_0(\cX_{\mathrm{in}},\bw_j,\delta) \\ \|\bQ_{L} \bx_i \| > 0}}  \frac{\bv_j^T \bx_i \bx_i^T \bu_j}{\|\bQ_L \bx_i\|} \leq \sqrt{\delta^2 + \epsilon^2} \max_{\bw \in L_* \cap S^{D-1}} \# (\cF_0(\cX_{\mathrm{in}}, \bw, \delta)).
\end{equation}
Thus, we finally must deal with the inlier term. We have the inequalities
\begin{align}
    \bv_j \bx_i \bx_i^T \bu_j &\geq \cos(\theta_1 - \eta) \sin(\theta_1 - \eta), \\
    \|\bQ_{L} \bx_i \| &\leq \sin(\theta_1) \| \bP_{L_*} \bx_i \| + \epsilon.
\end{align}
Applying these to the inlier term, as long as $\theta_1 > \eta$, we find
\begin{align} \label{eq:noiseininbd}
    - \sum_{\substack{\bx_i \in \cF_1(\cX_{\mathrm{in}},\bw_j,\delta) \\ \|\bQ_{L} \bx_i \| > 0}} & \frac{\bv_j^T \bx_i \bx_i^T \bu_j}{\|\bQ_L \bx_i\|} \\ \nonumber
                                                                                       &\leq - \frac{\cos(\theta_1 - \eta) \sin(\theta_1 - \eta)}{\sin(\theta_1)} \sum_{\substack{\bx_i \in \cF_1(\cX_{\mathrm{in}},\bw_j,\delta) \\ \|\bQ_{L} \bx_i \| > 0}}  \frac{\bw_j^T \bx_i \bx_i^T \bw_j}{\|\bP_{L_*} \bx_i\| + \epsilon / \sin(\theta_1)} \\ \nonumber
                                                                                       &\leq - \frac{\cos(\gamma - \eta)}{2} \lambda_d \left( \sum_i \frac{\bP_{L_*} \bx_i \bx_i^T \bP_{L_*}}{\|\bP_{L_*} \bx_i\| + \sqrt{\delta^2 + \epsilon^2}} \right).
\end{align}
Putting together~\eqref{eq:noiseoutbd}, \eqref{eq:noiseinoutbd}, and \eqref{eq:noiseininbd}, we find that
\begin{equation}\label{eq:grassgeodernoisebd}
    \partial F(L(t);\cX) \Big|_{t=0} \leq -l \cS_n(\cX,L_*, \epsilon, \delta,\gamma)< 0.
\end{equation}

Thus,~\eqref{eq:grassgeodernoisebd} implies that every subspace in ${B(L_*,\gamma)} \setminus {B(L_*,\eta)}$ has a direction with negative local subderivative. From here, the proof is the same as that of Theorem~\ref{thm:landscape}.

\subsection{Proof of Theorem~\ref{thm:conv}}
\label{app:suppconvres}

We will first show that
  \begin{equation}\label{eq:itdecr}
         \theta_1(\bV^{k+1},L_*) < \theta_1(\bV^k,L_*),
    \end{equation}
  for sufficiently small $t^k$.
  Let $\bV^* \in O(D,d)$ span $L_*$. We will establish~\eqref{eq:itdecr} by showing
  \begin{equation*}\label{eq:sigmaincr}
    \sigma_{d} \left( \bV^{*T} \bV^{k+1} \right) > \sigma_{d} \left( \bV^{*T} \bV^{k} \right).
  \end{equation*}
  Using~\eqref{eq:gradgeo} and the fact that
  \begin{align*}
    &\cos(\bSigma^k t^k) = \bI - O((t^k)^2),  &\sin(\bSigma^k t^k) = \bSigma^k t^k - O((t^k)^3),
  \end{align*}
  we can write
  \begin{align}\label{eq:perturb1}
    \sigma_{d} \left( \bV^{*T} \bV^{k+1} \right) &= \sigma_{d} \left( \bV^{*T} \left( \bV^k \bW^k \cos(\bSigma^k t^k) \bW^{kT} + \bU^k \sin(\bSigma^k t^k) \bW^{kT}  \right) \right) \\ \nonumber
    &= \sigma_{d} \Bigg( \bV^{*T} \Bigg( \bV^k - t^k \nabla F(\bV^k;\cX) + O((t^k)^2)\Bigg) \Bigg).
  \end{align}
  Let $\bv_1^k \in \Sp(\bV^k)$ be a unit vector corresponding to the maximum principal angle with $L_*$, and let $\bu_1^k$ be its complementary orthogonal vector.
  Define the unit vector $\by_1^k \in L_* \cap \Sp(\bv_1^k,\bu_1^k)$, and write $\theta_1^k = \theta_1(\Sp(\bV^k),L_*)$.
  Suppose that $\sigma_d^k = \sigma_d(\bV^{*T} \bV^k)$ has multiplicity $r$, and let $\bbeta_1, \ \bbeta_2 \in O(d,r) $ be such that
  \begin{equation}
    \bbeta_1^T \bV^{*T} \bV^k \bbeta_2 = \diag(\sigma_d^k,\dots,\sigma_d^k).
  \end{equation}

  We now apply Result 4.1 in~\cite{soderstrom1999perturbation}, which states the following. Suppose a matrix $\bA$ has a singular value $\sigma$ with multiplicity $r$, with corresponding left and right singular vectors $\bU$ and $\bV$. Suppose that we perturb $\bA$ by $\epsilon \bB$. Then, $\bA + \epsilon \bB$ has $r$ singular values $\sigma_1(\bA + \epsilon \bB),\dots,\sigma_r(\bA + \epsilon \bB)$ which satisfy
  \begin{equation*}
  	\sigma_j(\bA + \epsilon \bB) = \sigma_j(\bA) + \frac{\epsilon}{2} \lambda_j\left(\bV^T \bB^T \bU + \bU^T \bB \bV  \right) + O(\epsilon^2).
  \end{equation*}
  Applying this to~\eqref{eq:perturb1} yields
  \begin{align}\label{eq:perturb2}
    &\sigma_{d} \left( \bV^{*T} \bV^{k+1} \right) \geq \sigma_{d} \left( \bV^{*T} \bV^k \right) + \frac{t^k}{2} \lambda_d \left( \bbeta_1^T \bV^{*T} \sum_{\cX} \frac{\bQ_{\bV^k} \bx_i \bx_i^T \bV^k}{\| \bQ_{\bV^k} \bx_i \|} \bbeta_2 \right) + O((t^k)^2)\\ \nonumber
    & =\sigma_{d} \left( \bV^{*T} \bV^k \right) + t^k \left( \by_1^{kT} \sum_{\cX} \frac{\bQ_{\bV^k} \bx_i \bx_i^T}{\| \bQ_{\bV^k} \bx_i \|} \bv_1^k \right) + O((t^k)^2)\\ \nonumber
    &= \sigma_{d} \left( \bV^{*T} \bV^k \right) + t^k \sin(\theta_1^k) \left( \bu_1^{kT} \sum_{\cX_{\mathrm{in}}} \frac{\bx_i \bx_i^T}{\| \bQ_{\bV^k} \bx_i \|} \bv_1^k + \bu_1^{kT} \sum_{\cX_{\mathrm{out}}} \frac{\bx_i \bx_i^T}{\| \bQ_{\bV^k} \bx_i \|} \bv_1^k\right)   + O((t^k)^2).
  \end{align}
  Here, the $O((t^k)^2)$ term is bounded below by $-C_2 (t^{k})^2$, where $C_2$ does not depend on $\bV^k$, which follows from compactness of $O(D,d)$.

  Notice that the inlier term in~\eqref{eq:perturb2} is positive and bounded below
  \begin{align}\label{eq:thmconvin}
    &\bu_1^{kT} \sum_{\cX_{\mathrm{in}}} \frac{\bx_i \bx_i^T}{\| \bQ_{\bV^k} \bx_i \|} \bv_1^k \geq \frac{1}{\sin(\theta_1(\Sp(\bV^k),L_*))}  \sum_{\cX_{\mathrm{in}}} \frac{\bu_1^{kT} \bx_i \bx_i^T \bv_1^k}{\|\bx_i\|} \\ \nonumber
    &\hspace{30pt} \geq \cos(\theta_1(\Sp(\bV^k),L_*))) \sum_{\cX_{\mathrm{in}}} \frac{\by_1^{kT} \bx_i \bx_i^T \by_1^k }{\|\bx_i\|} \geq \cos(\gamma)  \lambda_d \left( \sum_{\cX_{\mathrm{in}}} \frac{\bx_i \bx_i^T}{\|\bx_i\|} \right).
  \end{align}
  Using the fact that $\bu_1^k \in \Sp(\bQ_{\bV^k})$ and $\bv_1^k \in \Sp(\bV^k)$, we can bound the outlier term in~\eqref{eq:perturb2}
  \begin{equation}\label{eq:thmconvout}
    \left|\bu_1^{kT} \sum_{\cX_{\mathrm{out}}} \frac{\bx_i \bx_i^T}{\| \bQ_{\bV^k} \bx_i \|} \bv_1^k \right| \leq \sigma_1 \left( \sum_{\cX_{\mathrm{out}}} \frac{ \bQ_{\bV^k} \bx_i \bx_i^T}{\| \bQ_{\bV^k} \bx_i \|} \bV^k \right) = \sigma_1 \left( \nabla F(\bV^k;\cX_{\mathrm{out}}) \right).
  \end{equation}
  Thus, from~\eqref{eq:thmconvin} and~\eqref{eq:thmconvout} we conclude
  \begin{align}\label{eq:sigmabd}
    &\sigma_{d} \left( \bV^{*T} \bV^{k+1} \right)  - \sigma_{d} \left( \bV^{*T} \bV^k \right)  \\ \nonumber
    &\hspace{10pt} \geq t^k\sin(\theta_1^k) \left(\cos(\gamma)  \lambda_d \left( \sum_{\cX_{\mathrm{in}}} \frac{\bx_i \bx_i^T}{\|\bx_i\|} \right) - \sigma_1 \left( \nabla F(\bV^k;\cX_{\mathrm{out}}) \right)\right) - C_2 (t^k)^2 \\ \nonumber
    &\hspace{10pt} \geq t^k\sin(\theta_1^k) \left(\cos(\gamma)  \lambda_d \left( \sum_{\cX_{\mathrm{in}}} \frac{\bx_i \bx_i^T}{\|\bx_i\|} \right) - \sup_{\bV \in {B(L_*,\gamma)}} \sigma_1 \left( \nabla F(\bV;\cX_{\mathrm{out}}) \right)\right) -  C_2 (t^k)^2 \\ \nonumber
    &\hspace{10pt} \geq t^k \sin(\theta_1^k) C_1 - C_2 (t^k)^2 = t^k \left(\sin(\theta_1^k) C_1 - C_2 t^k\right),
  \end{align}
  for positive constants $C_1$ and $C_2$ which do not depend on $\bV^k$.
	Hence, for small enough $t^k$, we have that~\eqref{eq:itdecr} holds.
	
	It remains to show that the sequence with step-size $s/\sqrt{k}$ converges to $L_*$ for sufficiently small $s$.
  Suppose that $s$ satisfies
    \begin{equation}\label{eq:sbd1}
    s < \min \left( \frac{C_1 \sin(\gamma)}{2 C_2},\frac{1}{4 \sqrt{C_2} }   \right).
  \end{equation}
  Then, for any $\bV^k$ with $\frac{\gamma}{2\sqrt{k}} \leq \theta_1(\bV^k,L_*) \leq \gamma$, looking at~\eqref{eq:sigmabd} and the first term in~\eqref{eq:sbd1}, $s/\sqrt{k}$ decreases the principal angle by at least $C_3/k$, for some constant $C_3$.
  On the other hand, for any $\bV^k$ such that $\theta_1(\bV^k,L_*) < \frac{\gamma}{2\sqrt{k}} $, we have the bound
  \begin{align}\label{eq:sigmaminusbd}
    \sigma_{d} \left( \bV^{*T} \bV^{k+1} \right) - \sigma_{d} \left( \bV^{*T} \bV^{k} \right) &>- C_2 (t^k)^2.
  \end{align}
  Note that~\eqref{eq:sigmaminusbd} gives the inequality
  \begin{align}\label{eq:sigmaminusbd2}
    \sigma_{d} \left( \bV^{*T} \bV^{k+1} \right) &>  \sigma_{d} \left( \bV^{*T} \bV^{k} \right) - C_2 (t^k)^2 \geq \cos\left(\frac{\gamma}{2\sqrt{k}}\right) - C_2 (t^k)^2.
  \end{align}
  It is straightforward to show that if
  \begin{equation*}\label{eq:sbd2}
    t^k < \frac{1}{4 \sqrt{C_2} \sqrt{k}} ,
  \end{equation*}
  then the right hand side of~\eqref{eq:sigmaminusbd2} is greater than $\cos(\gamma/\sqrt{k})$.
  Thus, the second term in the minimum of~\eqref{eq:sbd1} implies that if $\theta_1(\bV^k,L_*) < \frac{\gamma}{2\sqrt{k}} $, then $\theta_1(\bV^{k+1},L_*) < \gamma / \sqrt{k}$.

 We summarize this in the following way. For any $k$, either $\frac{\gamma}{2\sqrt{k}} \leq \theta_1(\bV^k,L_*) \leq \gamma$ or  $ \theta_1(\bV^k,L_*) < \frac{\gamma}{2\sqrt{k}}$. If the former holds, then $\theta_1(\bV^{k+1},L_*) < \theta_1(\bV^{k},L_*) - C_3/k$. If the latter holds, then we have the bound $\theta_1(\bV^{k+1},L_*) < \gamma/\sqrt{k}$. Thus, the maximum principal angle with $L_*$ either decreases by $C_3/k$ or the distance is bounded by $\gamma/\sqrt{k}$.
  Put together, these imply that $\bV^k$ converges to $L_*$ with $O(1/\sqrt{k})$ rate of convergence.

\subsection{Proof of Lemmas in Theorem~\ref{thm:linconv}}
\label{sec:lemmaproof}

This appendix contains proofs for the three lemmas in Theorem~\ref{thm:linconv}.

\subsubsection{Proof of Lemma~\ref{lemma:decrbd}}

For any fixed subspace $L \in B(L_*,\gamma)$, let the principle vectors and angles for $L$ with respect to $L_*$ be given by $\{\bv_j\}_{j=1}^d$ and $\{\theta_j\}_{j=1}^d$. Then, writing $\bTheta=\diag(\theta_1,\cdots,\theta_d)$, we know that $F(L_*;\cX_{\mathrm{in}})=0$ and
\begin{align}\label{eq:inenergybd}
    F(L; \cX_{\mathrm{in}}) &=\sum_{\bx\in\cX_{\mathrm{in}}}\|\sin(\theta)\bV^T \bx\|<\sum_{\bx\in\cX_{\mathrm{in}}}\|\bTheta\bV^T \bx\| \leq  \| \bTheta \|_2 \max_{\bV \in O(D,d)} \sum_{\cX_{\mathrm{in}}}\|\bV^T \bx_i\| \\ \nonumber
                            &\leq \theta_1 \sum_{\cX_{\mathrm{in}}} \| \bx_i \| .
\end{align}
Let $L(t)$ be the geodesic from $L_*$ to $L$ parameterized by arclength, and let $s = \theta_1$. The difference in energies for the outliers is bounded by
\begin{align}\label{eq:outdevbd}
   \left| F(L;\cX_{\mathrm{out}})-F(L_*;\cX_{\mathrm{out}}) \right| &= \left| \int_{t=0}^s \frac{\di}{\di t} F(L(t); \cX_{\mathrm{out}})  \di t \right| \\ \nonumber   
   &\leq \theta_1 \sup_{L' \in {B(L_*,\sin(\gamma)}} (\sigma_1(\nabla F(L';\cX_{\mathrm{out}}))).
\end{align}
By the assumption $\cS(\cX, L_*, \gamma) > 0$, we have by the definition in~\eqref{eq:stab} that
\begin{align}\label{eq:alignbdlinconv}
    \sup_{L' \in {B(L_*,\sin(\gamma)} }(\sigma_1(\nabla F(L' ;\cX_{\mathrm{out}}))) &< \cos(\gamma) \lambda_d \left(\sum_{\cX_{\mathrm{in}}} \frac{\bx_i \bx_i^T}{\|\bx_i\|} \right) \leq \max_{\bv \in S^{D-1}} \left(\sum_{\cX_{\mathrm{in}}} \bv^T \frac{\bx_i \bx_i^T}{\|\bx_i\|} \bv \right) \\ \nonumber
                                                                                                       &< \max_{\bv \in S^{D-1}} \sum_{\cX_{\mathrm{in}}} | \bv^T \bx| \leq \sum_{\cX_{\mathrm{in}}} \| \bx_i \|_2 .
   \end{align}
Thus, combining~\eqref{eq:inenergybd}-\eqref{eq:alignbdlinconv} yields
\begin{align}
   \left| F(L;\cX)-F(L_*;\cX) \right| &= \left| F(L;\cX_{\mathrm{in}}) + F(L;\cX_{\mathrm{out}}) - (F(L_*;\cX_{\mathrm{in}}) +F(L;\cX_{\mathrm{out}})) \right| \\ \nonumber
   &\leq \left| F(L;\cX_{\mathrm{in}})\right| + \left| F(L;\cX_{\mathrm{out}}) - F(L;\cX_{\mathrm{out}}) \right| \\ \nonumber
   &\leq 2 \theta_1 \sum_{\cX_{\mathrm{in}}} \|\bx_i \|.
\end{align}
As a result, the lemma is proved.

\subsubsection{Proof of Lemma~\ref{lemma:gradnormbd}}

 For any $L \in B(L_*, \gamma) \setminus \{L_*\}$ and geodesic $L(t)$ parameterized by arclength from $L$ through $L_*$,
    \begin{align}
     	\left| \frac{\di}{\di t}F(L(t);\cX)|_{t=0} \right| &= \left|  \frac{\di}{\di t}F(L(t);\cX_{\mathrm{in}})|_{t=0} +  \frac{\di}{\di t}F(L(t);\cX_{\mathrm{out}})|_{t=0} \right|  \\ \nonumber
    	&\geq \left|  \frac{\di}{\di t}F(L(t);\cX_{\mathrm{in}})|_{t=0} \right| -  \left| \frac{\di}{\di t}F(L(t);\cX_{\mathrm{out}})|_{t=0} \right| \\ \nonumber
    	&\geq \cos(\gamma) \lambda_1 \left(\sum_{\cX_{\mathrm{in}}} \frac{\bx_i \bx_i^T}{\|\bx_i\|} \right) - \sup_{L \in  {B(L_*, \gamma)}} \sigma_1 (\nabla F(L; \cX_{\mathrm{out}})) \\ \nonumber
    	& \geq \cS(\cX, L_*, \gamma) \geq C_1 \sum_{\cX_{\mathrm{in}}} \| \bx_i \|  ,
    \end{align}
where, $C_1 = \cS(\cX, L_*, \gamma) / \sum_{\cX_{\mathrm{in}}} \| \bx_i \| $.

\subsubsection{Proof of Lemma~\ref{lemma:applip}}

First, assume that $c_0$ is small enough such that
\begin{align} \label{eq:c0small}
        \inf_{L(0) \in B(L_*,\gamma) \setminus\{L_*\}} \frac{1}{4}  \left| \frac{\di}{\di t}F(L(t);\cX)|_{t=0} \right| &>  \sum_{ \cX} 3 \sqrt{c_0} \| \bx_i \|,
    \end{align}
    where $L(t)$ is a geodesic parameterized by arclength between $L(0)$ and $L_*$.
    We are guaranteed that such a $c_0$ exists by Lemma~\ref{lemma:decrbd}.

Fix $c< c_0$ and define a geodesic line $L(t)$ on $G(D,d)$ such that $L(0)=L_k$ and $L(c\theta_1(L_k,L_*))=L_{k+1}$. We first investigate the derivatives of the function $\dist(\bx,L(t))$. We will then show, under the given assumptions on the data, that $\sum_i \dist(\bx_i, L(t))$ is close to being Lipschitz.

Applying \citet[(23)-(24)]{lerman2014lp}, and assuming that $\sum_{j=1}^d\theta_j^2=1$, we have
\begin{align}\label{eq:secondder}
\frac{\di^2}{\di t^2}&\dist(\bx,L(t))  \\ \nonumber
&=\frac{\di}{\di t}\left[-\frac{\sum_{j=1}^d\theta_j((\cos(t\theta_j)\bv_j+\sin(t\theta_j)\bu_j)\cdot \bx)((-\sin(t\theta_j)\bv_j+\cos(t\theta_j)\bu_j)\cdot \bx)}{\dist(\bx,L(t))}\right]\\ \nonumber
=&-\frac{\left(\sum_{j=1}^d\theta_j((\cos(t\theta_j)\bv_j+\sin(t\theta_j)\bu_j)\cdot \bx)((-\sin(t\theta_j)\bv_j+\cos(t\theta_j)\bu_j)\cdot \bx)\right)^2}{\dist(\bx,L(t))^3}\\ \nonumber
&-\frac{\sum_{j=1}^d\theta_j^2((-\sin(t\theta_j)\bv_j+\cos(t\theta_j)\bu_j)\cdot \bx)((-\sin(t\theta_j)\bv_j+\cos(t\theta_j)\bu_j)\cdot \bx)}{\dist(\bx,L(t))}\\ \nonumber
&-\frac{\sum_{j=1}^d\theta_j^2((\cos(t\theta_j)\bv_j+\sin(t\theta_j)\bu_j)\cdot \bx)((-\cos(t\theta_j)\bv_j-\sin(t\theta_j)\bu_j)\cdot \bx)}{\dist(\bx,L(t))}.
\end{align}
Using $|(\cos(t\theta_j)\bv_j+\sin(t\theta_j)\bu_j)\cdot \bx|\leq \|\bx\|$ and the Cauchy-Schwarz inequality in this equation gives
\begin{align}\label{eq:term1bd}
&\left(\sum_{j=1}^d\theta_j((\cos(t\theta_j)\bv_j+\sin(t\theta_j)\bu_j)\cdot \bx)((-\sin(t\theta_j)\bv_j+\cos(t\theta_j)\bu_j)\cdot \bx)\right)^2 \\ \nonumber
& \quad \leq \left(\sum_{j=1}^d\theta_j^2((\cos(t\theta_j)\bv_j+\sin(t\theta_j)\bu_j)\cdot \bx)^2\right) \left(\sum_{j=1}^d((-\sin(t\theta_j)\bv_j+\cos(t\theta_j)\bu_j)\cdot \bx)^2\right) \\ \nonumber
& \quad \leq \sum_{j=1}^d\theta_j^2 \|\bx\|^2 \dist(\bx,L(t))^2=\|\bx\|^2 \dist(\bx,L(t))^2.
\end{align}
Putting~\eqref{eq:secondder} and~\eqref{eq:term1bd} together yields
\begin{align}
\left|\frac{\di^2}{\di t^2}\dist(\bx,L(t))\right|\leq \frac{\|\bx\|^2 \dist(\bx,L(t))^2}{\dist(\bx,L(t))^3}+\frac{2\|\bx\|^2\sum_{j=1}^d\theta_j^2}{\dist(\bx,L(t))}=\frac{3\|\bx\|^2}{\dist(\bx,L(t))}.
\end{align}

On the other hand, \citet[Lemma 3.2]{lerman2014lp} implies that
\begin{align}
\left|\frac{\di}{\di t}\dist(\bx,L(t))\right|\leq \|\bx\|.
\end{align}
Then, define the set
\begin{align}\label{eq:Gset}
    \cG(\cX, L_*, L(t), &c) = \left\{ \bx \in \cX : \min_{t \in [0,c\theta_1(L,L_*)]} \frac{\dist(\bx,L(t))}{\|\bx\|} \leq \sqrt{c} \theta_1(L_k, L_*) \right\}.
\end{align}
Then, for all $0\leq s\leq c\theta_1(L_k,L_*)$,
\begin{align}\label{eq:diff_derivative}
&\left|\sum_{\bx\in\cX}\left(\frac{\di}{\di t}\dist(x,L(t))\Big|_{t=s}-\frac{\di}{\di t}\dist(x,L(t))\Big|_{t=0}\right)\right| \\ \nonumber
&\ \ \ \ \ \ \ \ \ \ \ \ \ \ \ \ \ \ \ \ \leq \sum_{\bx\in\cX\setminus \cG(\cX, L_*, L(t), c) }\int_{t=0}^s\Big|\frac{\di^2}{\di t^2}\dist(\bx,L(t))\Big|\di t +\sum_{\bx\in \cG(\cX, L_*, L(t), c) } 2\|\bx\| \\ \nonumber
&\ \ \ \ \ \ \ \ \ \ \ \ \ \ \ \ \ \ \ \ \leq  \sum_{\bx\in\cX\setminus \cG(\cX, L_*, L(t), c) } s \frac{3\|\bx\|}{\sqrt{c} \theta_1(L_k,L_*)} + \sum_{\bx\in \cG(\cX, L_*, L(t), c) } 2\|\bx\| \\ \nonumber
&\ \ \ \ \ \ \ \ \ \ \ \ \ \ \ \ \ \ \ \ \leq   \sum_{\bx\in\cX\setminus \cG(\cX, L_*, L(t), c)) } 3\sqrt{c}\|\bx\| + \sum_{\bx\in \cG(\cX, L_*, L(t), c) } 2\|\bx\|.
\end{align}
In the last line of~\eqref{eq:diff_derivative}, the first term can be made as small as one would like by taking $c_0$ small enough. On the other hand, for $c$ small enough, all points in $\cG(\cX, L_*, L(t), c)$ must be contained in a subspace $L \in B(L_*, \gamma) \setminus\{L_*\}$.
Thus, for small enough $c$, the function $F(L(t);
\cX)$ is approximately Lipschitz for $t \in [0,c\theta_1(L_k,L_*)]$.

As a result,
\begin{align*}
    F(L_{k};\cX)-F(L_{k+1};\cX)) &=F(L(0);\cX)-F(L(c\theta_1(L_k,L_*));\cX)\\ \nonumber
                                &=-\int_{t=0}^{c\theta_1(L_k,L_*)}\frac{\di}{\di t}F(L(t);\cX) \di t \\
                                &\geq -c\theta_1(L_k,L_*) \frac{\di}{\di t}F(L(t);\cX)|_{t=0}\\ \nonumber
                                & - c\theta_1(L_k,L_*) \left(\sum_{\bx\in\cX\setminus \cG(L_k, L_{k+1}) } 3\sqrt{c}\|\bx\| + \sum_{\bx\in \cG(L_k, L_{k+1}) } 2\|\bx\| \right).
\end{align*}
Finally, noting that $\frac{\di}{\di t}F(L(t))|_{t=0}<0$, we use~\eqref{eq:c0small} and the assumption in~\eqref{eq:gradconds} to find
\begin{align*}
    F(L_{k})-F(L_{k+1})) &\geq \left| c\theta_1(L_k,L_*) \frac{\di}{\di t}F(L(t);\cX)|_{t=0} \right|  - \frac{c\theta_1(L_k,L_*)}{2} \left| \frac{\di}{\di t}F(L(t);\cX)|_{t=0} \right| \\ \nonumber
    & \geq  \frac{c\theta_1(L_k,L_*)}{2} \left| \frac{\di}{\di t}F(L(t);\cX)|_{t=0} \right| .
\end{align*}

\subsection{Proof of Inlier Permeance Bounds}
\label{sec:inlierpermbds}

\subsubsection{Proof of Proposition~\ref{prop:genposition}}
\label{subsubsec:genpositionin}

The continuous distribution assumption implies that all directions in $S^{D-1} \cap L_*$ have nonzero probability. Let $\bx$ be a random variable following the inlier distribution. By the central limit theorem, we have that
    \begin{equation*}
        \frac{1}{N_{\mathrm{in}}} \sum_{\cX_{\mathrm{in}}} \left(\frac{|\bv^T\bx_i|^2}{\|\bx_i\|} - \Var\left( \frac{\bv^T \bx}{ \|\bx\|^{1/2}}\right) \right) \overset{d}{\to} N(0, 1) .
    \end{equation*}
    We also have that $\max_i \| \bx_i \| < M$ by the bounded-support assumption assumption.
    Thus, for large enough $N_{\mathrm{in}}$ and a covering argument, we have for some absolute constant $C$ that
\begin{equation*}
\min_{\bv \in L_* \cap S^{D-1}} \sum_{\cX_{\mathrm{in}}} \frac{|\bv^T\bx_i|^2}{\| \bx_i \|} \geq C \frac{N_{\mathrm{in}}}{M} \min_{\bv \in L_* \cap S^{D-1}} \Var\left( \frac{\bv^T \bx}{ \|\bx\|^{1/2}}\right) , \ \text{w.h.p.}
\end{equation*}
This can be done by using, e.g., Hoeffding's inequality on the bounded random variable $\bv^T \bx/\|\bx\|^{1/2}$.

\subsubsection{Proof of Proposition~\ref{prop:subgaussinlier}}
\label{subsubsec:subgaussin}

	For the inliers, we need to bound $\cP(\cX_{\mathrm{in}})$.
	Using Theorem 3.1 in~\cite{lu2000some} and letting $\widetilde{\cdot}$ denote the spherization operator,
	\begin{equation}\label{eq:2termperm}
	\lambda_d \left( \sum_{\cX_{\mathrm{in}}} \frac{\bx_i \bx_i^T}{\| \bx_i \|} \right) = \sigma_d ( \widetilde{\bX_{\mathrm{in}}}\bX_{\mathrm{in}}^T) \geq \sigma_d ( \widetilde{\bX_{\mathrm{in}}}) \sigma_d( \bX_{\mathrm{in}}).
	\end{equation}
    We proceed by bounding the last term in~\eqref{eq:2termperm}. Notice that $\bSigma_{\mathrm{in}}^{-1/2} \bX_{\mathrm{in}}$ has directional variance $1/d$ in all directions of $L_*$ (i.e., these transformed inliers are isotropic with variance $1/d$). We again apply Theorem 3.1 in~\cite{lu2000some} and then apply Theorem 5.39 in~\cite{vershynin2012introduction} to $\sigma_d \left( \bSigma_{\mathrm{in}}^{-1/2} \bX_{\mathrm{in}} \right)$, where we scale by the standard deviation $1/\sqrt{d}$ and choose $t = a\sqrt{N_{\mathrm{in}}}$:
	\begin{align}\label{eq:inconcbd1}
	\sigma_d \left( \bX_{\mathrm{in}} \right) &= \sigma_d \left( \bSigma_{\mathrm{in}}^{1/2} \bSigma_{\mathrm{in}}^{-1/2} \bX_{\mathrm{in}} \right) \geq \sigma_d \left( \bSigma_{\mathrm{in}}^{1/2}\right) \sigma_d\left( \bSigma_{\mathrm{in}}^{-1/2} \bX_{\mathrm{in}} \right) \\ \nonumber
	  &>  \lambda_d\left(\bSigma_{\mathrm{in}}^{1/2}\right) \left( (1-a) \sqrt{\frac{N_{\mathrm{in}}}{d}}- C_1 \right), \ \text{w.p. at least } 1-2e^{- c_1 a^2 N_{\mathrm{in}}}.
	\end{align}
	Here, $c_1$ and $C_1$ are constants that depend on the sub-Gaussian norm of $\bSigma_{\mathrm{in}}^{-1/2} \bx$, where $\bx$ follows the inlier distribution, and $a$ must be chosen such that $(1-a)^2 N_{\mathrm{in}} >  C_1^2 d$. On the other hand, $\widetilde{\bX}_{\mathrm{in}}$ is still sub-Gaussian and $\widetilde{\bSigma_{\mathrm{in}}^{-1/2} \bX_{\mathrm{in}}}$ is isotropic sub-Gaussian. Therefore, by using~\eqref{eq:inconcbd1} and applying Theorem 3.1 in~\cite{lu2000some}, we have
	\begin{align}\label{eq:inconcbd2}
	\sigma_d \left( \widetilde{\bX_{\mathrm{in}}} \right) &= \sigma_d \left( \left[\frac{\bx_1}{\|\bx_1\|},\dots, \frac{\bx_{{N_{\mathrm{in}}}}}{\|\bx_{N_{\mathrm{in}}}\|}\right] \right) \\ \nonumber
	&= \sigma_d \left(\bSigma_{\mathrm{in}}^{1/2} \left[ \frac{\bSigma_{\mathrm{in}}^{-1/2}\bx_1}{\|\bSigma_{\mathrm{in}}^{1/2}\bSigma_{\mathrm{in}}^{-1/2} \bx_1\|},\dots, \frac{\bSigma_{\mathrm{in}}^{-1/2} \bx_{{N_{\mathrm{in}}}}}{\|\bSigma_{\mathrm{in}}^{1/2}\bSigma_{\mathrm{in}}^{-1/2} \bx_{N_{\mathrm{in}}}\|}\right] \right) \\ \nonumber
	&\geq \sigma_d \left( \frac{\bSigma_{\mathrm{in}}^{1/2}}{\|\bSigma_{\mathrm{in}}^{1/2}\|_2} \left[ \frac{\bSigma_{\mathrm{in}}^{-1/2}\bx_1}{\|\bSigma_{\mathrm{in}}^{-1/2} \bx_1\|},\dots, \frac{\bSigma_{\mathrm{in}}^{-1/2} \bx_{{N_{\mathrm{in}}}}}{\|\bSigma_{\mathrm{in}}^{-1/2} \bx_{N_{\mathrm{in}}}\|}\right] \right) \\ \nonumber
	&\geq \frac{\lambda_d (\bSigma_{\mathrm{in}}^{1/2})}{\lambda_1 (\bSigma_{\mathrm{in}}^{1/2})} \sigma_d \left( \widetilde{\bSigma_{\mathrm{in}}^{-1/2} \bX_{\mathrm{in}}} \right)  \\ \nonumber
	&>  \frac{\lambda_d(\bSigma_{\mathrm{in}}^{1/2})}{\lambda_1 (\bSigma_{\mathrm{in}}^{1/2})} \left( (1-a) \sqrt{\frac{N_{\mathrm{in}}}{d}}- C_1' \right), \ \text{w.p. at least } 1-2e^{- c_1' a^2 N_{\mathrm{in}}}.
	\end{align}
    Here, $c_1'$ and $C_1'$ are constants that depend on the sub-Gaussian norm of $\widetilde{\bSigma_{\mathrm{in}}^{-1/2} \bx}$, where $\bx$ is a random vector that follows the inlier distribution, and $a$ must chosen such that $(1-a)^2 N_{\mathrm{in}} >  C_1^{'2} d$. Therefore, for $a$ to satisfy both requirements, we need $(1-a)^2 N_{\mathrm{in}} >  \max(C_1,C_1)^{'2} d$.
	Abusing notation, we let $c_1$ be the minimum of $c_1$ and $c_1'$, and we let $C_1$ be the maximum of $C_1$ and $C_1'$.
    Putting~\eqref{eq:inconcbd1} and~\eqref{eq:inconcbd2} together, we find
	\begin{align}\label{eq:inconcbd}
	\lambda_d \left( \sum_{\cX_{\mathrm{in}}} \frac{\bx_i \bx_i^T}{\|\bx_i \|} \right) &>   \frac{\lambda_d(\bSigma_{\mathrm{in}})}{ \lambda_1(\bSigma_{\mathrm{in}})^{1/2}}  \left( (1-a) \sqrt{\frac{N_{\mathrm{in}}}{d}}- C_1 \right)^2, \text{ w.p.~at least } 1 - 4e^{- c_1 a^2 N_{\mathrm{in}}}.
	\end{align}

\subsection{Proof of Proposition~\ref{thm:pcainit}}
\label{sec:haystackcorproof}

To prove this proposition, we only need to show that PCA initializes in ${B(L_*,\gamma)}$. Note that the covariance matrix for the data $\cX$ is $\bSigma = N_{\mathrm{out}} \bSigma_{\mathrm{out}} / (N D_{\mathrm{out}}) +
N_{\mathrm{in}} \bSigma_{\mathrm{in}} / (Nd)$.  We want to see how close the sample covariance approximates $\bSigma_{\mathrm{in}}$, since the PCA subspace is spanned by the top $d$ eigenvectors of the sample covariance.
 Denoting the sample covariance by $\bSigma_N$, let $\bV_{PCA}$ be its top $d$ eigenvectors. Also, let $\bV^*$ be the top $d$ eigenvectors of $\bSigma_{\mathrm{in}}$ and let $\bV$ be the top $d$ eigenvectors of $\bSigma$. We note that
\begin{equation}\label{eq:trisin}
	|\sin(\theta_1(\bV_{PCA},\bV^*))| \leq |\sin(\theta_1(\bV_{PCA},\bV))| + |\sin(\theta_1(\bV,\bV^*))|.
\end{equation}

To deal with the last term in~\eqref{eq:trisin}, the Davis-Kahan $\sin \theta$ Theorem~\citep{davis1970rotationIII}, or more precisely Corollary 3.1 of \cite{vu2013fantope}, gives
\begin{equation*}
	|\sin(\theta_1(\bV,\bV^*)| \leq \sqrt{2} \frac{ \lambda_1( N_{\mathrm{out}} \bSigma_{\mathrm{out}} / (N D_{\mathrm{out}})) }{\lambda_d ( N_{\mathrm{in}} \bSigma_{\mathrm{in}} / (N d) )} =  \sqrt{2}  \frac{N_{\mathrm{out}}}{N_{\mathrm{in}}} \frac{d}{D_{\mathrm{out}}} \frac{\lambda_1(\bSigma_{\mathrm{out}})}{\lambda_d (\bSigma_{\mathrm{in}})}.
\end{equation*}
Thus,
\begin{equation}\label{eq:sincond1}
	\sin \left(\gamma \right) > \sqrt{2}  \frac{N_{\mathrm{out}}}{N_{\mathrm{in}}} \frac{d}{D_{\mathrm{out}}} \frac{\lambda_1(\bSigma_{\mathrm{out}})}{\lambda_d (\bSigma_{\mathrm{in}})}
\end{equation}
is a sufficient condition for $|\sin(\theta_1(\bV,\bV^*)| < \sin(\gamma ) $.

On the other hand, similar to the derivation of~\eqref{eq:inconcbd1}, for the first term in the right hand side of~\eqref{eq:trisin}, we must bound how close the sample covariance is to the true covariance. Proposition 2.1 of~\citet{vershynin2012close} implies that, for every $\delta > 0$,
\begin{equation}\label{eq:samptruecov}
  \|\bSigma - \bSigma_N\|_2 \leq \epsilon(\delta,\bSigma_{\mathrm{in}},\bSigma_{\mathrm{out}})  \left( \frac{D_{\mathrm{out}}}{N} \right)^{\frac{1}{2}},
\end{equation}
with probability at least $1-\delta$, where $\epsilon(\delta,\bSigma_{\mathrm{in}},\bSigma_{\mathrm{out}})$ is a constant depending on $\delta$, $\bSigma_{\mathrm{in}}$, and $\bSigma_{\mathrm{out}}$.
By assumption, $\bSigma$ has a positive $d$th eigengap.
Thus, another application of the Davis-Kahan $\sin \theta$ Theorem yields
\begin{align*}
	|\sin(\theta_1(\bV_{PCA},\bV))| &\leq \frac{\|\bSigma - \bSigma_N\|_2 }{\lambda_d ( \bSigma) - \lambda_{d+1} (\bSigma) } \leq \epsilon_2(\delta,\bSigma_{\mathrm{in}},\bSigma_{\mathrm{out}})  \left( \frac{d D}{N} \right)^{\frac{1}{2}},
\end{align*}
with probability $1-\delta$.
For large enough $N$, we have
\begin{align}\label{eq:pcatruesub}
	|\sin(\theta_1(\bV_{PCA},\bV)| \leq \sin\left( \gamma \right) - |\sin(\theta_1(\bV,\bV^*)| .
\end{align}
Rearranging~\eqref{eq:pcatruesub} and using the triangle inequality in~\eqref{eq:trisin} yields
\begin{equation*}
	|\sin(\theta_1(\bV_{PCA},\bV^*))| \leq \sin(\gamma),
\end{equation*}
with probability $1-\delta$. 

\subsection{Proof of Theorem~\ref{thm:haystack}}
\label{app:propaligncov}

We begin by bounding the outlier term. The last term on the right hand side of~\eqref{eq:alignbdiso} is optimally bounded in the proof of Theorem~\ref{thm:genhaystack}. On the other hand, in the following proposition, we obtain tighter bounds for the first term on the right hand side of this inequality.
\begin{proposition} \label{prop:aligncov}
		Suppose that $\cX$ is drawn from the Haystack Model with parameters $N_{\mathrm{in}}$, $\sigma_{\mathrm{in}}$, $N_{\mathrm{out}}$, $\sigma_{\mathrm{out}}$, $D$, and $d$. Then, for $D-d \geq 3$
		\begin{align}
		&\max_{L \in G(D,d)}  \left\| \widetilde{\bQ_L \bX_{\mathrm{out}}} \right\|_2 \leq \frac{7}{2} \sqrt{\frac{N_{\mathrm{out}}}{D-d}} + \sqrt{2}, \\ \nonumber
		& \ \ \ \ \ \ \text{w.p. at least } 1 -\exp\left(-\frac{N_{\mathrm{out}}}{4} \right)- C_1 \exp\left(-\frac{N_{\mathrm{out}}}{4(D-d)} + \frac{d(D-d) \log(D)}{2}  \right),
		\end{align}
		where $C_1$ is an absolute constant.
\end{proposition}

\begin{proof}[Proof of Proposition~\ref{prop:aligncov}]
First, for a single $L$, Lemma 8.4 in \citet{lerman2015robust} yields
\begin{equation}\label{eq:oneLoutsphere}
\left\| \widetilde{\bQ_L \bX_{\mathrm{out}}} \right\|_2 \leq  \left( \sqrt{ \frac{N_{\mathrm{out}}}{D-d-0.5}} + \sqrt{2} + \frac{t}{\sqrt{D-d-0.5}}\right), \ \text{w.p. at least } 1-1.5 e^{-t^2}.
\end{equation}
We set $t=\sqrt{N_{\mathrm{out}}}/2$ to obtain the desired bound for a single subspace $L$. However, in order to cover all of $G(D,d)$, we much use a more complicated argument.

Assume that we have two subspaces $L_0$ and $L_1$ such that $\theta_1(L_0, L_1) < 1/(2\sqrt{D})$. First, due to the triangle inequality,
\begin{equation}\label{eq:sphouttri}
\left\| \widetilde{\bQ_{L_0} \bX_{\mathrm{out}}} -  \widetilde{\bQ_{L_1} \bX_{\mathrm{out}}} \right\|_2 < \left\| \widetilde{\bQ_{L_0} \bX_{\mathrm{out}}^0} -  \widetilde{\bQ_{L_1} \bX_{\mathrm{out}}^0} \right\|_2 + \left\| \widetilde{\bQ_{L_0} \bX_{\mathrm{out}}^1} -  \widetilde{\bQ_{L_1} \bX_{\mathrm{out}}^1} \right\|_2.
\end{equation}
Here, we have partitioned the outliers into two parts. The following set-valued functions define these parts. Let $L(t)$ be the geodesic between $L_0$ and $L_1$ such that $L(0) = L_0$ and $L(1) = L_1$. Then, we define
\begin{align}
\cX_{\mathrm{out}}^0 =  \left\{ \bx \in \cX_{\mathrm{out}} : \min_{t \in [0,1]} \angle(\bx, L(t)) < \frac{1}{2\sqrt{D}} \right\}, \\
\cX_{\mathrm{out}}^1 =  \left\{ \bx \in \cX_{\mathrm{out}} : \min_{t \in [0,1]} \angle(\bx, L(t)) \geq \frac{1}{2\sqrt{D}} \right\}.
\end{align}
Notice that
\begin{equation}
\cX_{\mathrm{out}}^0 \subset \left\{ \bx \in \cX_{\mathrm{out}} : \angle(\bx, L_0) < \frac{1}{\sqrt{D}} \right\}.
\end{equation}
With these datasets, their data matrices are $\bX_{\mathrm{out}}^0$ and $\bX_{\mathrm{out}}^1$.

For the last term in~\eqref{eq:sphouttri}, since $\theta_1(L_0, L_1) < 1/(2\sqrt{D})$, we have
\begin{equation}\label{eq:faroutbd}
\left\| \widetilde{\bQ_{L_0} \bX_{\mathrm{out}}^1} -  \widetilde{\bQ_{L_1} \bX_{\mathrm{out}}^1} \right\|_2 < \frac{1}{2}\sqrt{\frac{N_{\mathrm{out}} }{D}}.
\end{equation}

On the other hand, we must look at the concentration of points around subspaces for the second to last term in~\eqref{eq:sphouttri}. For any given $L$, we have the following concentration lemma.
\begin{lemma}
	If $\bx \sim \cN(\bzero,\bI)$, then
	$$\Pr \left( \angle(\bx,L) < \frac{1}{2} \sqrt{\frac{D-d}{D}} \right) < \exp\left( -\frac{D-d}{2} \right).$$
\end{lemma}
\begin{proof}
	This is a direct consequence of Lemma 2.2 in \citet{dasgupta2003elementary}.
\end{proof}

For $D-d \geq 3$, we have
\begin{equation}\label{eq:Dbd}
\Pr \left( \angle(\bx,L) <  \frac{1}{2} \sqrt{\frac{D-d}{D}} \right) \leq \exp \left(-\frac{D-d}{2}\right) \leq \frac{1}{D-d}.
\end{equation}
Using~\eqref{eq:Dbd} together with a loose Chernoff bound for the concentration of binomial random variables~\citep{mitzenmacher2005probability}, we know that
\begin{equation}\label{eq:binbd}
\Pr \left( \#\left(\left\{ \bx \in \cX_{\mathrm{out}} : \angle(\bx,L) <  \frac{1}{2} \sqrt{\frac{D-d}{D}} \right\}\right) > \frac{3}{2} \frac{N_{\mathrm{out}}}{D-d} \right) < \exp \left( -\frac{N_{\mathrm{out}}}{4(D-d)} \right).
\end{equation}
Thus, we use~\eqref{eq:binbd} on the second to last term in~\eqref{eq:sphouttri} to find
\begin{align}\label{eq:closeoutbd}
\left\| \widetilde{\bQ_{L_0} \bX_{\mathrm{out}}^0} -  \widetilde{\bQ_{L_1} \bX_{\mathrm{out}}^0} \right\|_2 &< \sqrt{ \#(\cX_{\mathrm{out}}^0)} \leq \sqrt{\frac{3}{2}\frac{N_{\mathrm{out}}}{D-d}}, \\ \nonumber &\text{w.p. at least } 1-\exp\left(-\frac{N_{\mathrm{out}}}{4(D-d)}\right).
\end{align}
Notice that also, by construction, this is true for all $L_0', L_1' \in B(L(t),  \frac{1}{4} \sqrt{(D-d)/D})$.

We finish by completing the covering argument on $G(D,d)$. By remark 8.4 of~\citep{szarek1983finite}, $G(D,d)$ can be covered by $(C)^{d(D-d)} / (\gamma_1)^{d(D-d)}$ balls of radius $\gamma_1$. This means that, using a union bound and taking~\eqref{eq:oneLoutsphere},~\eqref{eq:sphouttri},~\eqref{eq:faroutbd}, and~\eqref{eq:closeoutbd} together,
\begin{align}
\Pr&\left( \max_{L \in G(D,d)} \| \widetilde{\bQ_{L} \bX_{\mathrm{out}}} \| < \frac{7}{2} \sqrt{ \frac{N_{\mathrm{out}}}{D-d-0.5}} + \sqrt{2} \right) > \\ \nonumber & 1 - C_1 \exp\left(-\frac{N_{\mathrm{out}}}{4(D-d)} + C_2 d(D-d) \log\left( \frac{D}{D-d} \right)   \right) - \exp\left( -\frac{N_{\mathrm{out}}}{4} \right) ,
\end{align}
where $C_1$ and $C_2$ are absolute constants.
Thus, we see that $N$ must be on the order $N=O(d(D-d)^2\log(D))$, and we have have the desired statement.
\end{proof}

We now continue with the proof of Theorem~\ref{thm:haystack}. In Proposition~\ref{prop:aligncov}, We see that $N$ needs to be on the order of $d(D-d)^2\log(D)$, which is not the case of the small sample regime. On the other hand, Theorem 5.39 of~\citet{vershynin2012introduction} states that
\begin{equation}\label{eq:outconcbdhaystack}
	\| \bX_{\mathrm{out}} \|_2 \leq \sigma_{\mathrm{out}} \left[ \frac{5}{4} \sqrt{\frac{N_{\mathrm{out}}}{D}} + 1 \right], \ \text{w.p. at least } 1-2e^{-N_{\mathrm{out}}/16}.
	\end{equation}
	
	We combine the result of Proposition~\ref{prop:aligncov} with~\eqref{eq:outconcbdhaystack} to find that
	\begin{align}\label{eq:haystackfinaloutbd}
	&\| \widetilde{\bQ_{L} \bX_{\mathrm{out}}}\|_2 \| \bX_{\mathrm{out}} \|_2 \leq \frac{7}{2} \frac{N_{\mathrm{out}}}{\sqrt{D(D-d)}} + O\left(\sigma_{\mathrm{out}} \sqrt{\frac{N_{\mathrm{out}}}{D}} \right), \\ \nonumber
	& \text{w.p. at least } 1 -2e^{-N_{\mathrm{out}}/16} - e^{-N_{\mathrm{out}} / 4} - C_1 \exp\left(-\frac{N_{\mathrm{out}}}{4(D-d)} + C_2 d(D-d) \log\left( \frac{D}{D-d} \right)   \right).
	\end{align}
	
	We can also improve the results for the inlier permeance bound since they are already isometric. In this case, we have that~\eqref{eq:inconcbd1} still holds. On the other hand, $\widetilde{\bX_{\mathrm{in}}}$ is isotropic, which implies that
	\begin{align}\label{eq:haystack1}
	\sigma_d \left( \widetilde{\bX_{\mathrm{in}}} \right) &>   \left( (1-a) \sqrt{\frac{N_{\mathrm{in}}}{d}}- C_1' \right), \ \text{w.p. at least } 1-2e^{- c_1' a^2 N_{\mathrm{in}}},
	\end{align}
	where as before we must have $(1-a)^2 N_{\mathrm{in}}^2 > C_1^{'2} d$.
	Abusing notation again (as we did in the proof of Proposition~\ref{prop:subgaussinlier}), we let $c_1$ be the minimum of $c_1$ and $c_1'$ and $C_1$ be the max of $C_1$ and $C_1'$ and we find that
	\begin{align}\label{eq:haystack2}
	\lambda_d \left( \sum_{\cX_{\mathrm{in}}} \frac{\bx_i \bx_i^T}{\|\bx_i \|} \right) &>  \sigma_{\mathrm{in}}  \left( (1-a) \sqrt{\frac{N_{\mathrm{in}}}{d}}- C_1 \right)^2, \text{ w.p. at least } 1 - 4e^{- c_1 a^2 N_{\mathrm{in}}}.
	\end{align}
	This results in the following exact statement for the Haystack Model. We combine~\eqref{eq:haystackfinaloutbd} with~\eqref{eq:haystack2} to find that if~\eqref{eq:haystackSNR} holds, then $\cS(\cX, L_*, \gamma)>0$ with probability at least~\eqref{eq:probhaystackbd}.

\subsection{Proof of Theorem~\ref{thm:haystacksnrzero}}
\label{app:snrzeroproof}

Here, we will show that GGD can recover the underlying subspace for any fixed fraction of outliers, provided that $N_{\mathrm{out}}$ is large enough. 
We denote the percentage of inliers by $\alpha_1$ and the percentage of outliers by $\alpha_0$. Then, the SNR is given by $\alpha_1 / \alpha_0$. 
    
    Assume for simplicity that $\sigma_{\mathrm{out}} = 1$, since the general case follows from the same logic. 
    First, we will bound the maximum norm of the set of outlier points. Each $\bx_i$ is i.i.d.~$\cN(\bzero, \bI / D)$, we can bound its norm by 
    \begin{equation}\label{eq:outnorm}
    \Pr\left(  \| \bx_i \| \leq 1 + \frac{t}{\sqrt{D}}  \right) \geq  1-\exp\left( -c t^2 \right),
    \end{equation}
    where $c$ is just some universal constant.
    Thus, applying a union bound to~\eqref{eq:outnorm} yields 
    \begin{equation}
    \Pr\left( \max_i  \| \bx_i \| \leq 1 + \frac{t}{\sqrt{D}}  \right) \geq 1- N_{\mathrm{out}} \exp\left( -c t^2 \right).
    \end{equation}
    We can use a value of $t=\sqrt{D} N_{\mathrm{out}}^{1/6}$ to find 
    \begin{equation}\label{eq:gaussmaxnorm}
    \Pr\left( \max_i  \| \bx_i \| \leq 1  + N^{1/6} \right) \geq 1- N_{\mathrm{out}}\exp\left( -c D N_{\mathrm{out}}^{1/3} \right).
    \end{equation}

    Next, consider the alignment for a single fixed subspace $L \neq L_*$:
    \begin{equation}\label{eq:alignspec}
    \cA(\cX_{\mathrm{out}}, L) = \left\| \sum_{\cX_{\mathrm{out}}} \frac{\bQ_L \bx_i \bx_i^T \bP_L}{\| \bQ_L \bx_i \|} \right\|_2.
    \end{equation}
    Taking an expectation within the norm in~\eqref{eq:alignspec} yields 
    \begin{equation}\label{eq:alignexp}
        \E \sum_{\cX_{\mathrm{out}}} \frac{\bQ_L \bx_i \bx_i^T \bP_L}{\| \bQ_L \bx_i \|}  = \bzero.
    \end{equation}
    This follows by the symmetry of the outlier distribution, $\cN(\bzero, \bI/D)$.
    Fix arbitrary vectors $\bu \in \Sp(\bQ_L)\cap S^{D-1}$ and $\bv \in \Sp(\bP_L)\cap S^{D-1}$. Then, \eqref{eq:alignexp} implies that
    \begin{equation}\label{eq:alignexp2}
    \E \sum_{\cX_{\mathrm{out}}} \frac{\bu^T \bx_i \bx_i^T \bv}{\| \bQ_L \bx_i \|}  = 0.
    \end{equation}
    In the following, we will continue to use $\bu$ and $\bv$ defined in this way. 
    
    To this end, define the following random variable
    \begin{equation}
        J(\bx,\bu,\bv) = \frac{\bu^T \bx \bx^T \bv}{\|\bQ_L \bx\|},
    \end{equation}
   where $\bx$ is distributed $\cN(\bzero, \bI/D)$ (i.e., it is an outlier in the Haystack model). We will first give a concentration bound for $\sum_{\cX_{\mathrm{out}}} J(\bx_i,\bu,\bv) $, which appears in~\eqref{eq:alignexp2}.
    Notice that $J(\bx, \bu, \bv)$ is a mean zero random variable and is bounded by
    \begin{equation}\label{eq:Jbd}
    |J(\bx,\bu,\bv)| \leq |\bx^T \bv| .
    \end{equation}
    Here, $\bx^T \bv$ is Gaussian with variance $1/D$.
    Thus, $J(\bx,\bu,\bv)$ is sub-Gaussian with variance proxy $1/D$. This implies that
    \begin{equation}
        \Pr \left( \left|\sum_{\cX_{\mathrm{out}}} J(\bx_i,\bu,\bv) \right| > N_{\mathrm{out}}t \right) \leq 2\exp \left( - \frac{t^2 DN_{\mathrm{out}}}{2} \right).
    \end{equation}
    Letting $t = N_{\mathrm{out}}^{-1/3}$ we find that
    \begin{equation}\label{eq:oneJbd}
        \Pr \left( \left|\sum_{\cX_{\mathrm{out}}} J(\bx_i,\bu,\bv) \right| > N_{\mathrm{out}}^{2/3} \right) \leq 2\exp \left( - \frac{DN_{\mathrm{out}}^{1/3}}{2} \right).
    \end{equation}
    
    Notice that $J(\bx,\bu,\bv)$ is continuous as a function of $\bu$ and $\bv$.
    Using~\eqref{eq:Jbd}, between two points $\bv_1,\bv_2 \in S^{D-1}$, we can bound the deviation in $\sum_{\cX_{\mathrm{out}}} J(\bx_i,\bu,\cdot)$ by 
    \begin{align}\label{eq:changeJ}
        \left|\sum_{\cX_{\mathrm{out}}} J(\bx_i,\bu,\bv_1) -  \sum_{\cX_{\mathrm{out}}} J(\bx_i,\bu,\bv_2)\right| &\leq \sum_{\cX_{\mathrm{out}}} | (\bv_1 - \bv_2)^T \bx_i | \leq \| \bv_1 - \bv_2\| N_{\mathrm{out}} \max_i \| \bx_i \|.
    \end{align}
    Combining~\eqref{eq:gaussmaxnorm},~\eqref{eq:oneJbd} and~\eqref{eq:changeJ} with $\| \bv_1 - \bv_2\| < N_{\mathrm{out}}^{-1/3}$ yields
    \begin{align} \label{eq:subballprob}
    \Pr \left( \max_{\bv_2 : \| \bv_2 - \bv_1\| \leq N_{\mathrm{out}}^{-1/3}} \left| \sum_i J(\bx_i, \bu, \bv_2) \right| \leq 3 N_{\mathrm{out}}^{5/6} \right) &\geq  1 - N_{\mathrm{out}}\exp\left( -c D N_{\mathrm{out}}^{1/3} \right) \\ \nonumber
    & \quad  - 2\exp \left( - \frac{DN_{\mathrm{out}}^{1/3}}{2} \right) .
    \end{align}
    Since $L \cap S^{D-1}$ can be covered by $\left( 2  N_{\mathrm{out}}^{1/3} + 1 \right)^d$ balls of radius $N_{\mathrm{out}}^{-1/3}$, we use a covering argument with~\eqref{eq:subballprob} to find
    \begin{align}
        \Pr \left( \max_{\bv \in L \cap S^{D-1}} \left|\sum_{\cX_{\mathrm{out}}} J(\bx_i,\bu,\bv) \right| \leq  3 N_{\mathrm{out}}^{5/6} \right) &\geq 1 -   N_{\mathrm{out}}\exp\left( -c D N_{\mathrm{out}}^{1/3} \right) \\ \nonumber
            &\quad   - 2 \left( 2 N_{\mathrm{out}}^{1/3} + 1 \right)^d \exp \left( - \frac{DN_{\mathrm{out}}^{1/3}}{2} \right).
    \end{align}
    Notice that we do not include the probability from~\eqref{eq:gaussmaxnorm} in the covering, since this probability holds independently of the choice of $\bv \in L \cap S^{D-1}$.
    
    The previous argument was actually independent of choice of $\bu \in L^{\perp} \cap S^{D-1}$. Thus, we finally have
    \begin{align}\label{eq:subchange}
        \Pr &\left( \left \| \sum_{\cX_{\mathrm{out}}} \frac{ \bQ_{L} \bx_i \bx_i^T \bP_L}{\| \bQ_{L} \bx_i \|} \right\|_2 \leq  3 N_{\mathrm{out}}^{5/6}  \right) \geq \\ \nonumber
            & \quad \quad \quad 1  - N_{\mathrm{out}}\exp\left( -c D N_{\mathrm{out}}^{1/3} \right) - 2 \left( 2 N_{\mathrm{out}}^{1/3} + 1 \right)^d \exp \left( - \frac{DN_{\mathrm{out}}^{1/3}}{2} \right) .
    \end{align}

    Now that we have covered a single $L \in G(D,d)$, we must extend this to all of $G(D,d)$ by another covering argument. From Lemma 2.2 of~\cite{dasgupta2003elementary}, we have, for each $\bx_i \in \cX_{\mathrm{out}}$
    \begin{equation}
        \Pr \left( \angle(\bx_i, L) < \frac{\pi}{2} \sqrt{\frac{\beta (D-d)}{D}} \right) \leq \exp \left( \frac{D-d}{2} (1 + \log(\beta)) \right).
    \end{equation}
    If we choose $\beta < N_{\mathrm{out}}^{-2/3}$, then it is not hard to show that
    \begin{equation}
        \Pr \left( \angle(\bx_i, L) <  \frac{\pi}{2}\sqrt{ \frac{\beta (D-d)}{D}} \right) \leq N_{\mathrm{out}}^{-1/3}.
    \end{equation}
    Define the cone around a subspace $L$ as
	\begin{equation}
		\cC(L, \xi) = \{ \bx \in \reals^{D} : \angle(\bx, L) < \xi\}.
	\end{equation}
    Using a loose Chernoff bound for the concentration of binomial random variables~\citep{mitzenmacher2005probability}, we have
    \begin{equation}
        \Pr \left( \left|\#\left(\cX_{\mathrm{out}} \cap \cC\left(L, N_{\mathrm{out}}^{-1/3} \frac{\pi}{2}\sqrt{ \frac{ (D-d)}{D}} \right)\right) -\delta N_{\mathrm{out}}^{2/3}\right| \geq N_{\mathrm{out}}^{2/3} \right) \leq \exp\left(-\frac{N_{\mathrm{out}}^{2/3} \delta^2}{3} \right).
    \end{equation}
Choosing $\delta = N_{\mathrm{out}}^{-1/6}$ yields
\begin{equation}\label{eq:coneprob}
        \Pr \left( \left|\#\left(\cX_{\mathrm{out}} \cap \cC\left(L, N_{\mathrm{out}}^{-1/3} \frac{\pi}{2}\sqrt{ \frac{ (D-d)}{D}} \right)\right) - N_{\mathrm{out}}^{1/2}\right| \geq N_{\mathrm{out}}^{2/3} \right) \leq \exp\left(-\frac{N_{\mathrm{out}}^{1/3}}{3} \right).
    \end{equation}
       
	    For any $L_0 \in G(D,d)$, and $L_1 \in B(L_0, \xi)$, we can separate the alignment term into two parts: those $\bx_i$ that are close to $L_0$ and $L_1$ and those that lie further away. The idea behind this is that points that are far away from $L_0$ and $L_1$ will contribute similar amounts to the alignment. On the other hand, those that are very close can contribute at most 2 times their norm to the alignment. 

    Rigorously, we write
    \begin{align}\label{eq:grassalignpert}
        &\left\| \sum_{\cX_{\mathrm{out}}} \frac{\bQ_{L_0} \bx_i \bx_i^T \bP_{L_0}}{\| \bQ_{L_0} \bx_i  \|_2 } -  \frac{\bQ_{L_1} \bx_i \bx_i^T \bP_{L_1}}{\| \bQ_{L_1} \bx_i  \|_2 } \right\|_2  \\ \nonumber
        & \quad \quad \quad \quad \quad \leq\left\| \sum_{\bx_i \in \cX_{\mathrm{out}} \cap \cC(L_0, \xi)} \frac{\bQ_{L_0} \bx_i \bx_i^T \bP_{L_0}}{\| \bQ_{L_0} \bx_i  \|_2 } -  \frac{\bQ_{L_1} \bx_i \bx_i^T \bP_{L_1}}{\| \bQ_{L_1} \bx_i  \|_2 } \right\|_2 + \\ \nonumber 
                                                                                                                                                                                                                   &\quad \quad \quad \quad \quad \quad \quad  \left\| \sum_{\bx_i \in \cX_{\mathrm{out}} \cap \cC(L_0, \xi)^C} \frac{\bQ_{L_0} \bx_i \bx_i^T \bP_{L_0}}{\| \bQ_{L_0} \bx_i  \|_2 } -  \frac{\bQ_{L_1} \bx_i \bx_i^T \bP_{L_1}}{\| \bQ_{L_1} \bx_i  \|_2 } \right\|_2 \\ \nonumber
                                                                                                                                                                                                                   &\quad \quad \quad \quad \quad \leq \sum_{\bx_i \in \cX_{\mathrm{out}} \cap \cC(L_0, \xi)} 2\|\bx_i\| + 2 \xi \sum_{\bx_i \in \cX_{\mathrm{out}} \cap \cC(L_0, \xi)^C} \| \bx_i \|_2.
    \end{align}
    We will examining~\eqref{eq:grassalignpert} term by term. For the first term, if we choose $\xi = N_{\mathrm{out}}^{-1/3}$ and combine~\eqref{eq:gaussmaxnorm} and~\eqref{eq:coneprob}, we find that
    \begin{align}\label{eq:gddcoverfirstterm}
        \sum_{\bx_i \in \cX_{\mathrm{out}} \cap \cC(L_0, \xi)} 2\|\bx_i\| &\leq 2 \#(\cX_{\mathrm{out}} \cap \cC(L_0, \xi)) \max_{\cX_{\mathrm{out}}} \| \bx_i \| \leq 2 \left(1+ N_{\mathrm{out}}^{1/6}\right) N_{\mathrm{out}}^{2/3}, \\ \nonumber
           &\text{w.p.~at least } 1 -N_{\mathrm{out}}\exp\left( -c D N_{\mathrm{out}}^{1/3} \right) - \exp\left(-\frac{N_{\mathrm{out}}^{1/3}}{3} \right) .
    \end{align}
    For the second term, choosing $\xi = N_{\mathrm{out}}^{-1/3}$ and again using~\eqref{eq:gaussmaxnorm} yields
    \begin{align}\label{eq:gddcoversecondterm}
2 \xi \sum_{\bx_i \in \cX_{\mathrm{out}} \cap \cC(L_0, \xi)^C} \| \bx_i \|_2
&\leq 2 \left(1+ N_{\mathrm{out}}^{1/6}\right) N_{\mathrm{out}}^{2/3} + o(1),   \\ \nonumber
    &\text{w.p.~at least } 1- N_{\mathrm{out}}\exp\left( -c D N_{\mathrm{out}}^{1/3} \right).
    \end{align}
    Putting~\eqref{eq:gddcoverfirstterm} and~\eqref{eq:gddcoversecondterm} yields
    \begin{align}\label{eq:grasschange}
        &\left\| \sum_{\cX_{\mathrm{out}}} \frac{\bQ_{L_0} \bx_i \bx_i^T \bP_{L_0}}{\| \bQ_{L_0} \bx_i  \|_2 } -  \frac{\bQ_{L_1} \bx_i \bx_i^T \bP_{L_1}}{\| \bQ_{L_1} \bx_i  \|_2 } \right\|_2 \leq  8 N_{\mathrm{out}}^{5/6}
        ,\\ \nonumber
          &\quad \quad \quad \quad \quad \text{w.p.~at least } 1 - N_{\mathrm{out}}\exp\left( -c D N_{\mathrm{out}}^{1/3} \right) -\exp\left(-\frac{N_{\mathrm{out}}^{1/3}}{3} \right) .
    \end{align}
    As already mentioned in Appendix~\ref{app:propaligncov} (following~\citet{szarek1983finite}), $G(D,d)$ can be covered by $(C_1)^{d(D-d)} / (\gamma_1)^{d(D-d)}$ balls of radius $\gamma_1$. Thus, by a union bound with~\eqref{eq:subchange} and~\eqref{eq:grasschange},
    \begin{align}\label{eq:finaloutbd}
    	 &\left \| \sum_{\cX_{\mathrm{out}}} \frac{ \bQ_{L} \bx_i \bx_i^T \bP_L}{\| \bQ_{L} \bx_i \|} \right\|_2 \leq 11 N_{\mathrm{out}}^{5/6}, \ \forall \ L \in G(D,d),  \\ \nonumber
    	&\quad \text{ w.p. at least } 1 -  N_{\mathrm{out}}\exp\left( -c D N_{\mathrm{out}}^{1/3} \right) - 2 \left( (C_1)^{d(D-d)} ( N_{\mathrm{out}}^{1/3})^{d(D-d)} \right) \cdot   \\ \nonumber 
    	& \quad \quad \quad \Bigg(\left( 2  N_{\mathrm{out}}^{1/3} + 1 \right)^d \exp \left( - \frac{DN_{\mathrm{out}}^{1/3}}{2 } \right) + \exp\left(-\frac{N_{\mathrm{out}}^{1/3}}{3} \right) \Bigg) .
    \end{align}
	As $N_{\mathrm{out}} \to \infty$, we see that this probability goes to 1.
	
	On the other hand, we recall the bound on the permeance of inliers from~\eqref{eq:haystack2}. Combining this with~\eqref{eq:finaloutbd} and adding back in the scale factor $\sigma_{\mathrm{out}}$, we have that if
	\begin{equation}
		\cos(\gamma) \sigma_{\mathrm{in}}  \left( (1-a) \sqrt{\frac{N_{\mathrm{in}}}{d}}- C_1 \right)^2 	\geq 11 \sigma_{\mathrm{out}} N_{\mathrm{out}}^{5/6},
	\end{equation}
	then $\cS(\cX, L_*, \gamma) > 0$ w.o.p. This equates to
	\begin{equation}\label{eq:finsnrzero}
	 	\SNR \geq  \frac{11 d \sigma_{\mathrm{out}} }{\cos(\gamma) N_{\mathrm{out}}^{1/6} \sigma_{\mathrm{in}} (1-a)^2} + O\left(\frac{\sqrt{N_{\mathrm{in}}}}{N_{\mathrm{out}}}\right) ,
	\end{equation}
    which goes to $0$ as $N \to \infty$ for any fixed fraction of outliers. We see, in terms of dependence on parameters in~\eqref{eq:finaloutbd} and~\eqref{eq:finsnrzero}, $N_{\mathrm{out}}$ must be at least on the order of $O(\max(d^3D^3 \log^3(N_{\mathrm{out}}),(dN_{\mathrm{out}}/N_{\mathrm{in}})^6))$. This is due to the fact that~\eqref{eq:finsnrzero} must hold and the limiting probability in~\eqref{eq:finaloutbd} is
    $$ 2 \left( (C_1)^{d(D-d)} ( N_{\mathrm{out}}^{1/3})^{d(D-d)} \right) \exp\left(-\frac{N_{\mathrm{out}}^{1/3}}{3} \right)  ,$$
	which is only close to zero when the number of outliers is more than the specified regime.

\bibliography{refs_2017}

\begin{thebibliography}{68}
\providecommand{\natexlab}[1]{#1}
\providecommand{\url}[1]{\texttt{#1}}
\expandafter\ifx\csname urlstyle\endcsname\relax
  \providecommand{\doi}[1]{doi: #1}\else
  \providecommand{\doi}{doi: \begingroup \urlstyle{rm}\Url}\fi

\bibitem[Absil et~al.(2004)Absil, Mahony, and Sepulchre]{absil2004riemannian}
P.-A. Absil, R.~Mahony, and R.~Sepulchre.
\newblock Riemannian geometry of {G}rassmann manifolds with a view on
  algorithmic computation.
\newblock \emph{Acta Applicandae Mathematica}, 80\penalty0 (2):\penalty0
  199--220, 2004.

\bibitem[Absil et~al.(2009)Absil, Mahony, and Sepulchre]{absil2009optimization}
P.-A. Absil, R.~Mahony, and R.~Sepulchre.
\newblock \emph{Optimization algorithms on matrix manifolds}.
\newblock Princeton University Press, 2009.

\bibitem[{Arias-Castro} and {Wang}(2017)]{ariascastro2017ransac}
E.~{Arias-Castro} and J.~{Wang}.
\newblock {RANSAC} algorithms for subspace recovery and subspace clustering.
\newblock \emph{ArXiv e-prints}, November 2017.

\bibitem[Arias-Castro et~al.(2011)Arias-Castro, Chen, and
  Lerman]{ariascastro2011spectral}
E.~Arias-Castro, G.~Chen, and G.~Lerman.
\newblock Spectral clustering based on local linear approximations.
\newblock \emph{Electron. J. Statist.}, 5:\penalty0 1537--1587, 2011.

\bibitem[Arora et~al.(2015)Arora, Ge, Ma, and Moitra]{arora2015simple}
S.~Arora, R.~Ge, T.~Ma, and A.~Moitra.
\newblock Simple, efficient, and neural algorithms for sparse coding.
\newblock In \emph{COLT}, 2015.

\bibitem[Baik et~al.(2005)Baik, Ben~Arous, and P{\'e}ch{\'e}]{baik2005phase}
J.~Baik, G.~Ben~Arous, and S.~P{\'e}ch{\'e}.
\newblock Phase transition of the largest eigenvalue for nonnull complex sample
  covariance matrices.
\newblock \emph{The Annals of Probability}, 33\penalty0 (5):\penalty0
  1643--1697, 2005.

\bibitem[Boumal(2016)]{boumal2016nonconvexphase}
N.~Boumal.
\newblock Nonconvex phase synchronization.
\newblock \emph{SIAM Journal on Optimization}, 26\penalty0 (4):\penalty0
  2355--2377, 2016.

\bibitem[Cand{\`e}s et~al.(2011)Cand{\`e}s, Li, Ma, and
  Wright]{candes2011robust}
E.~J. Cand{\`e}s, X.~Li, Y.~Ma, and J.~Wright.
\newblock Robust principal component analysis?
\newblock \emph{Journal of the ACM (JACM)}, 58\penalty0 (3):\penalty0 11, 2011.

\bibitem[Cherapanamjeri et~al.(2017)Cherapanamjeri, Jain, and
  Netrapalli]{cherapanamjeri2017thresholding}
Y.~Cherapanamjeri, P.~Jain, and P.~Netrapalli.
\newblock Thresholding based outlier robust {PCA}.
\newblock In \emph{COLT}, pages 593--628, 2017.

\bibitem[Clarke(1990)]{clarke1990optimization}
F.~H. Clarke.
\newblock \emph{Optimization and nonsmooth analysis}, volume~5.
\newblock SIAM, 1990.

\bibitem[Clarkson and Woodruff(2015)]{clarkson2015input}
K.~L. Clarkson and D.~P. Woodruff.
\newblock Input sparsity and hardness for robust subspace approximation.
\newblock In \emph{Foundations of Computer Science (FOCS), 2015 IEEE 56th
  Annual Symposium on}, pages 310--329. IEEE, 2015.

\bibitem[Coudron and Lerman(2012)]{coudron2012sample}
M.~Coudron and G.~Lerman.
\newblock On the sample complexity of robust {PCA}.
\newblock In \emph{NIPS}, pages 3221--3229. 2012.

\bibitem[Dasgupta and Gupta(2003)]{dasgupta2003elementary}
S.~Dasgupta and A.~Gupta.
\newblock An elementary proof of a theorem of {J}ohnson and {L}indenstrauss.
\newblock \emph{Random Structures \& Algorithms}, 22\penalty0 (1):\penalty0
  60--65, 2003.

\bibitem[Dauphin et~al.(2014)Dauphin, Pascanu, Gulcehre, Cho, Ganguli, and
  Bengio]{dauphin2014identifying}
Y.~N. Dauphin, R.~Pascanu, C.~Gulcehre, K.~Cho, S.~Ganguli, and Y.~Bengio.
\newblock Identifying and attacking the saddle point problem in
  high-dimensional non-convex optimization.
\newblock In \emph{Advances in neural information processing systems}, pages
  2933--2941, 2014.

\bibitem[Davis and Kahan(1970)]{davis1970rotationIII}
C.~Davis and W.~M. Kahan.
\newblock The rotation of eigenvectors by a perturbation. iii.
\newblock \emph{SIAM J. on Numerical Analysis}, 7:\penalty0 1--46, 1970.

\bibitem[Ding et~al.(2006)Ding, Zhou, He, and Zha]{ding2006r1}
C.~Ding, D.~Zhou, X.~He, and H.~Zha.
\newblock {R}1-{PCA}: rotational invariant ${L}_1$-norm principal component
  analysis for robust subspace factorization.
\newblock In \emph{ICML}, pages 281--288. ACM, 2006.

\bibitem[Edelman et~al.(1999)Edelman, Arias, and Smith]{edelman1999geometry}
A.~Edelman, T.~A. Arias, and S.~T. Smith.
\newblock The geometry of algorithms with orthogonality constraints.
\newblock \emph{SIAM J. Matrix Anal. Appl.}, 20\penalty0 (2):\penalty0 303--353
  (electronic), 1999.
\newblock ISSN 0895-4798.

\bibitem[Ge et~al.(2015)Ge, Huang, Jin, and Yuan]{ge2015escaping}
R.~Ge, F.~Huang, C.~Jin, and Y.~Yuan.
\newblock Escaping from saddle points—online stochastic gradient for tensor
  decomposition.
\newblock In \emph{COLT}, pages 797--842, 2015.

\bibitem[Ge et~al.(2016)Ge, Lee, and Ma]{ge2016matrix}
R.~Ge, J.~D. Lee, and T.~Ma.
\newblock Matrix completion has no spurious local minimum.
\newblock In \emph{NIPS}, pages 2973--2981, 2016.

\bibitem[Goes et~al.(2014)Goes, Zhang, Arora, and Lerman]{goes2014robust}
J.~Goes, T.~Zhang, R.~Arora, and G.~Lerman.
\newblock Robust stochastic principal component analysis.
\newblock \emph{JMLR W\&CP}, page 266–274, 2014.

\bibitem[Hardt(2014)]{hardt2014understanding}
M.~Hardt.
\newblock Understanding alternating minimization for matrix completion.
\newblock In \emph{FOCS}, pages 651--660. IEEE, 2014.

\bibitem[Hardt and Moitra(2013)]{hardt2013algorithms}
M.~Hardt and A.~Moitra.
\newblock Algorithms and hardness for robust subspace recovery.
\newblock In \emph{COLT}, pages 354--375, 2013.

\bibitem[Jain et~al.(2014)Jain, Tewari, and Kar]{jain2014iterative}
P.~Jain, A.~Tewari, and P.~Kar.
\newblock On iterative hard thresholding methods for high-dimensional
  m-estimation.
\newblock In \emph{NIPS}, pages 685--693, 2014.

\bibitem[Johnstone(2001)]{johnstone2001distribution}
I.~M. Johnstone.
\newblock On the distribution of the largest eigenvalue in principal components
  analysis.
\newblock \emph{Annals of statistics}, pages 295--327, 2001.

\bibitem[Jolliffe(2002)]{jolliffe2002principal}
I.~T. Jolliffe.
\newblock \emph{Principal Component Analysis}.
\newblock Springer Series in Statistics. Springer, 2nd edition, 2002.

\bibitem[Ledyaev and Zhu(2007)]{ledyaev2007nonsmooth}
Y.~Ledyaev and Q.~Zhu.
\newblock Nonsmooth analysis on smooth manifolds.
\newblock \emph{Transactions of the American Mathematical Society},
  359\penalty0 (8):\penalty0 3687--3732, 2007.

\bibitem[Lee et~al.(2016)Lee, Simchowitz, Jordan, and Recht]{lee2016gradient}
J.~D. Lee, M.~Simchowitz, M.~I. Jordan, and B.~Recht.
\newblock Gradient descent only converges to minimizers.
\newblock In \emph{COLT}, pages 1246--1257, 2016.

\bibitem[Lerman and Maunu(2018{\natexlab{a}})]{lerman2017fast}
G.~Lerman and T.~Maunu.
\newblock Fast, robust and non-convex subspace recovery.
\newblock \emph{Information and Inference: A Journal of the IMA}, 7\penalty0
  (2):\penalty0 277--336, 2018{\natexlab{a}}.

\bibitem[Lerman and Maunu(2018{\natexlab{b}})]{lerman2018overview}
G.~Lerman and T.~Maunu.
\newblock An overview of robust subspace recovery.
\newblock \emph{Proceedings of the IEEE}, 106\penalty0 (8):\penalty0
  1380--1410, Aug 2018{\natexlab{b}}.
\newblock ISSN 0018-9219.
\newblock \doi{10.1109/JPROC.2018.2853141}.

\bibitem[Lerman and Zhang(2011)]{lerman2011robust}
G.~Lerman and T.~Zhang.
\newblock Robust recovery of multiple subspaces by geometric ${{l_p}}$
  minimization.
\newblock \emph{Ann. Statist.}, 39\penalty0 (5):\penalty0 2686--2715, 2011.

\bibitem[Lerman and Zhang(2014)]{lerman2014lp}
G.~Lerman and T.~Zhang.
\newblock {$l_p$}-recovery of the most significant subspace among multiple
  subspaces with outliers.
\newblock \emph{Constructive Approximation}, 40\penalty0 (3):\penalty0
  329--385, 2014.

\bibitem[Lerman et~al.(2015)Lerman, McCoy, Tropp, and Zhang]{lerman2015robust}
G.~Lerman, M.~B. McCoy, J.~A. Tropp, and T.~Zhang.
\newblock Robust computation of linear models by convex relaxation.
\newblock \emph{Foundations of Computational Mathematics}, 15\penalty0
  (2):\penalty0 363--410, 2015.

\bibitem[Lim et~al.(2016)Lim, Wong, and Ye]{lim2016statistical}
L.~Lim, K.~S. Wong, and K.~Ye.
\newblock Statistical estimation and the affine {G}rassmannian.
\newblock \emph{arXiv preprint arXiv:1607.01833}, 2016.

\bibitem[Lim et~al.(2018)Lim, Wong, and Ye]{lim2018grassmannian}
L.~Lim, K.~S. Wong, and K.~Ye.
\newblock The {Grassmannian} of affine subspaces.
\newblock \emph{arXiv preprint arXiv:1807.10883}, 2018.

\bibitem[Locantore et~al.(1999)Locantore, Marron, Simpson, Tripoli, Zhang, and
  Cohen]{locantore1999robust}
N.~Locantore, J.~S. Marron, D.~G. Simpson, N.~Tripoli, J.~T. Zhang, and K.~L.
  Cohen.
\newblock Robust principal component analysis for functional data.
\newblock \emph{Test}, 8\penalty0 (1):\penalty0 1--73, 1999.

\bibitem[Lu and Pearce(2000)]{lu2000some}
L.-Z. Lu and C.~E.~M. Pearce.
\newblock Some new bounds for singular values and eigenvalues of matrix
  products.
\newblock \emph{Annals of Operations Research}, 98\penalty0 (1-4):\penalty0
  141--148, 2000.

\bibitem[Ma et~al.(2018)Ma, Wang, Chi, and Chen]{ma2018implicit}
C.~Ma, K.~Wang, Y.~Chi, and Y.~Chen.
\newblock Implicit regularization in nonconvex statistical estimation: Gradient
  descent converges linearly for phase retrieval and matrix completion.
\newblock In \emph{PMLR}, volume~80, pages 3345--3354, 10--15 Jul 2018.

\bibitem[Maronna(2005)]{maronna2005principal}
R.~A. Maronna.
\newblock Principal components and orthogonal regression based on robust
  scales.
\newblock \emph{Technometrics}, 47:\penalty0 264--273, 2005.
\newblock ISSN 1537-2723.

\bibitem[Maronna et~al.(2006)Maronna, Martin, and Yohai]{maronna2006robust}
R.~A. Maronna, R.~D. Martin, and V.~J. Yohai.
\newblock \emph{Robust statistics: Theory and methods}.
\newblock Wiley Series in Probability and Statistics. John Wiley \& Sons Ltd.,
  Chichester, 2006.
\newblock ISBN 978-0-470-01092-1; 0-470-01092-4.

\bibitem[Maunu and Lerman()]{ML18}
T.~Maunu and G.~Lerman.
\newblock Robust subspace recovery with adverserial outliers.
\newblock In preparation.

\bibitem[McCoy and Tropp(2011)]{mccoy2011two}
M.~McCoy and J.~A Tropp.
\newblock Two proposals for robust {PCA} using semidefinite programming.
\newblock \emph{Electronic Journal of Statistics}, 5:\penalty0 1123--1160,
  2011.

\bibitem[Mei et~al.(2018)Mei, Bai, and Montanari]{mei2018landscape}
S.~Mei, Y.~Bai, and A.~Montanari.
\newblock The landscape of empirical risk for nonconvex losses.
\newblock \emph{The Annals of Statistics}, 46\penalty0 (6A):\penalty0
  2747--2774, 2018.

\bibitem[Minsker(2015)]{minsker2015geometric}
S.~Minsker.
\newblock Geometric median and robust estimation in banach spaces.
\newblock \emph{Bernoulli}, 21\penalty0 (4):\penalty0 2308--2335, 2015.

\bibitem[Mitzenmacher and Upfal(2005)]{mitzenmacher2005probability}
M.~Mitzenmacher and E.~Upfal.
\newblock \emph{Probability and computing: Randomized algorithms and
  probabilistic analysis}.
\newblock Cambridge university press, 2005.

\bibitem[Netrapalli et~al.(2014)Netrapalli, Niranjan, Sanghavi, Anandkumar, and
  Jain]{netrapalli2014non}
P.~Netrapalli, U.~N. Niranjan, S.~Sanghavi, A.~Anandkumar, and P.~Jain.
\newblock Non-convex robust {PCA}.
\newblock In \emph{NIPS}, pages 1107--1115, 2014.

\bibitem[Osborne and Watson(1985)]{osborne1985analysis}
M.~R. Osborne and G.~A. Watson.
\newblock An analysis of the total approximation problem in separable norms,
  and an algorithm for the total $l_1$ problem.
\newblock \emph{SIAM journal on scientific and statistical computing},
  6\penalty0 (2):\penalty0 410--424, 1985.

\bibitem[Rahmani and Atia(2016)]{rahmani2016coherence}
M.~Rahmani and G.~K. Atia.
\newblock Coherence pursuit: Fast, simple, and robust principal component
  analysis.
\newblock \emph{IEEE Transactions on Signal Processing}, 65\penalty0
  (23):\penalty0 6260--6275, 2016.

\bibitem[Rudelson and Vershynin(2008)]{rudelson2008littlewood}
M.~Rudelson and R.~Vershynin.
\newblock The {L}ittlewood--{O}fford problem and invertibility of random
  matrices.
\newblock \emph{Advances in Mathematics}, 218\penalty0 (2):\penalty0 600 --
  633, 2008.
\newblock ISSN 0001-8708.

\bibitem[S{\"o}derstr{\"o}m(1999)]{soderstrom1999perturbation}
T.~S{\"o}derstr{\"o}m.
\newblock \emph{Perturbation results for singular values}.
\newblock Institutionen f{\"o}r informationsteknologi, Uppsala universitet,
  1999.

\bibitem[Soltanolkotabi and Cand{\`e}s(2012)]{soltanolkotabi2012geometric}
M.~Soltanolkotabi and E.~J. Cand{\`e}s.
\newblock {A geometric analysis of subspace clustering with outliers.}
\newblock \emph{Ann. Stat.}, 40\penalty0 (4):\penalty0 2195--2238, 2012.
\newblock \doi{10.1214/12-AOS1034}.

\bibitem[St.~Thomas et~al.(2014)St.~Thomas, Lin, Lim, and
  Mukherjee]{thomas2014learning}
B.~St.~Thomas, L.~Lin, L.~Lim, and S.~Mukherjee.
\newblock Learning subspaces of different dimension.
\newblock \emph{arXiv preprint arXiv:1404.6841}, 2014.

\bibitem[Sun et~al.(2015{\natexlab{a}})Sun, Qu, and
  Wright]{SunQuWright_nonconvex_sphere_2015}
J.~Sun, Q.~Qu, and J.~Wright.
\newblock Complete dictionary recovery over the sphere.
\newblock In \emph{SAMPTA}, pages 407--410, May 2015{\natexlab{a}}.

\bibitem[Sun et~al.(2015{\natexlab{b}})Sun, Qu, and Wright]{sun2015nonconvex}
J.~Sun, Q.~Qu, and J.~Wright.
\newblock When are nonconvex problems not scary?
\newblock \emph{arXiv preprint arXiv:1510.06096}, 2015{\natexlab{b}}.

\bibitem[Szarek(1983)]{szarek1983finite}
S.~J. Szarek.
\newblock The finite-dimensional basis problem with an appendix on nets of
  {G}rassmann manifolds.
\newblock \emph{Acta Math.}, 151\penalty0 (3-4):\penalty0 153--179, 1983.
\newblock ISSN 0001-5962.

\bibitem[Vershynin(2012{\natexlab{a}})]{vershynin2012close}
R.~Vershynin.
\newblock How close is the sample covariance matrix to the actual covariance
  matrix?
\newblock \emph{Journal of Theoretical Probability}, 25\penalty0 (3):\penalty0
  655--686, 2012{\natexlab{a}}.

\bibitem[Vershynin(2012{\natexlab{b}})]{vershynin2012introduction}
R.~Vershynin.
\newblock Introduction to the non-asymptotic analysis of random matrices.
\newblock In \emph{Compressed sensing}, pages 210--268. Cambridge Univ. Press,
  Cambridge, 2012{\natexlab{b}}.

\bibitem[Vu et~al.(2013)Vu, Cho, Lei, and Rohe]{vu2013fantope}
V.~Q. Vu, J.~Cho, J.~Lei, and K.~Rohe.
\newblock Fantope projection and selection: A near-optimal convex relaxation of
  sparse {PCA}.
\newblock In \emph{NIPS}, pages 2670--2678, 2013.

\bibitem[Watson(2001)]{watson2001some}
G.~A. Watson.
\newblock \emph{Some Problems in Orthogonal Distance and Non-Orthogonal
  Distance Regression}.
\newblock Defense Technical Information Center, 2001.
\newblock URL \url{http://books.google.com/books?id=WKKWGwAACAAJ}.

\bibitem[Xu et~al.(2012)Xu, Caramanis, and Sanghavi]{xu2012robust}
H.~Xu, C.~Caramanis, and S.~Sanghavi.
\newblock Robust {PCA} via outlier pursuit.
\newblock \emph{{IEEE} Trans. Information Theory}, 58\penalty0 (5):\penalty0
  3047--3064, 2012.

\bibitem[Xu et~al.(2013)Xu, Caramanis, and Mannor]{xu2013outlier}
H.~Xu, C.~Caramanis, and S.~Mannor.
\newblock Outlier-robust {PCA}: the high-dimensional case.
\newblock \emph{{IEEE} Trans. Information Theory}, 59\penalty0 (1):\penalty0
  546--572, 2013.

\bibitem[Ye and Lim(2016)]{ye2016schubert}
K.~Ye and L.~Lim.
\newblock Schubert varieties and distances between subspaces of different
  dimensions.
\newblock \emph{SIAM Journal on Matrix Analysis and Applications}, 37\penalty0
  (3):\penalty0 1176--1197, 2016.

\bibitem[Yi et~al.(2016)Yi, Park, Chen, and Caramanis]{yi2016fast}
X.~Yi, D.~Park, Y.~Chen, and C.~Caramanis.
\newblock Fast algorithms for robust {PCA} via gradient descent.
\newblock In \emph{NIPS}, pages 4152--4160, 2016.

\bibitem[Zhang and Balzano(2016)]{zhang2016global}
D.~Zhang and L.~Balzano.
\newblock Global convergence of a {Grassmannian} gradient descent algorithm for
  subspace estimation.
\newblock In \emph{AISTATS}, pages 1460--1468, 2016.

\bibitem[Zhang(2016)]{zhang2016robust}
T.~Zhang.
\newblock Robust subspace recovery by {T}yler's {M}-estimator.
\newblock \emph{Information and Inference}, 5\penalty0 (1):\penalty0 1--21,
  2016.

\bibitem[Zhang and Lerman(2014)]{zhang2014novel}
T.~Zhang and G.~Lerman.
\newblock A novel {M}-estimator for robust {PCA}.
\newblock \emph{Journal of Machine Learning Research}, 15\penalty0
  (1):\penalty0 749--808, 2014.

\bibitem[Zhang and Yang(2017)]{zhang2017robust}
T.~Zhang and Y.~Yang.
\newblock Robust principal component analysis by manifold optimization.
\newblock \emph{arXiv preprint arXiv:1708.00257}, 2017.

\bibitem[Zhang et~al.(2009)Zhang, Szlam, and Lerman]{zhang2009median}
T.~Zhang, A.~Szlam, and G.~Lerman.
\newblock Median {$K$}-flats for hybrid linear modeling with many outliers.
\newblock In \emph{International Conference on Computer Vision Workshops
  ({ICCV} Workshops)}, pages 234--241, Kyoto, Japan, 2009.

\bibitem[Zhou et~al.(2010)Zhou, Li, Wright, Cand\`es, and Ma]{zhou2010stable}
Z.~Zhou, X.~Li, J.~Wright, E.~Cand\`es, and Y.~Ma.
\newblock Stable principal component pursuit.
\newblock In \emph{International Symposium on Information Theory Proceedings
  (ISIT)}, pages 1518--1522. IEEE, 2010.

\end{thebibliography}
	
\end{document}